\documentclass{article}
\usepackage{kr}

\usepackage{times}
\usepackage{soul}
\usepackage{url}
\usepackage[hidelinks]{hyperref}

\usepackage{paralist}
\usepackage[utf8]{inputenc}
\usepackage[T1]{fontenc}
\usepackage{amsfonts}
\usepackage{latexsym,amsmath,amssymb}
\usepackage{stmaryrd}
\usepackage{wasysym}
\usepackage{rotating}
\usepackage{pifont}
\usepackage{booktabs}
\usepackage{array}
\usepackage{multirow}
\usepackage{tikz}
\usepackage{color}
\usepackage{xcolor}
\usepackage{enumitem}
\usepackage{xspace}
\usepackage{thm-restate}
\usepackage{bm}

\usepackage{misc}
\usepackage{enumitem}

\usepackage{pifont} 

\usepackage{tikz}
\usetikzlibrary{automata}
\usetikzlibrary{positioning,arrows,fit,calc}
\usetikzlibrary{decorations.pathreplacing}
\usetikzlibrary{matrix}

\newcommand{\qed}{\hfill\ding{113}}

\newcommand{\avec}[1]{\boldsymbol{#1}}

\newcommand{\LTL}{\textsl{LTL}\xspace}

\newcommand{\U}{\mathbin{\mathsf{U}}}

\newcommand{\ind}{\textit{ind}}

\newcommand{\dep}{\textit{tdp}}

\newcommand{\suc}{\textit{suc}}
\newcommand{\tp}{\textit{tp}}
\newcommand{\sub}{\textit{sub}}

\newcommand{\D}{\mathcal{D}}

\newcommand{\I}{\mathcal{I}}

\renewcommand{\P}{\mathcal{P}}
\newcommand{\Diamondw}{\Diamond_{\!r}}
\newcommand{\nxt}{{\ensuremath\raisebox{0.25ex}{\text{\scriptsize$\bigcirc$}}}}

\newcommand{\core}{\textit{core}}

\newcommand{\op}{{\boldsymbol{o}}}

\newcommand{\var}{\textit{var}}
\newcommand{\tvar}{\textit{tv}}
\newcommand{\len}{\max}

\newcommand{\A}{\mathcal A}
\newcommand{\N}{\mathcal N}
\newcommand{\q}{{\avec{q}}}
\newcommand{\el}{\avec{l}}
\newcommand{\s}{\avec{s}}
\renewcommand{\r}{\avec{r}}

\newcommand{\TO}{\mathcal O}

\newcommand{\sig}{\textit{sig}}

\newcommand{\qw}{\hat{\q}}

\newcommand{\NDR}{^{\smash{\raisebox{1pt}{$\scriptscriptstyle\bigcirc$}\Diamond_r}}}

\newcommand{\ND}{^{\smash{\raisebox{1pt}{$\scriptscriptstyle\bigcirc$}\Diamond}}}

\newcommand{\NDD}{^{\smash{\raisebox{1pt}{$\scriptscriptstyle\bigcirc$}\Diamond\Diamond_r}}}

\newcommand{\UN}{^{\smash{\scriptstyle\U}}}

\newcommand{\wrt}{wrt\ }
\newcommand{\Q}{\mathcal{Q}}
\newcommand{\FQ}{\Theta}
\newcommand{\C}{\mathcal{C}}
\newcommand{\Type}{{\boldsymbol{T}}}

\usepackage{amsthm}
\newtheorem{theorem}{Theorem}

\newtheorem{example}{Example}
\newtheorem{lemma}{Lemma}

\title{Unique Characterisability and Learnability of Temporal Queries\\ Mediated by an Ontology}

\author{%
Jean Christoph Jung$^1$\and
Vladislav Ryzhikov$^2$\and
Frank Wolter$^3$\and
Michael Zakharyaschev$^2$ \\
\affiliations
$^1$TU Dortmund University, Germany\\
$^2$Birkbeck, University of London, UK\\
$^3$University of Liverpool, UK\\
\emails
jean.jung@tu-dortmund.de,
\{v.ryzhikov, m.zakharyaschev\}@bbk.ac.uk,
wolter@liverpool.ac.uk
}	
	
\begin{document}

\maketitle

\begin{abstract}
  Algorithms for learning database queries from examples and unique
characterisations of queries by examples are prominent starting points for developing automated support for query construction and explanation. We investigate how far recent results and techniques on learning and unique charac\-terisations of atemporal queries mediated by an ontology can be extended to temporal data and queries. Based on a systematic review of the relevant approaches in the atemporal case, we obtain general transfer results identifying  conditions under which temporal queries composed of atemporal ones are (polynomially) learnable and uniquely characterisable.
%
\end{abstract}

\section{Introduction}\label{intro}

Providing automated support for constructing database queries from data examples has been an important research topic in database management, knowledge representation and computational logic, often subsumed under the query-by-example  paradigm~\cite{martins2019reverse}. One prominent approach is based on exact learning using membership queries~\cite{DBLP:journals/ml/Angluin87}, where one aims to identify a database query by repeatedly asking an oracle (e.g., domain expert) whether certain data examples are answers or non-answers to the query. Recently, the ability to uniquely characterise a database query by a finite set of positive and negative examples has been identified and investigated as a `non-procedural' necessary condition for learnability via membership queries~\cite{DBLP:conf/icdt/StaworkoW15,DBLP:journals/tods/CateD22,DBLP:conf/kr/FortinKRSWZ22}.
More precisely, a query $\q(x)$ is said to fit a pair $E=(E^{+},E^{-})$ of sets $E^{+}$ and $E^{-}$ of pointed databases $(\mathcal{D},a)$ if $\mathcal{D}\models \q(a)$ for all $(\mathcal{D},a) \in E^{+}$, and $\mathcal{D}\not\models \q(a)$ for all $(\mathcal{D},a) \in E^{-}$. The example set $E$ uniquely characterises $\q$ within a class $\mathcal{Q}$ of queries if $\q$ is the only one (up to equivalence) in $\mathcal{Q}$ that fits $E$. The existence of (polynomial-size) unique characterisations is a necessary pre-condition for (polynomial) learnability via membership queries. Such characterisations can also be employed for explaining and synthesising queries.

Extending results on characterising and learning conjunctive queries (CQs) under the standard closed-world semantics~\cite{DBLP:journals/tods/CateD22}, there has recently been significant progress towards CQs mediated by a description logic (DL) ontology under the open-world semantics (Funk, Jung, and Lutz~\citeyear{DBLP:conf/ijcai/FunkJL21,DBLP:conf/ijcai/FunkJL22}). The focus
has been on ontologies in the tractable \DL{} and $\mathcal{EL}$ families and tree-shaped CQs such as ELQs ($\mathcal{EL}$-concepts) and ELIQs ($\mathcal{ELI}$-concepts). In fact, even under the closed-world semantics, only acyclic queries can be uniquely characterised and, equivalently, learned using membership queries in polynomial time~\cite{DBLP:journals/tods/CateD22}.

In this paper, we aim to understand how far these characterisability
and learnability results for atemporal queries mediated by an ontology
can be expanded to the temporal case. Temporal ontology-mediated query
answering provides a framework for accessing temporal data using a
background ontology. It has been investigated for about a decade---see, e.g.,~\cite{DBLP:conf/time/ArtaleKKRWZ17} for a survey---resulting in different settings and a variety of query and ontology languages~\cite{DBLP:journals/ws/BaaderBL15,DBLP:conf/ijcai/BorgwardtT15,DBLP:journals/jair/ArtaleKKRWZ22,DBLP:conf/ijcai/Gutierrez-Basulto16,DBLP:journals/tocl/ArtaleKRZ14,KR2020-79}. As a natural starting point, we assume that the background ontology holds at all times and does not admit temporal operators in its axioms. As a query language we consider a combination of ELIQs with linear temporal logic (\LTL) operators. First observations on unique characterisability and learnability of plain \LTL{} queries~\cite{DBLP:conf/kr/FortinKRSWZ22} showed that, even without ontologies, a restriction to so-called \emph{path queries} (defined below) is needed to obtain positive general and useful results. Our main contributions in this paper are general transfer theorems identifying abstract properties of query and ontology languages that are needed to lift unique characterisability and learnability from atemporal ontology-mediated queries and ontology-free path \LTL{} queries to temporalised domain queries mediated by a DL ontology. To facilitate the transfer, we begin by revisiting the atemporal case. Below is an overview of the obtained results.

%
\medskip
\noindent
\textbf{Atemporal case.}
We present and compare two approaches to finding unique (polysize) characterisations of atemporal queries mediated by an ontology: via frontiers and via split-partners (aka dualities). Both tools are developed under the condition that query containment in the respective atemporal DLs can be reduced to query evaluation. We call this condition \emph{containment reduction}. It applies to all fragments of the expressive DL $\mathcal{ALCHI}$ and more general FO-ontologies without equality as well as to \DL{} with functional roles. It ensures that whenever a unique characterisation of a query mediated by an ontology exists, there is also one with a single positive example in $E^+$. These tools yield two essentially optimal unique characterisability results: frontiers give polynomial-size characterisations of ELIQs mediated by an ontology in the DLs
$\DL_{\mathcal{H}}$ and $\DL^-_{\mathcal{F}}$~(Funk, Jung, and
Lutz \citeyear{DBLP:conf/ijcai/FunkJL21,DBLP:conf/ijcai/FunkJL22}), while split-partners provide exponential-size characterisations of ELIQs mediated by an $\mathcal{ALCHI}$ ontology and polysize characterisations of ELQs mediated by an RDFS ontology.

\medskip
\noindent
\textbf{Temporalising unique characterisations.} We now assume that temporal data instances are finite sets of facts (ground unary and binary atoms) timestamped by the moments $i \in \mathbb N$ they happened and that queries are equipped with temporal operators.
By combining the results from the atemporal case above with the techniques of~\cite{DBLP:conf/kr/FortinKRSWZ22}, we establish general transfer theorems on (polysize) unique characterisations of temporal queries mediated by a DL ontology.

We first consider the temporal operators $\nxt$ (at the next moment), $\Diamond$ (sometime later), and $\Diamondw$ (now or later) and define, given a class $\Q$ of atemporal queries (say, ELIQs), the family $\LTL\NDD_p(\Q)$ of \emph{path queries} of the form
\begin{equation*}\label{pathquery}
\q = \r_0 \land \op_1 (\r_1 \land \op_2 (\r_2 \land \dots \land \op_n \r_n) ),
\end{equation*}
where $\op_i \in \{\nxt, \Diamond, \Diamondw\}$ and $\r_i \in \Q$. These queries are evaluated at time 0. Even if $\Q$ consists of  conjunctions of atoms only and no ontology is present, not all queries in $\LTL\NDD_p(\Q)$ can be uniquely characterised. A typical example of a non-characterisable query in this class is $\q(x) = \Diamondw (A(x) \land B(x))$~\cite{DBLP:conf/kr/FortinKRSWZ22}. We first give an effective syntactic criterion for an $\LTL\NDD_p(\Q)$-query to be `safe' in the sense of admitting a  unique characterisation. Then we prove a fully general transfer theorem stating that if a DL $\Lmc$ admits containment reduction and (polysize) unique characterisations for $\Q$-queries mediated by an $\Lmc$-ontology, then so does the class of safe temporalised queries in  $\LTL\NDD_p(\Q)$. For example, this theorem yields polysize unique characterisations of safe queries in $\LTL\NDD_p(\text{ELIQ})$ mediated by a $\DL^-_{\mathcal{F}}$ or $\DL_{\mathcal{H}}$ ontology and exponential ones for safe $\LTL\NDD_p(\text{ELIQ})$-queries mediated by an $\mathcal{ALCHI}$ ontology.

Our second transfer result concerns temporal queries with the binary operator $\U$ (until) under the strict semantics and the family $\LTL\UN_{p}(\Q)$ of path queries of the form
\begin{equation*}\label{peepath}
\q = \r_0 \land (\el_1 \U (\r_1 \land ( \el_2 \U ( \dots (\el_n \U \r_n) \dots )))).
\end{equation*}
For its subclass of `peerless' queries, in which the $\r_i,\el_i \in
\Q$ do not contain each other \wrt a given ontology $\mathcal{O}$, we prove general transfer of unique characterisations provided that unique characterisations for the atemporal class $\Q$ can be obtained via split-partners. For example, this result gives exponential-size unique characterisations of peerless queries in $\LTL\UN_{p}(\text{ELIQ})$ mediated by any $\mathcal{ALCHI}$ ontology and polysize characterisations of peerless queries in $\LTL\UN_{p}(\text{ELQ})$ mediated by any RDFS ontology. We also show that the general transfer fails if frontier-based characterisations of queries in $\Q$ are used in place of split-partners.

\medskip \noindent \textbf{Temporalising learning.} We apply our
results on unique characterisations to learning a target query
$\q_T$, known only to a teacher, \wrt a given ontology \Omc in
Angluin's framework of exact learning. We allow the learner to use \emph{membership
queries}, which return in unit time whether a given example $(\Dmc,a)$ is
a positive one for $\q_T$ \wrt to \Omc.
Given that we always
construct example sets effectively, it is not difficult to show that
our exponential-size unique characterisations entail exponential
learning algorithms. We are, however, mainly interested in efficient algorithms formalised as polynomial time or
polynomial query learnability.

Obtaining such algorithms from polysize characterisations is more
challenging and we currently only know how this can be done if the
unique characterisation is based on polysize frontiers. Hence, we
focus on queries in $\LTL_p\NDD(\Q)$ and show that
polynomial query learnability transfers from $\Q$ to safe queries in
$\LTL_p\NDD(\Q)$ and that polytime learnability transfers if natural
additional conditions hold for $\Q$ and the considered ontology language.

Omitted details and proofs can be found in the appendix.


\section{Related Work}
The unique characterisation framework for temporal formulas, underpinning this paper, was originally introduced by Fortin et al.~(\citeyear{DBLP:conf/kr/FortinKRSWZ22}). Recently, it has been generalised to finitely representable transfinite words as data examples \cite{Sestic}, whose results are not directly applicable to the problems we are concerned with as the queries have no DL component and no ontology is present. It would be of interest to extend the techniques used by Sestic~(\citeyear{Sestic}) to the more general languages considered here.

%

The database and KR communities have been working on
identifying queries and concept descriptions from data
examples~\cite{DBLP:conf/icdt/StaworkoW15,DBLP:journals/jmlr/KonevLOW17,DBLP:journals/tods/CateDK13,DBLP:journals/ki/Ozaki20,DBLP:journals/tods/CateD22}. In reverse engineering of queries, the goal is typically to decide whether there is a query separating given positive and negative examples. Relevant work
includes~\cite{DBLP:journals/tods/ArenasD16,DBLP:conf/icdt/Barcelo017} under the closed world
and~\cite{DBLP:journals/ml/LehmannH10,GuJuSa-IJCAI18,DBLP:conf/ijcai/FunkJLPW19,DBLP:journals/ai/JungLPW22} under the open world assumption.

We are not aware of any work on exact learning of temporal formulas
save~\cite{DBLP:conf/aips/CamachoM19} and the related work on exact learning of finite automata starting
with~\cite{DBLP:journals/iandc/Angluin87}. In contrast, reverse
engineering of \LTL{}-formulas has recently received significant attention~\cite{lemieux2015general,DBLP:conf/fmcad/NeiderG18,DBLP:conf/aips/CamachoM19,DBLP:journals/corr/abs-2102-00876,ourijcai23}.

The use of unique characterisations for explaining and constructing schema mappings was promoted and investigated by Kolaitis~(\citeyear{kolaitis:LIPIcs:2011:3359}) and Alexe et al.\ (\citeyear{DBLP:journals/tods/AlexeCKT11}).



Unique characterisability of DL concepts under both closed and open world assumptions has recently been studied
by ten Cate, Koudijs, and Ozaki~(\citeyear{CateKO24}).


\section{Atemporal Ontologies and Queries}\label{Sec:prelims}

We assume that background knowledge about the object domain is given as a standard description logic ontology. This section recaps the relevant definitions.

As usual in DL, we work with any signature of unary and binary
predicate symbols, typically denoted $A,B$ and $P,R$, respectively. A
\emph{data instance} is any finite set $\A \ne \emptyset$ of
\emph{atoms} of the form $A(a)$ and $P(a,b)$ with \emph{individual
names} $a,b$, and also $\top(a)$, which simply says that $a$ exists.
We denote by $\ind(\A)$ the set of individuals in~$\A$ and by $P^-$
the \emph{inverse} of $P$, assuming that $P^-(a,b)\in\A$ iff
$P(b,a)\in\A$. Let $S$ range over binary predicates and their
inverses. A \emph{pointed data instance} is a pair $(\A,a)$ with $a
\in \ind(\A)$. \mbox{The \emph{size} $|\A|$ of $\A$ is the number of
symbols in it.}

In general, an \emph{ontology}, $\mathcal{O}$, is a finite set of first-order (FO) sentences in the given signature. Ontologies and data instances are interpreted in structures $\I=(\Delta^{\I},\cdot^{\I})$ with domain $\Delta^{\I} \ne \emptyset$, $a^{\I}\in \Delta^{\I}$, $\top^{\I} = \Delta^{\I}$, $A^{\I}\subseteq \Delta^{\I}$, and $P^{\I}\subseteq \Delta^{\I} \times \Delta^{\I}$. As usual in database theory, we assume that $a^{\I}\not=b^{\I}$ for distinct $a,b$; moreover, to simplify notation, we adopt the \emph{standard name assumption} and interpret each individual name by itself, i.e., $a^{\I} = a$. Thus, $\I$ is a \emph{model} of $\A$ if $a \in A^{\I}$ and $(a,b)\in P^{\I}$, for all $A(a)\in \A$ and $P(a,b)\in \A$.
We call $\I$ a \emph{model} of an ontology $\mathcal{O}$ if all
sentences in $\mathcal{O}$ are true in $\I$, and say that
$\mathcal{O}$ and $\A$ are \emph{satisfiable} if they have a common model.

The ontology languages we consider here are certain members of the
\DL{} family, $\mathcal{ALCHI}$, and $\mathcal{ELHIF}$; we define them
below as fragments of first-order logic.
\begin{description}\itemsep=0pt
\item[\rm$\DL_\mathcal{F}$]\!\cite{romans} aka $\DL^\mathcal{F}_\core$~\cite{dllite-jair09} allows axioms of the following forms:
\begin{multline}\label{dl-lite}
\forall x \, \big(B(x) \to B'(x)\big), \quad \forall x \, \big(B(x) \land  B'(x) \to \bot \big),\\
\forall x,y,z \, \big( S(x,y) \land S(x,z) \to (y=z) \big),
\end{multline}
where \emph{basic concepts} $B(x)$ are either $A(x)$ or $\exists S(x)=\exists y\, S(x,y)$. In DL parlance, the first two axioms in~\eqref{dl-lite} are written as $B \sqsubseteq B'$ and $B \sqcap B' \sqsubseteq \bot$, and the third one as $\ge 2\, S \sqsubseteq \bot$ or $\text{fun}(S)$, a \emph{functionality constraint} stating that relation $S$ is \emph{functional}.

\item[\rm$\DL_{\mathcal{F}}^{-}$]\!\cite{DBLP:conf/ijcai/FunkJL22}  is the fragment of $\DL_{\mathcal{F}}$, in which \emph{concept inclusions} (CIs) $B \sqsubseteq B'$ cannot have $B' = \exists S$ with functional $S^{-}$.

\item[\rm$\DL_\mathcal{H}$]\!\cite{romans} aka $\DL_\core^\mathcal{H}$~\cite{dllite-jair09} is obtained by disallowing the functionality constraints in $\DL_{\mathcal{F}}$ and adding axioms of the form
\begin{equation}\label{RI}
\forall x,y \, (S(x,y) \to S'(x,y))
\end{equation}
known as \emph{role inclusions} (RIs) and written as $S \sqsubseteq S'$.

\item[\rm RDFS]\hspace*{-2mm}\footnote{\url{https://www.w3.org/TR/rdf12-schema/}} has CIs between concept names, RIs between role names, and CIs of the forms $\exists P \sqsubseteq A$ or $\exists P^{-} \sqsubseteq A$ saying that the domain of $P$ and range of $P$ are in $A$, respectively.

\item[$\mathcal{ALCHI}$]\!\cite{DL-Textbook} has the same RIs as in~\eqref{RI} but more expressive CIs $\forall x\, (C_1(x) \to C_2(x))$, in which the \emph{concepts} $C_i$ are defined inductively starting from atoms $\top(x)$ and $A(x)$ and using the constructors $C(x) \land C'(x)$, $\neg C(x)$, and $\exists y \, (S(x,y) \land C(y))$---or $C \sqcap C'$, $\neg C$, and $\exists S. C$ in DL terms.


\item[\rm$\mathcal{ELHIF}$]\!\cite{DL-Textbook} has RIs~\eqref{RI}, functionality constraints, and CIs with concepts built from atoms and $\bot$ using $\land$ and $\exists y \, (S(x,y) \land C(y))$ only. \ELHI and $\ELIF$ are the fragments of $\mathcal{ELHIF}$ without functionality constraints and RIs, respectively.
\end{description}
We reserve $\mathcal{L}$ for denoting any of these ontology languages:\\
\centerline{\small
\begin{tikzpicture}

\node (rdfs) at (0, .6) {RDFS \ \ \ $\subset$};

\node (one) at (3.2,0) {$\DL_{\mathcal{F}}^{-} \ \subset \ \DL_{\mathcal{F}} \ \subset \ \ELIF$};

\node (elhif) at (6.6,0) {$\subset \ \mathcal{ELHIF}$};

\node (dlh) at (1.58,.6) {$\DL_\mathcal{H}$};
\node (elhi) at (5.05,.6) {$\mathcal{ELHI}$} edge[draw=none] node[sloped]{$\subset$} (dlh) edge[draw=none] node[sloped]{$\subset$} (6.2,0.2);

\node (alchi) at (6.8,.6) {$\mathcal{ALCHI}$} edge[draw=none] node[sloped]{$\subset$} (elhi);
\end{tikzpicture}
}

The most general query language over the object domain we consider
consists of \emph{conjunctive queries} (CQs, for short) $\q(x)$ with a single \emph{answer variable} $x$. We often think of
$\q(x)$ as the set of its atoms and denote by $\var(\q)$ and $\sig(\q)$ the sets of its individual variables and predicates symbols, respectively. We say that $\q(x)$ is \emph{satisfiable} \wrt an ontology $\mathcal{O}$ if $\mathcal{O} \cup \{\q(x)\}$ has a model.
%
%

Given a CQ $\q(x)$, an ontology $\mathcal{O}$, and a data instance $\A$, we say that $a\in \ind(\A)$ is a (\emph{certain})  \emph{answer to $\q$ over $\A$ \wrt $\mathcal{O}$} and write $\mathcal{O},\A\models \q(a)$  if $\I\models \q(a)$ for all models $\I$ of $\mathcal{O}$ and $\A$. Recall that $\emptyset,\A\models \q(a)$ iff there is function $h \colon \var(\q) \to \ind(\A)$ such that $h(x) = a$, $A(y) \in \q$ implies $A(h(y)) \in \A$, and $P(y,z) \in \q$ implies $P(h(y),h(z)) \in \A$. Such a function $h$ is called a \emph{homomorphism} from $\q$ to $\A$, written $h \colon \q \to \A$; $h$ is \emph{surjective} if $h(\var(\q)) = \ind(\A)$.

We say that a CQ $\q_1(x)$ is \emph{contained} in a CQ $\q_{2}(x)$ \wrt an ontology $\mathcal{O}$ and write $\q_{1}\models_{\mathcal{O}} \q_{2}$ if $\mathcal{O},\A\models\q_{1}(a)$ implies $\mathcal{O},\A\models\q_{2}(a)$, for any data instance $\A$ and any $a \in \ind(\A)$. If $\q_{1}\models_{\mathcal{O}} \q_{2}$ and $\q_{2}\models_{\mathcal{O}} \q_{1}$, we say that $\q_{1}$ and $\q_{2}$ are \emph{equivalent \wrt $\mathcal{O}$}, writing $\q_{1}\equiv_{\mathcal{O}} \q_{2}$. For $\TO = \emptyset$, we often write $\q_{1}\equiv \q_{2}$ instead of $\q_{1}\equiv_{\emptyset} \q_{2}$.
%

Two smaller query languages we need are $\mathcal{ELI}$-\emph{queries} (or ELIQs, for short) that can be defined by the grammar
\begin{equation*}
	\q  \quad := \quad \top \quad \mid \quad A \quad \mid \quad \exists S.\q \quad \mid \quad \q \land \q'
\end{equation*}
and $\mathcal{EL}$-\emph{queries} (or ELQs), which are ELIQs without inverses $P^-$. Semantically, an ELIQ $\q$ has the same meaning as the \emph{tree-shaped CQ} $\q(x)$ that is defined inductively starting from atoms $\top(x)$ and $A(x)$ and using the constructors $\exists y \, (S(x,y) \land \q(y))$, for a fresh $y$, and $\q(x) \land \q'(x)$. The only free (i.e., answer) variable in $\q$ is $x$.

We reserve $\Q$ for denoting a class of queries with answer variable $x$ such that whenever $\q_1,\q_2\in \Q$, then $\q_1 \land \q_2\in \Q$. Some of our results require restricting $\Q$ to a \emph{finite signature} $\sigma$: we denote by $\Q^\sigma$ the class of those queries in $\Q$ that are built from predicates in $\sigma$. The classes of all $\sigma$-ELIQs and $\sigma$-ELQs are denoted by ELIQ$^\sigma$ and ELQ$^\sigma$, respectively.

It will be convenient to include the `inconsistency query' $\bot$ into
all of our query classes. By definition, we have $\TO, \A \models
\bot(a)$ iff $\TO$ and $\A$ are unsatisfiable.


%
%

\section{Unique Characterisability}\label{Sec:uniqueChar}

An \emph{example set} is a pair $E = (E^+, E^-)$, where $E^+$ and $E^-$ are finite sets of pointed data instances $(\A,a)$. A CQ $\q(x)$ \emph{fits} $E$ \wrt $\mathcal{O}$ if $\mathcal{O},\A^+ \models \q(a^{+})$ and $\mathcal{O},\A^- \not\models \q(a^{-})$, for all $(\A^+,a^+) \in E^+$ and $(\A^-,a^-) \in E^-$. We say that $E$ \emph{uniquely characterises $\q$ \wrt $\mathcal{O}$ within a given class $\Q$ of queries} if $\q$ fits $E$ and $\q \equiv_{\mathcal{O}} \q'$, for every $\q' \in \Q$ that fits $E$.
Note that, in this case, $E^+ = \emptyset$ implies $\q \equiv_\mathcal{O} \bot$, and so $\q$ is not satisfiable \wrt $\mathcal{O}$.

We first observe that, for a large class of ontologies $\mathcal{O}$, including all those considered here, if $\q$ is uniquely characterised by some $E = (E^+, E^-)$ \wrt $\mathcal{O}$, then $\q$ has a unique characterisation of the form $E' = (\{(\hat{\q},a)\}, E^-)$ with a single positive example $(\hat{\q},a)$.
Say that an ontology $\mathcal{O}$ \emph{admits containment reduction} if, for any CQ $\q(x)$, there is a pointed data instance $(\hat{\q},a)$ such that the following conditions hold:
\begin{description}
\item[(cr$_1$)] $\q(x)$ is satisfiable \wrt $\mathcal{O}$ iff $\mathcal{O}$ and $\hat{\q}$ are satisfiable;

\item[(cr$_2$)] there is a surjective $h \colon \q \to \hat{\q}$ with $h(x) = a$;

\item[(cr$_3$)] if $\q(x)$ is satisfiable \wrt $\mathcal{O}$, then for every CQ $\q'(x)$, we have $\q\models_{\mathcal{O}} \q'$ iff $\mathcal{O},\hat{\q} \models \q'(a)$.
\end{description}
An ontology language $\mathcal{L}$ \emph{admits containment reduction} if every $\mathcal{L}$-ontology does. If the pointed data instance $(\hat{\q},a)$ is computable in polynomial time, for every $\mathcal{O}$ in $\mathcal{L}$, we say that $\mathcal{L}$ \emph{admits tractable containment reduction}. The next lemma 
illustrates this definition by a few concrete examples.

\begin{restatable}{lemma}{acr}\label{acr}
$(1)$ FO without equality admits tractable containment reduction\textup{;} in particular, $\mathcal{ALCHI}$ admits tractable containment reduction.

$(2)$ $\mathcal{ELIF}$  admits tractable containment reduction.

	
	
$(3)$ $\{ \ge 3 \,P \sqsubseteq \bot\}$
%
does not admit containment reduction.
\end{restatable}
\begin{proof}
	For (1), one can define $\hat{\q}$ as $\q$, with the variables
	regarded as individual names. To show (2), $\q$ has to be factorised first to ensure functionality; (3) is shown in appendix.
\end{proof}

%
%

It is readily checked that we have the following:

\begin{restatable}{lemma}{norm}\label{norm}
Suppose $\TO$ admits containment reduction and \mbox{$\q \in \Q$} is satisfiable \wrt $\mathcal{O}$, having a unique characterisation  $E=(E^+, E^-)$ \wrt $\TO$ within $\Q$. Then $E' = (\{(\hat{\q},a)\}, E^-)$ is a unique characterisation of $\q$ \wrt $\TO$ within $\Q$, too.
%
\end{restatable}
We use two ways of constructing unique characterisations: via  frontiers and via split-partners.
Let $\mathcal{O}$ be an ontology, $\Q$ a class of queries, and $\q \in \Q$ a satisfiable query \wrt $\mathcal{O}$. A \emph{frontier of $\q$ \wrt $\mathcal{O}$ within $\Q$} is a set $\mathcal{F}_{\q} \subseteq \Q$ such that
\begin{itemize}
\item for any $\q'\in \mathcal{F}_{\q}$, we have $\q\models_{\mathcal{O}} \q'$ and $\q' \not \models_{\mathcal{O}} \q$;

\item for any $\q'' \in \Q$, if $\q\models_{\mathcal{O}} \q''$,
	then either $\q''\models_{\mathcal{O}} \q$ or there is $\q'\in \mathcal{F}_{\q}$ with $\q'\models_{\mathcal{O}} \q''$.

\end{itemize}
(Note that if $\q \equiv_\mathcal{O} \top$, then $\mathcal{F}_{\q} = \emptyset$.) An ontology $\mathcal{O}$ is said to \emph{admit} (\emph{finite}) \emph{frontiers within $\Q$} if every $\q \in \Q$ satisfiable \wrt $\mathcal{O}$ has a (finite)  frontier \wrt $\mathcal{O}$ within $\Q$. Further, if such frontiers can be computed in polynomial time, we say that $\mathcal{O}$ \emph{admits polytime-computable frontiers.}

The next theorem follows directly from the definitions:

\begin{restatable}{theorem}{critO}\label{critO}
Suppose $\Q$ is a class of queries, an ontology $\mathcal{O}$ admits containment reduction, $\q \in \Q$ is satisfiable \wrt $\mathcal{O}$, and $\mathcal{F}_{\q}$ is a finite frontier of $\q$ \wrt $\mathcal{O}$ within $\Q$. Then $(\{(\hat{\q},a)\},\{(\hat{\r},a) \mid \r \in \mathcal{F}_{\q}\})$  is a unique characterisation of $\q$ \wrt $\mathcal{O}$ within $\Q$.
\end{restatable}

As shown by Funk, Jung, and Lutz~(\citeyear{DBLP:conf/ijcai/FunkJL22}), the two main ontology languages that admit polytime-computable frontiers within ELIQ are $\DL_{\mathcal{H}}$ and $\DL_{\mathcal{F}}^{-}$, whereas $\DL_{\mathcal{F}}$ itself does not admit finite ELIQ-frontiers. By Theorem~\ref{critO} and Lemma~\ref{acr}, we then obtain:

\begin{theorem}\label{polyuniqDL-Lite}
If an ELIQ $\q$ is satisfiable \wrt a $\DL_{\mathcal{H}}$ or $\DL_{\mathcal{F}}^{-}$ ontology $\mathcal{O}$, then $\q$ has a polysize unique characterisation \wrt $\mathcal{O}$ within ELIQ.
\end{theorem}

We next introduce split-partners aka dualities~\cite{McKenzie1972EquationalBA,DBLP:journals/tods/CateD22}. Let $\sigma$ be a finite signature, $\Q^\sigma$ a class of $\sigma$-queries, $\mathcal{O}$ a $\sigma$-ontology, and $\FQ \subseteq \Q^\sigma$ a finite set of queries. A set $\mathcal{S}(\FQ)$ of pointed data instances $(\mathcal{A},a)$ is called a \emph{split-partner for $\FQ$ \wrt $\mathcal{O}$ within $\Q^\sigma$} if, for all $\q' \in \Q^\sigma$, we have
%
\begin{multline}\label{eq:split-def}
\mathcal{O},\mathcal{A}\models \q'(a) \text{ for some $(\mathcal{A},a) \in \mathcal{S}(\FQ)$} \quad \text{iff} \\
\q'\not\models_{\mathcal{O}} \q \text{ for all $\q\in \FQ$.}
\end{multline}
Say that an ontology language $\Lmc$ \emph{has general split-partners within $\Q^\sigma$} if all finite sets of $\Q^\sigma$-queries have split partners \wrt any $\Lmc$-ontology in $\sigma$. If this holds for all singleton subsets of $\Q^\sigma$, we say that $\Lmc$ \emph{has split-partners within $\Q^\sigma$}.

We illustrate the notion of split-partner by a few examples, the last of which shows that, without the restriction to a finite signature $\sigma$, split-partners almost never exist.

\begin{example}\label{ex:22}\em
$(i)$ Let $\mathcal{O}$ be any ontology such that $\mathcal{O}$ and $\A$  are satisfiable for all data instances $\A$, say, $\mathcal{O} =\{A \sqsubseteq B\}$. Let $\Q^\sigma$ be any class of $\sigma$-CQs, for some signature $\sigma$. Then the split-partner $\mathcal{S}_{\bot}$ of the query $\bot$ \wrt $\mathcal{O}$ within $\mathcal{Q}^\sigma$ is 
$$
\mathcal{S}_{\bot} = \{\mathcal{B}_{\sigma}\}, \text{for } \mathcal{B}_{\sigma}= \{R(a,a) \mid R\in \sigma\} \cup \{A(a) \mid A\in \sigma\}.
$$
(Here and below we drop $a$ from $(\A,a)$ if $\ind(\A) = \{a\}$.)
Clearly, $\TO, \mathcal{B}_{\sigma} \models \q$, for any $\q \in \Q^\sigma$ different from $\bot$.

$(ii)$ For $\TO = \{A \sqcap B \sqsubseteq \bot\}$ and $\sigma = \{A,B\}$, we have $\mathcal{S}_\bot = \{ \{A(a)\}, \ \{B(a)\} \}$.

$(iii)$ There does not exist a split-partner for $\FQ =\{A\}$ \wrt the empty ontology $\mathcal{O}$ within ELIQ. To show this, observe that  $B\not\models_{\mathcal{O}}A$ for any unary predicate $B\not=A$. Hence,
as any data instance $\A$ is finite, there is no finite set $\mathcal{S}(\{A\})$ satisfying \eqref{eq:split-def}.
\end{example}

In contrast, for frontiers and unique characterisations, restrictions
to sets of predicates containing all symbols in the query and ontology
do not make any difference. Indeed, let $\sigma$ be
  the signature of $\mathcal{O}$ and $\q$. Then, for any class $\Q$
  of queries, a set $\mathcal{F}_{\q}$ is a frontier for $\q$ \wrt
  $\mathcal{O}$ within $\Q$ iff it is a frontier for $\q$ \wrt
  $\mathcal{O}$ within the restriction of $\Q$ to $\sigma$.
The same holds for unique characterisations $E$ of $\q$ \wrt $\mathcal{O}$.

%
%
%
The following result is proved (in the appendix) using a construction from the reduction of ontology-mediated query answering to constraint satisfaction~\cite{DBLP:journals/tods/BienvenuCLW14}.
\begin{restatable}{theorem}{alchi}\label{alchi-split}
$\mathcal{ALCHI}$ has general split-partners within ELIQ$^\sigma$ that  can be computed in exponential time.
%
\end{restatable}

For ELQs, we can construct general split-partners \wrt RDFS ontologies in polynomial time, provided that the number of input queries is bounded. The proof generalises the construction of split-partners for queries in ELQ \wrt to the empty ontology in~\cite{DBLP:conf/kr/FortinKRSWZ22,DBLP:conf/ijcai/CateFJL23}.

\begin{restatable}{theorem}{elsplit}\label{thm:ELsplit}
	Let $\sigma$ be a signature, $\TO$ a $\sigma$-ontology in RDFS, and $n>0$. For any set $\Theta\subseteq \text{ELQ}^{\sigma}$ with $|\Theta|\leq n$, one can compute in polynomial time a
	split-partner $\mathcal{S}(\Theta)$ of $\Theta$ \wrt $\TO$ within ELQ$^{\sigma}$.
\end{restatable}

Here is our second sufficient characterisability condition:

\begin{restatable}{theorem}{thmcritone}\label{thm:crit1}
Suppose $\Q$ is a class of queries, an ontology $\mathcal{O}$ admits containment reduction, $\q\in \Q$ is satisfiable \wrt $\mathcal{O}$, and $\sigma$ contains the predicate symbols in $\q$ and $\mathcal{O}$. If $\mathcal{S}_{\q}$ is a split-partner for $\{\q\}$ \wrt $\mathcal{O}$ within $\mathcal{Q}^\sigma$, then $(\{(\hat{\q},a)\}, \mathcal{S}_{\q})$  is a unique characterisation of $\q$ \wrt $\mathcal{O}$ within $\Q$.
\end{restatable}
%
%
As a consequence of Theorems~\ref{alchi-split},~\ref{thm:ELsplit},~\ref{thm:crit1} and Lemma~\ref{acr}, we obtain the following:

\begin{theorem}\label{thm:alchi}
If a query $\q \in \text{ELIQ}^\sigma$ is satisfiable \wrt an $\mathcal{ALCHI}$-ontology $\mathcal{O}$ in a signature $\sigma$, then $\q$ has a unique characterisation \wrt $\mathcal{O}$ within ELIQ$^\sigma$. 
\end{theorem}


The sufficient conditions of Theorems~\ref{critO}  and~\ref{thm:crit1} use the notions of frontier and split-partner, respectively. (Notice that both of them are applicable to ELIQs \wrt $\DL_{\mathcal{H}}$ and RDFS ontologies; however, only split-partners will give us polysize unique characterisations of temporalised ELQs \wrt RDFS ontologies in  Theorem~\ref{thmforU} $(ii)$, Section~\ref{sec:seven}.) We now show examples of queries and ontologies having frontiers but no split-partners and vice versa.
The query witnessing that frontiers can exist where split-partners do not exist provides a counterexample even if one admits \emph{CQ-frontiers}, frontiers containing not only ELIQs but also CQs and defined in the obvious way in the appendix.
\begin{restatable}{theorem}{counterexone}\label{counterex1}
$\mathcal{EL}$ does not admit finite CQ-frontiers within ELIQ.
\end{restatable}
\begin{proof}
The query $\q=A\wedge B$ does not have a finite CQ-frontier \wrt the ontology $\mathcal{O} = \{ A \sqsubseteq \exists R.A, \ \exists R.A \sqsubseteq A\}$
within ELIQs.
\end{proof}
\begin{example}\label{count2}\em
Observe that the following set of pointed data instances is a split-partner of $\{\q\}$ \wrt $\mathcal{O}$ from the proof of  Theorem~\ref{counterex1} within ELIQ$^{\{A,B,R\}}$; here all arrows are assumed to be labelled by $R$:\\
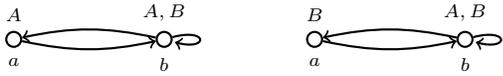
\centerline{
\begin{tikzpicture}[thick, transform shape]
\tikzset{every state/.style={minimum size=0pt}}
%
%
%
\node[circle, draw=black,inner sep=2pt, label = below:{\scriptsize $a$},  label = above:{\scriptsize $A$}] at (0,0) (one) {};
\node[circle, draw=black,inner sep=2pt, label = below:{\scriptsize $b$},  label = above:{\scriptsize $A, B$}] at (2,0) (two) {};
\path (one) edge[->, bend right=12]  (two);
\path (two) edge[->, bend right=12]  (one);
%
%
%
\path (two) edge[->, loop right] (two);
%
%
\node[circle, draw=black,inner sep=2pt, label = below:{\scriptsize $a$},  label = above:{\scriptsize $B$}] at (4,0) (one) {};
\node[circle, draw=black,inner sep=2pt, label = below:{\scriptsize $b$},  label = above:{\scriptsize $A, B$}] at (6,0) (two) {};
\path (one) edge[->, bend right=12]  (two);
\path (two) edge[->, bend right=12]  (one);%
\path (two) edge[->, loop right] (two);
\end{tikzpicture}
}
\end{example}

\begin{restatable}{theorem}{counterexampletwo}\label{counterexample}
There exist a $\DL^-_\mathcal{F}$ ontology $\TO$, a query $\q$ and a signature $\sigma$ such that $\{\q\}$ does not have a finite split-partner \wrt $\TO$ within ELIQ$^\sigma$.
\end{restatable}
\begin{proof}
Let $\mathcal{O}=\{ \text{fun}(P), \text{fun}(P^-),B \sqcap \exists P^{-}\sqsubseteq \bot\}$ and $\q=A$. Then $Q = \{\q\}$ does not have a finite split-partner \wrt $\mathcal{O}$ within ELIQ$^{\{A,B,P\}}$.
\end{proof}

Observe that $\{\top\}$ is a frontier for $A$ \wrt $\mathcal{O}$ from the proof of Theorem~\ref{counterexample} within ELIQ and that we can combine the two proofs above to also refute the natural conjecture
that frontiers and splittings together provide a `universal tool' for constructing unique characterisations.


\section{Temporal Data and Queries}
\label{sec:tempintro}

We now extend the definitions of Sections~\ref{Sec:prelims} and~\ref{Sec:uniqueChar} by adding a temporal dimension to the domain data and queries mediated by an ontology.
Our definitions generalise those of~\cite{DBLP:conf/kr/FortinKRSWZ22}, where the ontology-free case was first considered.

A \emph{temporal data instance}, denoted $\D$, is a finite sequence $\mathcal{A}_{0},\dots,\mathcal{A}_{n}$ of data instances, where each $\A_i$ comprises the facts with timestamp $i$. We assume all $\ind(\A_{i})$ to be the same, adding $\top(a)$ to $\A_i$ if needed, and set $\ind(\D) = \ind(\A_{0})$. The \emph{length} of $\D$ is $\len(\D) = n$ and the \emph{size} of $\D$ is $|\D| = \sum_{i\le n} |\A_i|$.
Within a temporal $\sigma$-data instance, we often denote by $\emptyset$ the instance \mbox{$\{\top(a) \mid a \in \ind(\D)\}$.}


Temporal queries for accessing temporal data instances we propose here are built from domain queries (with one implicit answer variable $x$)  in a given class $\Q$ (say, ELIQs) using $\wedge$ and the (future-time) temporal operators of the standard linear temporal logic \LTL{} over the time flow $(\mathbb N,<)$: unary $\nxt$ (next time), $\Diamond$ (sometime later), $\Diamondw$ (now or later), and binary $\U$ (until); see below for the precise semantics.
The class of such temporal queries that only use the operators from a set $\Phi \subseteq \{\nxt, \Diamond, \Diamondw, \U\}$ is denoted by $\LTL^\Phi(\Q)$.
%
The class $\LTL\NDD_p(\Q)$ comprises  \emph{path queries} of the form
\begin{equation}\label{dnpath}
\q = \r_0 \land \op_1 (\r_1 \land \op_2 (\r_2 \land \dots \land \op_n \r_n) ),
\end{equation}
where $\op_i \in \{\nxt, \Diamond, \Diamondw\}$ and $\r_i \in \Q$; \emph{path queries} in $\LTL\UN_p(\Q)$ take the form
\begin{equation}\label{upath}
\q = \r_0 \land (\el_1 \U (\r_1 \land ( \el_2 \U ( \dots (\el_n \U \r_n) \dots )))),
\end{equation}
where $\r_i \in \Q$ and either $\el_i \in \Q$ or $\el_i = \bot$.
We use $\C$ to refer to classes of temporal queries.
The \emph{size} $|\q|$ of $\q$ is the number of symbols in $\q$; the \emph{temporal depth} $\dep(\q)$ of $\q$ is the maximum number of nested temporal operators in $\q$. %




An (atemporal) ontology $\mathcal{O}$ and temporal data instance $\D=\A_{0},\ldots,\A_{n}$ are \emph{satisfiable} if $\mathcal{O}$ and $\A_{i}$ are satisfiable for each $i\leq n$. For satisfiable  $\mathcal{O}$ and $\D$, the \emph{entailment relation} $\mathcal{O},\D,\ell,a \models \q$ with $\ell \in \mathbb N$ and $a \in \ind(\D)$ is defined by induction as follows, where $\A_\ell = \emptyset$, for $\ell > n$:
\begin{align*}
	&\mathcal{O},\D,\ell,a \models \q \text{ iff $\mathcal{O},\A_{\ell}\models \q(a)$, for any $\q \in \Q$},\\
	&\mathcal{O},\D,\ell,a \models \q_1 \land \q_2 \text{ iff } \mathcal{O},\D,\ell,a \models \q_i, \text{ for $i=1,2$},\\
	&\mathcal{O},\D,\ell,a \models \nxt \q \ \text{ iff } \ \mathcal{O},\D,\ell+1,a \models \q,\\
	&\mathcal{O},\D,\ell,a \models \Diamond \q \ \text{ iff } \ \mathcal{O},\D,m,a \models \q, \text{ for some $m > \ell$}, \\
	&\mathcal{O},\D,\ell,a \models \Diamondw \q \ \text{ iff } \ \mathcal{O},\D,m,a \models \q, \text{ for some $m \ge \ell$}, \\
	&\mathcal{O},\D,\ell,a \models \q_1 \U \q_2 \text{ iff }  \mathcal{O},\D,m,a \models \q_2,\! \text{ for some } m > \ell, \\
	& \hspace*{2.2cm} \text{ and } \mathcal{O},\D,k,a \models \q_1, \text{ for all $k$, $\ell < k < m$}.
\end{align*}
If $\mathcal{O}$ and $\D$ are not satisfiable, we set $\mathcal{O},\D,\ell,a \models \q$ to hold for all $\q$, $\ell$ and $a$. Our semantics follows the well established epistemic approach to evaluating temporal queries; see~\cite{DBLP:conf/ijcai/CalvaneseGLLR07,DBLP:journals/jair/ArtaleKKRWZ22} and references therein. The alternative classical Tarski semantics based on temporal interpretations is equivalent to our semantics for all \emph{Horn ontologies} whose FO-translations belong to the Horn fragment of first-order logic~\cite{modeltheory}, and so for all DLs we consider here except $\mathcal{ALCHI}$. A detailed discussion of the relationship between the two semantics is given in the appendix.

We call a temporal query $\q$ \emph{satisfiable} \wrt an ontology $\mathcal{O}$ if $\mathcal{O} \cup \{\q\}$ has a model. Note that $\q$ of the form~\eqref{dnpath} or \eqref{upath} is satisfiable
	\wrt $\mathcal{O}$ iff all $\r_{i}$ in $\q$ are satisfiable \wrt $\mathcal{O}$. 
	

By an \emph{example set} we now mean a pair $E = (E^+, E^-)$ of finite sets
$E^+$ and $E^-$ of pointed temporal data instances $\D,a$ with $a
\in \ind(\D)$. We say that a query $\q$ \emph{fits} $E$ \wrt $\mathcal{O}$
if $\mathcal{O},\D^+,0,a^+ \models \q$ and $\mathcal{O},\D^-,0,a^-
\not\models \q$, for all $(\D^+,a^+) \in E^+$ and $(\D^-,a^-) \in E^-$. As before, $E$ \emph{uniquely characterises $\q$ \wrt $\mathcal{O}$ within a class $\mathcal{C}$ of temporal queries} if $\q$ fits $E$ \wrt $\mathcal{O}$ and every $\q' \in \mathcal{C}$ fitting $E$ \wrt $\mathcal{O}$ is
equivalent to $\q$ \wrt $\mathcal{O}$.

The following lemma shows that typically queries that are not satisfiable \wrt the ontology in question cannot be uniquely characterised.

\begin{lemma}\label{lem:sat}
Let $\TO$ be an ontology and $\C$ a class of temporal queries containing, for any $n>0$,  a satisfiable \wrt $\TO$ query of temporal depth $\geq n$. If a query $\q$ is not satisfiable \wrt $\TO$, then $\q$ is not uniquely characterisable \wrt $\TO$ within $\mathcal{C}$.
\end{lemma}

To prove Lemma~\ref{lem:sat}, assume that $\q$ is not satisfiable \wrt $\TO$ but $E=(E^{+},E^{-})$ uniquely characterises $\q$ \wrt $\TO$ within $\mathcal{C}$. In this case $E^{+}=\emptyset$. Let $n$ be the maximal length of instances in $E^{-}$. Then any $\q'\in  \mathcal{C}$ of temporal depth $>n$ fits $E$, which is a contradiction.

In what follows, we mostly exclude such unsatisfiable queries from  consideration. 

Suppose $\mathcal{C}$ is a class of queries and $\mathcal{O}$ an ontology. If each $\q \in \mathcal{C}$ satisfiable \wrt $\TO$ is uniquely characterised by some $E$ \wrt $\mathcal{O}$ within $\mathcal{C}' \supseteq \mathcal{C}$, we say that $\mathcal{C}$ is \emph{uniquely} \emph{characterisable \wrt $\mathcal{O}$ within $\mathcal{C}'$}.
Let $\mathcal{C}^n$ be the set of queries in $\mathcal{C}$ of temporal depth $\le n$. We say that $\mathcal{C}$ is \emph{polysize characterisable \wrt $\mathcal{O}$ for bounded temporal depth} if there is a polynomial $f$ such that every $\q \in \mathcal{C}^n$  is characterised by some $E$ of size $\le f (n)$ within $\mathcal{C}^n$, $n \in \mathbb N$.

Note that $\Diamond\q \equiv \nxt\Diamondw \q$, so $\Diamond$ does not add any expressive power to $\LTL_p\NDD(\Q)$ and $\LTL_p\NDD(\Q) = \LTL_p\NDR(\Q)$; however, $\LTL_p\ND(\Q) \subsetneqq \LTL_p\NDD(\Q)$.
We also observe that our temporal query languages do not admit containment reduction as, for example, there is no temporal data instance $\hat\q$ for $\q = \nxt (A \wedge \Diamond B)$ because  it will have to fix the number of steps between 0 and the moment of time  where $B$ holds.

We next prove general theorems lifting unique characterisability from domain queries considered above and ontology-free \LTL{} queries of~\cite{DBLP:conf/kr/FortinKRSWZ22} to temporal queries mediated by a DL ontology.


\section{Unique Characterisations in $\LTL_p\NDD(\Q)$}
\label{sec:six}

The aim of this section is to give a criterion of (polysize) unique characterisability of temporal queries in the class $\LTL_p\NDD(\Q)$ under certain conditions on the ontology and on the class $\Q$ of domain queries.
It will be convenient to represent queries $\q$ of the form \eqref{dnpath} as a sequence
\begin{equation}\label{eq:cqform}
\q = \r_{0}(t_0), R_{1}(t_0,t_1), \dots,  R_{m}(t_{m-1},t_m),\r_{m}(t_m),
\end{equation}
where $R_{i}\in \{\suc,<,\leq\}$, $\suc(t,t')$ stands for $t' = t+1$, and $\tvar = \{t_0,\dots, t_m\}$ are variables over the timeline $(\mathbb N,<)$.
%
\begin{example}\label{q1}\em
Below are a temporal query $\q$ and its representation of the form~\eqref{eq:cqform}:
\begin{multline}
\q = \exists P.B \land \nxt(\exists P.A \land \Diamond A)\ \ \leadsto \ \ \\
\exists P.B(t_{0}), \suc(t_{0},t_{1}), \exists P.A(t_{1}), (t_{1}< t_{2}), A(t_{2})
\end{multline}
with $\tvar(\q)=\{t_{0},t_{1},t_{2}\}$. 
\end{example}

We divide $\q$ of the form~\eqref{eq:cqform} into \emph{blocks} $\q_i$ such that
\begin{align}\label{fullq2appendix}
&	\q = \q_{0} \mathcal{R}_{1} \q_{1} \dots \mathcal{R}_{n} \q_{n},
\end{align}
where $\mathcal{R}_{i} = R_{1}^{i}(t_{0}^{i},t_{1}^{i}) \dots  R_{n_{i}}^{i}(t_{n_{i}-1}^{i},t_{n_{i}}^{i})$, \mbox{$R_{j}^{i}\in \{<,\leq\}$} and
\begin{equation}\label{subqi}
\q_{i} =  \r_{0}^{i}(s_{0}^{i}) \suc (s_{0}^{i},s_{1}^{i})
\dots \suc(s_{k_{i}-1}^{i},s_{k_{i}}^{i}) \r_{k_{i}}^{i}(s_{k_{i}}^{i})
\end{equation}
with $s_{k_{i}}^{i}=t_{0}^{i+1}$, $t_{n_{i}}^{i}=s_{0}^{i}$. If $k_{i}=0$, the block $\q_{i}$ is called \emph{primitive}.

\begin{example}\em
The query $\q$ from Example~\ref{q1} has two blocks
$$
\q_{0} =\exists P.B(t_{0}), \suc(t_{0},t_{1}), \exists P.A(t_{1}) \quad \text{and} \quad \q_{1}=A(t_{2})
$$
connected by $(t_1 < t_2)$. It contains one primitive block, $\q_{1}$.
\end{example}

Suppose we are given an ontology $\mathcal{O}$ and a class $\Q$ of domain queries. Then
a primitive block $\q_{i}=\r_{0}^{i}(s_{0}^{i})$ with $i>0$ in $\q$  of the form \eqref{fullq2appendix} is called a \emph{lone conjunct \wrt $\mathcal{O}$ within $\Q$} if $\r_{0}^{i}$ is \emph{meet-reducible \wrt $\mathcal{O}$ within $\Q$} in the sense that
there are queries $\r_{1},\r_{2} \in \Q$ such that
$\r \equiv_{\mathcal{O}} \r_{1} \land \r_{2}$ and $\r
\not\equiv_{\mathcal{O}} \r_{i}$, for $i=1,2$. Lone
conjuncts and their impact on unique characterisability are illustrated by the next example.


\begin{example}\label{ex:3}\em
The query $\Diamond A$, which is represented by the sequence $\top(t_0), (t_0 < t_1), A(t_1)$,  does not have any lone conjuncts \wrt the empty ontology within ELIQ, but $A$ is a lone conjunct of $\Diamond A$ \wrt $\mathcal{O}=\{A \equiv B \sqcap C\}$ within ELIQ.
		
The query $\q = \Diamond A$ is uniquely characterised \wrt the empty ontology within $\LTL_p\NDD(\text{ELIQ})$ by the example set $E = (E^+, E^-)$, where $E^+$ contains two temporal data instances $\emptyset,\{A\}$ and $\emptyset,\emptyset,\{A\}$ and $E^-$ consists of one instance $\{A\}$.
However, $\q = \Diamond A$ cannot be uniquely characterised \wrt $\mathcal{O}=\{A \equiv B \sqcap C\}$ within $\LTL_p\NDD(\text{ELIQ})$ as it cannot be separated from queries of the form
$$
\Diamond(B \wedge \Diamond_{r}(C \wedge \Diamond_{r} (B \land \Diamond_{r} (C \land \Diamond_{r}(\dots)))))
$$
by a finite example set.
Observe also that $A$ is a lone conjunct in $\q'=\Diamond (A \land \Diamondw D)$ \wrt $\mathcal{O}'=\mathcal{O}\cup \{D\sqsubseteq A\}$ but, for the simplification $\q''=\Diamond D$ of $\q'$, we have $\q'' \equiv _{\mathcal{O}'} \q'$ and $\q''$ does not have any lone conjuncts \wrt $\mathcal{O}'$.
\end{example}

Example~\ref{ex:3} shows that the notion of lone conjunct depends on the presentation of the query.
To make lone conjuncts semantically meaningful, we introduce a normal form. Given an ontology $\mathcal{O}$ and a query $\q$ of the form \eqref{fullq2appendix}, we say that $\q$ is in \emph{normal form \wrt $\mathcal{O}$} if the
following conditions hold:
\begin{description}
	\item[(n1)] $\r_{0}^{i}\not\equiv_{\mathcal{O}}\top$ if $i>0$, and $\r_{k_{i}}^{i}\not\equiv_{\mathcal{O}}\top$ if either $i>0$ or $k_{i}>0$
	(thus, of all the first/last $\r$ in a block only $\r_0^0$ can be trivial);
	
	\item[(n2)] each $\mathcal{R}_{i}$ is either a single $t_{0}^{i}\leq t_{1}^{i}$ or a sequence of $<$;
	
	\item[(n3)] $\r_{k_{i}}^{i}\not\models_{\mathcal{O}} \r_{0}^{i+1}$ if $\q_{i+1}$ is primitive and $\mathcal{R}_{i+1}$ is $\le$;
	
	\item[(n4)] $\r_{0}^{i+1}\not\models_{\mathcal{O}}\r_{k_{i}}^{i}$ if $i>0$, $\q_{i}$ is primitive and $\mathcal{R}_{i+1}$ is $\le$;
	
\item[(n5)] $\r_{k_{i}}^{i} \land \r_{0}^{i+1}$ is satisfiable \wrt $\mathcal{O}$ whenever $\mathcal{R}_{i+1}$ is $\le$.
\end{description}

\begin{restatable}{lemma}{lemquerynormalform}\label{lem:query-normalform}
Let $\mathcal{O}$ be an FO-ontology \textup{(}possibly with
$=$\textup{)}. Then every query $\q\in \LTL_p\NDD(\Q)$ is equivalent
\wrt $\mathcal{O}$ to a query in normal form of size at most $|\q|$
and of temporal depth not exceeding $\dep(\q)$.
This query can be computed in polynomial time if containment between queries in $\Q$ \wrt $\mathcal{O}$ is decidable in polynomial time.
\end{restatable}

Note that containment between queries in $\Q =
  \text{ELIQ}$ is decidable in polynomial time for
  $\DL_{\cal{F}}$-ontologies $\mathcal{O}$ but not for
  $\DL_{\mathcal{H}}$-ontologies \textup{(}unless $\textup{P} =
  \textup{NP}$\textup{)}~\cite{DBLP:conf/kr/KikotKZ12}. 

We call a query $\q\in \LTL_p\NDD(\Q)$ \emph{safe \wrt $\mathcal{O}$}
if it is equivalent \wrt $\mathcal{O}$ to an $\LTL_p\NDD(\Q)$-query in
normal form that has no lone conjuncts. It follows
  from the proof of Theorem~\ref{thm:first} given in the appendix
  that, for any query satisfiable \wrt $\mathcal{O}$, the normal form is unique modulo equivalence of its constituent domain queries.

%
%

We are now in a position to formulate the main result of this section.

\begin{restatable}{theorem}{thmfirst}\label{thm:first}
Suppose an ontology $\mathcal{O}$ admits containment reduction and $\Q$ is
a class of domain queries that is uniquely characterisable \wrt $\mathcal{O}$. Then the following hold\textup{:}

$(i)$ Let $\q\in \LTL_p\NDD(\Q)$ be satisfiable \wrt $\TO$. Then $\q$ is uniquely characterisable within $\LTL_p\NDD(\Q)$ \wrt $\mathcal{O}$ iff $\q$ is safe \wrt $\mathcal{O}$.

$(ii)$ If $\mathcal{O}$ admits polysize characterisations within $\Q$, then those queries that are uniquely characterisable within $\LTL_p\NDD(\Q)$ \wrt $\mathcal{O}$ are actually polysize characterisable within $\LTL_p\NDD(\Q)$ \wrt $\mathcal{O}$.
		
$(iii)$ $\LTL_p\NDD(\Q)$ is polysize characterisable \wrt $\mathcal{O}$ for bounded temporal depth if $\mathcal{O}$ admits polysize unique characterisations within $\Q$.
		
$(iv)$ $\LTL_p\ND(\Q)$ is uniquely characterisable \wrt $\mathcal{O}$. It is
polysize characterisable \wrt $\mathcal{O}$ if $\mathcal{O}$ admits polysize unique characterisations within $\Q$.
\end{restatable}

A detailed proof of Theorem~\ref{thm:first} is given in the appendix.  To explain the intuition behind it, we show and discuss the positive and negative examples that provide the unique characterisation required for $(i)$.
Suppose $\mathcal{O}$ admits containment reduction and $\Q$ is a class of domain queries with a unique
characterisation $(\{\hat{\r}\},\mathcal{N}_{\r})$ of $\r\in \Q$ \wrt $\mathcal{O}$ within $\Q$.
Assume that  $\q\in \LTL_p\NDD(\Q)$ in normal form \wrt $\mathcal{O}$ takes the form~\eqref{fullq2appendix} with $\q_i$ of the form~\eqref{subqi}.
We define an example set $E=(E^{+},E^{-})$ characterising $\q$ under the assumption that $\q$ has no lone conjuncts \wrt $\mathcal{O}$. Let $b$ be the number of ocurrences of $\nxt$ and $\Diamond$ in $\q$ plus 1.
For every block $\q_i$ of the form~\eqref{subqi}, let $\hat{\q}_{i}$ be the temporal data instance
\begin{align*}
	\hat{\q}_{i}= \hat{\r}_{0}^{i} \hat{\r}_{1}^{i} \dots \hat{\r}_{k_{i}}^{i} .
\end{align*}
For any two blocks $\q_i, \q_{i+1}$ such that $\r_{k_{i}}^{i}\wedge \r_{0}^{i+1}$ is satisfiable \wrt $\mathcal{O}$, we take the temporal data instance
\begin{align*}
	\hat{\q}_{i} \Join \hat{\q}_{i+1} = \hat{\r}_{0}^{i} \dots \hat{\r}_{k_{i-1}}^{i}  \widehat{\r_{k_{i}}^{i} \wedge \r_{0} ^{i+1}}  \hat{\r}_{1}^{i+1}\dots \hat{\r}_{k_{i+1}}^{i+1} .
\end{align*}
Now, the set $E^{+}$ contains the data instances given by
\begin{itemize}
	\item[--] $\D_{b} = \qw_{0} \emptyset^{b} \dots \qw_{i} \emptyset^{b} \qw_{i+1} \dots \emptyset^{b} \qw_{n}$,
	
	\item[--] $\D_{i} = \qw_{0} \emptyset^{b} \dots (\qw_{i} \! \Join \! \qw_{i+1}) \dots \emptyset^{b} \qw_{n}$, \  if $\mathcal{R}_{i+1}$ is $\leq$ and
	
	\item[--] $\D_{i} = \qw_{0} \emptyset^{b} \dots \qw_{i} \emptyset^{n_{i+1}} \qw_{i+1} \dots \emptyset^{b} \qw_{n}$, \ otherwise.
\end{itemize}
Here, $\emptyset^b$ is a sequence of $b$-many $\emptyset$ and similarly for $\emptyset^{n_{i+1}}$ (intuitively, these `paddings' of multiple $\emptyset$s are needed to ensure that queries fitting the examples have the same block structure as the target query).
By the definition of $\hat{\r}$ using containment reduction, it follows that  $\TO,\D,0,a\models \q$, for all $\D\in E^{+}$. Intuitively, the data instances in $E^{+}$ force any query that is entailed to be divided into blocks in a similar way as $\q$. The set $E^{-}$ contains all data instances of the form
\begin{itemize}
	\item[--] $\D_i^- = \qw_{0} \emptyset^{b} \dots \qw_{i} \emptyset^{n_{i+1} - 1} \qw_{i+1} \dots \emptyset^{b} \qw_{n}$, \  if $n_{i+1} > 1$,
	
	\item[--] $\D^-_i = \qw_{0} \emptyset^{b} \dots \qw_{i} \! \Join \! \qw_{i+1} \dots \emptyset^{b} \qw_{n}$, \ if $\mathcal{R}_{i+1}$ is a single $<$ and $\r_{k_{i}}^{i}\wedge\r_{0}^{i+1}$ is satisfiable \wrt $\mathcal{O}$,
	
	\item[--] the data instances obtained from $\D_{b}$ by applying to it exactly once each of the rules (a)--(e) defined below in all possible ways.
\end{itemize}
It follows from the assumption that $\q$ is in normal form and the reduced `gaps' between blocks in $\D_{i}^{-}$ that we have $\TO,\D_{i}^{-},0,a\not\models \q$ for all $\D_{i}^{-}$. To obtain a unique charcaterisation, the additional data instances
obtained by applying rules (a)--(e) to $\D_{b}$ are crucial. They `weaken'
$\D_{b}$ by replacing some $\hat{\r}$ by negative examples in $\mathcal{N}_{\r}$
or by introducing big `gaps' between some $\hat{\r}$s. To make our notation more uniform, we think of the pointed data instances in $\N_{\r}$ as having the form $\hat{\r}'$, for a suitable CQ $\r'$ (which is not necessarily in $\Q$). The rules are as follows:

\begin{description}
\item[\rm (a)] replace some $\hat{\r}^i_j$ with
	$\r^i_j\not\equiv_{\mathcal{O}}\top$ by an $\hat{\r} \in \N_{\r^i_j}$, for $i,j$ such that $(i,j)\not=(0,0)$---that is, the rule is not applied to $\r_{0}^{0}$;
	
\item[\rm (b)] replace some pair $\hat{\r}^i_j\hat{\r}^i_{j+1}$
	within block $i$
	by $\hat{\r}^i_j\emptyset^{b}\hat{\r}^i_{j+1}$;
	
\item[\rm (c)] replace some $\hat{\r}_j^i$ such that $\r_{j}^{i} \not\equiv_{\mathcal{O}}\top$ by $\hat{\r}_j^i \emptyset^b \hat{\r}_j^i$,
	where $k_{i}>j>0$---that is, the rule is not applied to $\r_{j}^{i}$ if it is on the border of its block;
	
	
\item[\rm (d)] replace $\hat{\r}^i_{k_i}$ ($k_i > 0$)
	by $\hat{\r}\emptyset^b \hat{\r}^i_{k_i}$, for some $\hat{\r} \in \N_{\r^i_{k_i}}$, or replace $\hat{\r}^i_{0}$ ($k_i > 0$) by $\hat{\r}^i_{0}  \emptyset^b \hat{\r}$, for some $\hat{\r}\in \N_{\r^i_{0}}$;
	
	\item[\rm (e)] replace $\hat{\r}_0^0$ with $\r^0_{0} \not\equiv_{\mathcal{O}}\top$ by $\hat{\r}\emptyset^{b}\hat{\r}^{0}_{0}$, for $\hat{\r}\in \N_{\r^0_0}$, if $k_0 = 0$, and by $\hat{\r}_0^0 \emptyset^b \hat{\r}_0^0$ if $k_0 > 0$.
\end{description}
The proof that $(E^{+},E^{-})$ as defined above uniquely characterises $\q$ \wrt $\TO$ if $\q$ contains no lone conjuncts is non-trivial and extends ideas from the ontology-free case investigated in~\cite{DBLP:conf/kr/FortinKRSWZ22}.
Claim $(ii)$ follows from the observation that the unique characterisation constructed in $(i)$ is polynomial in the size of the characterisations $(\{\hat{\r}\},\mathcal{N}_{\r})$ of the domain queries used in $\q$.
For $(iii)$, assume that 
$\dep(\q) \le n$. Then we add to rules (a)--(e) the following rule: if $\hat{\r}$ is a lone conjunct in $\q$, then replace $\hat{\r}$ by $(\hat{\r}_{1}\emptyset^{b}\cdots \emptyset^{b} \hat{\r}_{k})^{n}$
in $\D_{b}$ for $\N_{\r} = \{\hat\r_{1},\dots,\hat\r_{k}\}$ with $\r_{i}\not\equiv_{\mathcal{O}}\r_{j}$, for $i\ne j$.
As $\r$ is meet-reducible \wrt $\TO$, one can first show that $|\N_{\r}|\geq 2$ and then that we obtain a unique characterisation of $\q$ \wrt $\TO$ within the class of queries in $\Q$ of temporal depth $\leq n$. To show $(iv)$, one can follow the proof of $(i)$ without $\Diamond_{r}$ in $\q$ but possibly with lone conjuncts.
Now, rules (c), (d), and (e) are not needed in the construction of $E^{-}$.

As an immediate consequence of Lemma~\ref{acr} and Theorems~\ref{polyuniqDL-Lite}, \ref{thm:alchi} and~\ref{thm:first} we obtain:
\begin{theorem}
$(i)$ For any $\DL_\mathcal{H}$ or $\DL^-_\mathcal{F}$ ontology $\mathcal{O}$, the following hold\textup{:}
\begin{description}
\item[$(i_1)$] any $\q\in \LTL_p\NDD(\text{ELIQ})$ satisfiable \wrt $\TO$ is uniquely charac\-terisable---in fact, polysize characterisable---\wrt $\mathcal{O}$ within $\LTL_p\NDD(\text{ELIQ})$ iff $\q$ is safe \wrt $\mathcal{O}$\textup{;}

\item[$(i_2)$] $\LTL_p\NDD(\text{ELIQ})$ is polysize characterisable \wrt $\mathcal{O}$ for bounded temporal depth\textup{;}
		
\item[$(i_3)$] $\LTL_p\ND(\text{ELIQ})$ is polysize characterisable \wrt $\mathcal{O}$.
\end{description}
$(ii)$ Let $\sigma$ be a signature. Then claims $(i_1)$--$(i_3)$ also hold for $\mathcal{ALCHI}$ ontologies provided that `polysize' is replaced by `exponen\-tial-size' and ELIQ by ELIQ$^{\sigma}$.
\end{theorem}


\section{Unique Characterisations in $\LTL_{p}\UN(\Q^{\sigma})$}
\label{sec:seven}

We next consider temporalisations by means of the binary operator $\U$ (until), which is more expressive than $\nxt$ and $\Diamond$ as  $\nxt \q \equiv \bot \U \q$ and $\Diamond \q \equiv \top \U \q$ under the strict semantics. Compared to the previous section, we now have to restrict queries to a finite signature because otherwise the implicit universal quantification in $\U$ makes queries such as $\bot \U A$ not uniquely characterisable \wrt the empty ontology~\cite{DBLP:conf/kr/FortinKRSWZ22}. For the same reason, we also have to disallow nesting of $\U$ on the left-hand side of $\U$ in queries. Finally, in the ontology-free case, polysize unique characterisations for propositional \LTL-queries with $\U$ are only available for the so-called peerless queries~\cite{DBLP:conf/kr/FortinKRSWZ22}. These observations lead to the following classes of temporal queries, for which we are going to obtain our transfer results.

Let $\Q$ be a domain query language and $\sigma$ a finite signature of unary and binary predicate symbols. Then $\mathcal{Q}^{\sigma}$ denotes the set of queries in $\Q$ that only use symbols in $\sigma$. The class $\LTL_p\UN(\Q^{\sigma})$ comprises temporal path queries of the form~\eqref{upath}
%
%
where each $\r_{i}\in \mathcal{Q}^{\sigma}$ and each $\el_{i}$ is either in $\mathcal{Q}^{\sigma}$ or $\bot$ (recall that $\q$, $\r_i$, $\el_i$ have a single answer domain variable $x$ and that we evaluate $\q$ at time point $0$).
Given an ontology $\mathcal{O}$, we consider the class $\LTL_{pp}\UN(\Q^{\sigma})$ of \emph{$\mathcal{O}$-peerless queries}
in  $\LTL_{p}\UN(\Q^{\sigma})$ of the form~\eqref{upath}, in which $\r_{i}\not\models_{\mathcal{O}}\el_{i}$ and $\el_{i}\not\models_{\mathcal{O}}\r_{i}$, for all $i \le n$.
In what follows we  write $\TO, \D \models \q$ instead of $\TO, \D, 0, a \models \q$ when $a$ is clear from context. We also write $\D \models \q$ instead of $\emptyset, \D \models \q$ (that is, for the empty ontology).


A fundamental difference to the previous section and Theorem~\ref{thm:first} is that now containment reduction and
unique characterisability of domain queries are not sufficient to guarantee transfer to the temporal case. Recall
that $\DL^-_\mathcal{F}$ admits polytime computable frontiers but no split-partners.
\begin{restatable}{theorem}{frontiersnotworking}\label{frontiersnotworking}
There exist a $\DL^-_\mathcal{F}$ ontology $\TO$, a signature $\sigma$ and a query $\q \in \LTL_{pp}\UN(\text{ELIQ}^{\sigma})$ satisfiable \wrt $\TO$ such that $\q$ is not uniquely characterisable \wrt $\TO$ within $\LTL_{p}\UN(\text{ELIQ}^{\sigma})$.
\end{restatable}

In fact, one can take $\TO$ and $\sigma$ from the proof of Theorem~\ref{counterexample} and set $\q = \bot \U A \equiv \nxt A$. Observe that to separate $\nxt A$ from $\q'\U A$
with a $\sigma$-ELIQ $\q'$ such that $\q'\not\models_{\TO}A$,
one has to add to $E^{-}$ a temporal $\sigma$-data instance $\D=\{\top(a)\}, \A, \{A(a)\}$ such that $\TO,\A \models \q'(a)$ but $\TO,\A\not\models A(a)$.
Such $\A$ could be provided by a finite split-partner for $\{A\}$ \wrt $\TO$ within ELIQ$^{\sigma}$ had it existed, but not from a frontier.

%
%
%

We establish the following general transfer theorem, assuming containment reduction and split-partners:

\begin{restatable}{theorem}{thmmainuntil}\label{thm:mainuntil}
	Suppose $\Q$ is a class of domain queries, $\sigma$ a signature, an ontology language $\mathcal L$ has general split-partners within $\Q^\sigma$, and $\TO$ is a $\sigma$-ontology in $\mathcal L$ admitting  containment reduction. Let $\q \in \LTL_{pp}\UN(\Q^{\sigma})$ be satisfiable \wrt $\TO$. Then the following hold\textup{:}
	
	$(i)$ $\q$ is uniquely characterisable \wrt $\TO$ within $\LTL_{p}\UN(\Q^{\sigma})$.
	
	$(ii)$ If a split-partner for any set $\FQ$, $|\FQ| \le 2$, of $\Q^\sigma$ queries  \wrt $\TO$ within $\Q^\sigma$ is exponential, then there is an exponential-size unique characterisation of $\q$ \wrt $\TO$.
	
	$(iii)$ If a split-partner of any set $\FQ$ as above is polynomial and a split-partner $\mathcal{S}_{\bot}$ of $\bot(x)$ within $\Q^\sigma$ \wrt $\TO$ is a singleton, then there is a polynomial-size unique characterisation of $\q$ \wrt $\TO$.
\end{restatable}

The detailed proof of Theorem~\ref{thm:mainuntil} given in the appendix is by reduction to the ontology-free \LTL{} case, using a characterisation of~\cite{DBLP:conf/kr/FortinKRSWZ22}. Here, we define the example set
that provides the characterisation for $(i)$.
Suppose a signature $\sigma$, a $\sigma$-ontology $\mathcal{O}$, and a query $\q \in \LTL_{pp}\UN(\Q^{\sigma})$ of the form \eqref{upath} are given.
We may assume that $\r_{n}\not\equiv_{\mathcal{O}} \top$.
We obtain the set $E^{+}$ of positive examples by taking
\begin{description}
	\item[$(\mathfrak p_0')$] $\hat{\r}_{0}\dots\hat{\r}_n$;
	
	\item[$(\mathfrak p_1')$] $\hat{\r}_0\dots\hat{\r}_{i-1}\hat{\el}_i\hat{r}_i \dots \hat{\r}_n$;
	
	\item[$(\mathfrak p_2')$] 
	$\hat{\r}_0 \dots \hat{\r}_{i-1} \hat{\el}_{i}^k \hat{\r}_{i} \dots \hat{\r}_{j-1} \hat{\el}_{j}\hat{\r}_j\dots\hat{\r}_n$, \ $i < j$, \ $k = 1,2$.
	
\end{description}
Here, $\hat{\el}_{i}^k$ is a sequence of $k$-many $\hat{\el}_{i}$.
The negative examples $E^{-}$ comprise the following instances $\mathcal{D}$ whenever $\mathcal{D}\not\models \q$:
\begin{description}
	\item[$(\mathfrak n_0')$] 
	$\mathcal{A}_{1},\ldots,\A_{n}$ and $\mathcal{A}_{1},\ldots,\A_{n-i},\A,\A_{n-i+1},\ldots,\A_{n}$, for $(\mathcal{A}, a) \in \mathcal{S}(\{\r_i\})$ and $(\A_{1}, a),\ldots,(\A_{n}, a) \in \mathcal{S}_\bot$;
	
	\item[$(\mathfrak n_1')$] $\D=\hat{\r}_0\dots\hat{\r}_{i-1}\mathcal{A}\hat{\r}_i\dots\hat{\r}_n$, where $(\mathcal{A}, a)$ is an element of $\mathcal{S}(\{\el_{i},\r_{i}\})\cup \mathcal{S}(\{\el_{i}\}) \cup \mathcal{S}_\bot$;
	
	\item[$(\mathfrak n_2')$]
	for all $i$ and $(\mathcal{A}, a) \in \mathcal{S}(\{\el_i,\r_{i}\})\cup\mathcal{S}(\{\el_i\})\cup\mathcal{S}_\bot$, \emph{some} data instance
	$$
	\D^i_{\mathcal{A}} = \hat{\r}_{0}\dots \hat{\r}_{i-1} \mathcal{A} \hat{\r}_{i}\hat{\el}_{i+1}^{k_{i+1}}\hat{\r}_{i+1} \dots \hat{\el}_{n}^{k_n}\r_n,
	$$
	if any, such that $\max(\D^i_{\mathcal{A}}) \le (n+1)^2$ and $\D^i_{\mathcal{A}}\not\models \q^\dag$ for $\q^\dag$ obtained from $\q$ by replacing all $\el_j$, for  $j \le i$, with $\bot$.
\end{description}
%
%
We have $(ii)$ since $(E^{+},E^{-})$ is at most exponential in the size of split-partners of sets with at most two queries. For $(iii)$, observe that $(\mathfrak n_1')$ is exponential in $|\mathcal{S}_{\bot}|$ iff $|\mathcal{S}_{\bot}|\geq 2$.

As a consequence of Lemma~\ref{acr}, Theorem~\ref{thm:mainuntil} $(ii)$ and $(iii)$, and Theorems~\ref{alchi-split} and~\ref{thm:ELsplit} we obtain the following (note that, for every RDFS ontology, the split-partner $\mathcal{S}_{\bot}$ of $\bot$ is a singleton by Example~\ref{ex:22} $(i)$):
\begin{theorem}\label{thmforU}
$(i)$ Each $\q \in \LTL_{pp}\UN(\text{ELIQ}^{\sigma})$  satisfiable \wrt an $\mathcal{ALCHI}$ ontology $\TO$ in a signature $\sigma$ is exponential-size uniquely characterisable \wrt $\TO$ within $\LTL_{p}\UN(\text{ELIQ}^{\sigma})$.

	
$(ii)$ Each $\q \in \LTL_{pp}\UN(\text{ELQ}^{\sigma})$ is polysize uniquely characterisable \wrt any RDFS ontology in $\sigma$   within $\LTL_{p}\UN(\text{ELQ}^{\sigma})$.
\end{theorem}

\newcommand{\ELHIF}{\ensuremath{\mathcal{ELHIF}}\xspace}

\section{Exact Learnability}
\label{sec:learning}

We apply the results on unique characterisability obtained in
Section~\ref{sec:six} to exact learnability of queries
\wrt ontologies.
Given a query class $\Cmc$ and
an ontology \Omc, the \emph{learner} aims to identify a \textit{target
query} $\q_T\in \Cmc$ by means of membership queries of the form `does
$\Omc,\Dmc,0,a\models \q_T$ hold?' to the \emph{teacher}.  We call \Cmc
\textit{polynomial time learnable \wrt
\Lmc ontologies using membership queries} if there is a learning
algorithm that given \Omc constructs (up to equivalence \wrt
\Omc) any $\q_T$
satisfiable \wrt \Omc in time polynomial in the sizes of $\q_T$ and $\Omc$. For the weaker requirement of
\textit{polynomial query learnability}, it suffices that
the total size of the examples given to the oracle be bounded by a
polynomial. 
We start with making the following observation, where exponential
query learnability is defined in the expected way.
\begin{theorem}\label{thm:explearning}
  Let \Lmc be an ontology language and $\Cmc$ be a class of queries
  such that $(1)$ \Cmc has an effective syntax, $(2)$ for every \Lmc-ontology \Omc, every query
  in $\Cmc$ satisfiable \wrt \Omc admits effective
  exponential size unique characterisations
  \wrt \Omc, and $(3)$ satisfiability of queries in \Cmc \wrt
\Lmc-ontologies is decidable. 
  Then, \Cmc is exponential query learnable \wrt \Lmc ontologies.
\end{theorem}
\begin{proof}
  Let $\q_T\in\Cmc$ be the target query and \Omc an \Lmc ontology. 
  Due to~(1), we can enumerate all queries from \Cmc in increasing size. For every
  enumerated $\q$, test whether $\q$ is satisfiable \wrt \Omc,
  using~(3). If so, then using~(2), we compute its unique
  characterisation $(E^+,E^-)$ \wrt \Omc and use membership
  queries to check whether all examples in $E^+$ are positive examples
and all examples in $E^-$ are negative examples. If so, output
  $\q$. 
\end{proof}


Since our main focus in this section is, however, polynomial time and query
learnability, we consider below cases which allow for polynomial
size unique characterisations. As the presence of $\sqcap$ and 
$\bot$ in the ontology language precludes polynomial query
learnability already in the atemporal case, c.f.\ Theorem~6 in~\cite{DBLP:conf/ijcai/FunkJL22}, we follow
their approach and assume 
that the learner also receives an initial positive example
$\Dmc,a$ with $\Dmc$ and $\Omc$
satisfiable. Note that the existence of such an example
implies that the target query is satisfiable \wrt the ontology. In order to state
our main result, we introduce one further natural condition. An ontology language \Lmc \emph{admits polynomial time
instance checking} if given an \Lmc ontology \Omc, a pointed instance
$(\Amc,a)$, and a concept name $A$, it is decidable in polynomial time
whether $\Omc,\Amc\models A(a)$. 
%
\begin{restatable}{theorem}{thmlearning} \label{thm:learning}
  Let \Lmc be an ontology language that contains only
  \ELHI or only $\ELIF$ ontologies and that admits polysize frontiers
  within ELIQ that can be computed. Then:
  \begin{description}

    \item[\rm$(i)$] The class of safe 
      $\LTL_p\NDD(\text{ELIQ})$ queries are polynomial query
      learnable \wrt \Lmc ontologies using membership queries.

    \item[\rm$(ii)$] The class $\LTL_p\NDD(\text{ELIQ})$ is
      polynomial query learnable \wrt \Lmc ontologies using membership
      queries if the learner knows the temporal depth of the target query.

    \item[\rm$(iii)$] The class
      $\LTL_p\ND(\text{ELIQ})$ is polynomial query
      learnable \wrt \Lmc ontologies using membership queries.

  \end{description}
  If $\Lmc$ further admits polynomial time instance checking and polynomial time computable frontiers within ELIQ, then in $(ii)$ and $(iii)$,
  polynomial query learnability can be replaced by polynomial time
  learnability. If, in addition, meet-reducibility \wrt \Lmc ontologies 
  can be decided in polynomial time, then also in $(i)$ polynomial
  query learnability can be replaced by polynomial time learnability.
  %
\end{restatable}

To achieve the generality of the results independently of the exact
languages, in the proof of Theorem~\ref{thm:learning} we rely on the results and techniques from
Section~\ref{sec:six} and general results proved in the context of
exact learning of (atemporal) ELIQs \wrt 
ontologies~\cite{DBLP:conf/dlog/FunkJL22}. 
%
%

Let $\q_T$ be a target query, $\Omc$ be an ontology, and $\Dmc,a$ be a positive
example with $\Dmc=\Amc_0\ldots\Amc_n$ and \Dmc and \Omc
satisfiable. The idea is to modify \Dmc in a number of steps such
that, in the end, \Dmc viewed as temporal query is equivalent to~$\q_T$.

We describe how to show~$(i)$; $(ii)$
and~$(iii)$ are slight modifications thereof. In \textbf{Step 1}, the goal is
to find a temporal data instance $\Dmc$ where each $\Amc_i$ is
\emph{tree-shaped} and hence can be viewed as an ELIQ. This can be done separately for each time point
using membership queries and standard unraveling techniques from the
atemporal setting~\cite{DBLP:conf/dlog/FunkJL22}. 
In \textbf{Step 2}, we exhaustively apply
Rules~(a)-(e) from the proof of Theorem~\ref{thm:first} to \Dmc, as
long as $\Dmc,a$ remains a positive example. In \textbf{Step 3}, we
take care of lone conjuncts in \Dmc (when viewed as a temporal query)
-- recall that $\q_T$ is safe and thus does not have any. For this
step, we rely on a characterisation of meet-reducibility in terms of
\emph{minimal} frontiers. For computing those, we exploit the fact
that containment of ELIQs \wrt \ELHI and \ELIF ontologies is
decidable~\cite{DBLP:conf/ijcai/Bienvenu0LW16}. After Step~3, \Dmc
(viewed as query) is already very similar to $\q_T$. More precisely,
when representing $\q_T$ in shape~\eqref{fullq2appendix} as a sequence of blocks
$\q_0\Rmc_1\q_1\ldots\Rmc_m \q_m$, then \Dmc has the shape
$\Dmc_0\emptyset^b\ldots\emptyset^b \Dmc_m$, for sufficiently large
$b$, and each $\q_i$ is isomorphic to $\Dmc_i$. So
in \textbf{Step~4}, it remains to identify the precise separators
$\Rmc_i$. They can be a single $\leq$ or a sequence of $<$, and
the two
cases can be distinguished using suitable membership queries. 

In order to show that this entire process terminates after asking
polynomially many membership queries, we lift the notion of
\emph{generalisation sequences} from~\cite{DBLP:conf/dlog/FunkJL22} to
the temporal setting. For the sake of convenience, we treat the data
instances in the time points as CQs. A sequence $\D_{1},\ldots$ of
temporal data instances is a \emph{generalisation sequence towards
$\q_T$ \wrt $\Omc$} if for all $i\geq 1$:
  \begin{itemize}

  \item $\D_{i+1}$ is obtained from $\D_{i}$ by modifying one
    non-temporal CQ $\r_{j}$ in $\D_{i}$ to $\r_{j}'$ such that
    $\r_{j}\models_\Omc \r_{j}'$ and $\r_{j}'\not\models_\Omc \r_{j}$;

  \item $\Omc,\D_{i},0,a\models \q_T$ for all $i\geq 1$.

\end{itemize}
Intuitively, data instances in generalisation sequences become weaker
and weaker, and based on this, we show that the length of
generalisation sequences towards $\q_T$ \wrt \Omc is bounded by a
polynomial in $\max(\Dmc_1)$ and the sizes of $\q_T,\Omc$.
The crucial observation is that the sequences
of temporal data instances obtained by rule application are mostly generalisation sequences towards $\q_T$ \wrt \Omc;
thus the steps terminate in polynomial time. If they are not,
we use a different (but usually easier) termination argument.

It remains to note that the sketched algorithm runs in polynomial time
when \Lmc satisfies all the required criteria. 
\qed

We finally apply Theorem~\ref{thm:learning} to concrete ontology
languages, namely $\DL_\Fmc^-$ and $\DL_{\mathcal{H}}$. 
\begin{restatable}{theorem}{thmlearningdllite}\label{thm:learning-dllite}
  %
  The following learnability results hold:
  \begin{description}

    \item[\rm$(i)$] The class of safe queries in
      $\LTL_p\NDD(\text{ELIQ})$ is polynomial query
      learnable \wrt $\DL_{\mathcal{H}}$ ontologies using
      membership queries
      and polynomial time
      learnable \wrt $\DL_{\mathcal{F}}^{-}$ ontologies using
      membership queries.

    \item[\rm$(ii)$] The class $\LTL_p\NDD(\text{ELIQ})$ is polynomial time
      learnable \wrt both $\DL_{\cal{F}}^{-}$ and
      $\DL_{\mathcal{H}}$ ontologies using membership
      queries if the learner knows the temporal depth of the target
      query in advance.

    \item[\rm$(iii)$] The class
      $\LTL_p\ND(\text{ELIQ})$ is polynomial time
      learnable \wrt  both $\DL_{\cal{F}}^{-}$ and
      $\DL_{\mathcal{H}}$ ontologies using membership
      queries.

  \end{description}
\end{restatable}
Theorem~\ref{thm:learning-dllite} is a direct consequence of Theorem~\ref{thm:learning} and the
fact that the considered ontology languages satisfy all conditions
mentioned there. In particular, we show in the appendix that
meet-reducibility of ELIQs \wrt $\DL_{\Fmc}^-$ ontologies Turing
reduces to ELIQ containment \wrt $\DL_\Fmc^-$ ontologies, which is
tractable~\cite{DBLP:conf/ijcai/BienvenuOSX13}. The latter is not true
for $\DL_\Hmc$ which explains the difference in~$(i)$. We leave
it for future work whether $\LTL_p\NDD(\text{ELIQ})$ is
polynomial \emph{time} learnable \wrt $\DL_{\mathcal{H}}$
ontologies.
\section{Outlook}

Many interesting and challenging problems remain to be addressed. We discuss a few of them below.
\begin{description}\itemsep=0pt
\item [\rm(1)] Is it possible to overcome our `negative' unique characterisability results by admitting some form of infinite (but finitely presentable) examples? Some results in this direction without ontologies are obtained in~\cite{Sestic}.

\item [\rm(2)] We have not considered learnability using membership queries of temporal queries with $\U$. In fact, it remains completely open how far our characterisability results for these queries can be exploited to obtain polynomial query (or time) learnability.

\item [\rm(3)] We only considered path queries with no temporal operator occurring in the scope of a DL operator. This is motivated by the negative results of~\cite{DBLP:conf/kr/FortinKRSWZ22}, which showed that $(i)$ applying $\exists P$ to $\nxt\Diamond$-queries quickly leads to non-characterisability and that $(ii)$ even without DL-operators and without ontology, branching $\Diamond$-queries are often not uniquely characterisable. We still  believe there is some scope for useful positive characterisability results.
\end{description}

\bibliographystyle{kr}
\bibliography{local}

\newpage

\newpage
\begin{appendix}
%
%

\section{Proofs and Detailed Definitions for Section~\ref{Sec:uniqueChar}}
	\acr*
	
	\begin{proof}
		(1) By~\cite[Proposition 5.9]{DBLP:journals/tods/BienvenuCLW14}, for any FO-ontology $\mathcal{O}$ without =, any CQ $\q$, and any pointed instances $\A_{1},a_{1}$ and $\A_{2},a_{2}$, if there is $h \colon \A_{1} \to \A_{2}$ with $h(a_{1}) = a_{2}$, then $\mathcal{O},\A_{1} \models \q(a_{1})$ implies $\mathcal{O},\A_{2} \models \q(a_{2})$.
		Let $(\hat{\q}(x),a)$ be \emph{induced} by $\q(x)$, i.e., obtained by replacing the variables in $\q$ by distinct constants, with $x$ replaced by $a$.
		Suppose $\mathcal{O},\hat{\q}\models \q'$ and $\mathcal{O},\A\models \q(a)$ but $\mathcal{O},\A\not\models \q'(a)$. Take a model $\I$ witnessing $\mathcal{O},\A\not\models \q_{2}(a)$. Then
		$\I \models \q_{1}(a)$ and this is witnessed by a homomorphism $h:\q_{1}\rightarrow \I$. Take the image $h(\q_{1})$. Then $\mathcal{O},\widehat{h(\q_{1})}\not\models \q_{2}(x)$ is witnessed by $\I$, and so $\mathcal{O},\hat{\q}_{1}\not\models \q_{2}(x)$, which is a contradiction.
		
		(2) Let $\mathcal{O}$ be a $\mathcal{ELIF}$ ontology. Given a CQ $\q(x)$, define an equivalence relation $\sim$ on $\var(\q)$ as the transitive closure of the following relation:  $y\sim' z$ iff there is $u \in \var(\q)$ such that $S(u,y),S(u,z) \in \q$, for a functional $S$ in $\mathcal{O}$. Let $\q/_\sim$ be obtained by identifying (glueing together) all of the variables in each equivalence class $y/_\sim$. Clearly, $\q/_\sim(x/_\sim)$ is a homomorphic image of $\q(x)$ and $\q(x) \equiv_\mathcal{O} \q/_\sim(x/_\sim)$. We define $(\hat{\q},a)$ as the pointed data instance induced by $\q/_\sim(x/_\sim)$. Conditions (cr$_1$) and (cr$_2$) are obvious, and (cr$_3$) follows from the fact that $\q/_\sim \models_{\mathcal{O}} \q'/_\sim$ iff $\q/_\sim \models_{\mathcal{O}'} \q'/_\sim$, where $\mathcal{O}'$ is obtained from $\mathcal{O}$ by omitting all of its functionality constraints, which is in the scope of part (1).
		
		(3) Suppose otherwise. Let $\mathcal{O} = \{ \ge 3 \,P \sqsubseteq \bot\}$ and let $\q(x) = \{P(x,y_i), A_i(y_i) \mid i =1,2,3\}$ with a suitable $(\hat{\q},a)$. As $\mathcal{O}$ and the instance induced by $\q$ are not satisfiable and in view of (cr$_2$), $\hat{\q}$ contains at most three individuals, say, $\hat{\q} = \{P(a,b), A_1(b), A_2(b), P(a,c), A_3(c)\}$. But then, by (cr$_3$), $\q'(x) = \{P(x,y), A_1(y), A_2(y), P(x,z), A_3(z)\}$ should satisfy $\q\models_{\mathcal{O}} \q'$, which is not the case as witnessed by $\A = \{P(a,b), A_1(b), P(a,c), A_2(c), A_3(c)\}$ because $\mathcal{O}, \A \models \q(a)$ but $\mathcal{O}, \A \not\models \q'(a)$.
	\end{proof}
	
	\norm*
	
	\begin{proof}
		To show that $E'$ is as required, we first observe that $\q$ fits $E'$ by (\textbf{cr}$_1$) and (\textbf{cr}$_3$). Suppose $\q' \not \equiv_{\TO} \q$ for some $\q' \in \Q$. We show that then either $\TO, \hat{\q} \not \models \q'$ or $\TO, \D \models \q'$ for some $\D \in E^-$.
		Let $\q \not \models_\TO \q'$. Then $\TO, \hat{\q} \not \models \q'$ by (\textbf{cr}$_3$). Let $\q \models_\TO \q'$ and $\q' \not \models_\TO \q$. Then $\TO, \D \models \q'$ for all $\D \in E^+$, and so $\TO, \D \models \q'$, for some $\D \in E^-$, because $E$ is a unique characterisation of $\q$ \wrt $\TO$.
	\end{proof}

\critO*
	\begin{proof}
		By \textbf{(cr$_2$)}, $\TO, \hat{\q} \models \q(a)$.
		To show $\TO, \hat{\r} \not \models \q(a)$ for all $\r \in \mathcal{F}_{\q}$, we observe that $\r \not \models_{\TO} \q$ by the definition of $\mathcal{F}_{\q}$, so $\r(x)$ is consistent with $\TO$ and by \textbf{(cr$_3$)} for $\r$, from which $\TO, \hat{\r} \not \models \q(a)$. Thus, $\q$ fits $E$.
		
		Let $\q' \in \Q$ and $\q \not\equiv_\TO \q'$. We  show that either $\TO, \hat{\q} \not \models \q'$ or $\TO, \hat{\r} \models \q'$ for some $\r \in \mathcal{F}_{\q}$. If $\q \not \models_\TO \q'$, then, since $\TO$ admits containment reduction and $\q(x)$ is satisfiable \wrt $\TO$, we obtain $\TO, \hat{\q} \not \models \q'$ by \textbf{(cr$_3$)}.
		So suppose $\q \models_\TO \q'$ and $\q' \not \models_\TO \q$. As $\mathcal{F}_{\q}$ is a frontier of $\q$ \wrt $\mathcal{O}$,   there is $\r \in \mathcal{F}_{\q}$ with $\r \models_{\TO} \q'$. If $\r(x)$ is unsatisfiable \wrt $\TO$, then $\TO$ and $\hat{\r}$ are unsatisfiable by \textbf{(cr$_1$)}, and so $\TO, \hat{\r} \models \q'$. And if $\r(x)$ is satisfiable \wrt $\TO$, we  obtain $\TO, \hat{\r} \models \q'(a)$ by \textbf{(cr$_3$)}.
	\end{proof}

\alchi*
	
\begin{proof}
	Suppose a finite set $Q \subseteq \text{ELIQ}^\sigma$ and an $\mathcal{ALCHI}$-ontology $\mathcal{O}$ in the signature $\sigma$ are given. Let $\sub_{\TO, Q}$ be the closure under single negation of the set of subconcepts of concepts in $Q$ and $\mathcal{O}$. A \emph{type for} $\TO$ is any maximal subset $\tp \subseteq \sub_{\TO, Q}$ consistent with $\TO$. Let $\Type$ be the set  of all types for $\TO$.
	%
	Define a $\sigma$-data instance $\A$ with $\ind (\A) = \Type$, $A(\tp) \in \A$ for all concept names $A \in \sigma$ and $\tp$ such that $A \in \tp$, and $P(\tp, \tp') \in \A$ for all role names $P \in \sigma$, $\tp$ and $\tp'$ such that $\tp$ and $\tp'$ can be satisfied by the domain elements of a model of $\TO$ that are related via $P$. We consider an interpretation $\I_\A$ with $\Delta^{\I_\A} = \{ \ind( \A ) \}$, $A^{\I_\A} = \{ \tp \mid A(\tp) \in \A\}$ for concept names $A \in \sigma$, $A^{\I_\A} = \emptyset$ for $A \not \in \sigma$, $P^{\I_\A} = \{ (\tp, \tp') \mid P(\tp, \tp') \in \A\}$ for role names $P \in \sigma$, $P^{\I_\A} = \emptyset$ for $P \not \in \sigma$. It can be readily checked that, for any $\q \in \Q^\sigma$,
	\begin{align}\label{eq:alc-can-mod1}
		&\I_\A \models \TO, \A,\\
		&\label{eq:alc-can-mod2}   \q \in \tp \quad \text{iff} \quad \I_\A, \tp \models \q \quad \text{iff} \quad \TO, \A \models \q(\tp).
	\end{align}
	%
	%
	Let $Q = \{\q_1, \dots, \q_n\}$ and let $\A^{n}$ be the $n$-times direct product of $\A$. Set
	$$
	\mathcal S(Q) = \{(\A^n, (\tp_1, \dots, \tp_n)) \mid \neg \q_i \in \tp_i, \ i=1,\dots n \}.
	$$
	We prove~\eqref{eq:split-def} for an arbitrary $\q' \in \Q^\sigma$. For the $(\Rightarrow)$ direction, suppose $\TO, \A^n \models \q'(\vec{\tp})$ for some $(\A^n, \vec{\tp}) \in \mathcal S(Q)$ and $\vec{\tp} = (\tp_1, \dots, \tp_n)$. Fix $i \in \{1, \dots, n\}$. Observe that the \emph{projection} map $h((\tp_1', \dots, \tp_n')) = \tp_i'$ for $(\tp_1', \dots, \tp_n') \in \Type^n$ is a homomorphism from $\A^n$ to $\A$ such that $h(\vec{tp}) = \tp_i$. As in the proof of Lemma~\ref{acr} $(1)$, we obtain $\TO, \A \models \q'(\tp_i)$. Recall that $\neg \q_i \in \tp_i$. Then, by~\eqref{eq:alc-can-mod2}, we have $\I_\A, \tp_i \models \q'$, $\I_\A, \tp_i \not \models \q_i$, and so using~\eqref{eq:alc-can-mod1} we obtain $\q' \not \models_\TO \q_i$. For the opposite direction, suppose $\q' \not \models_\TO \q_i$ for all $1 \leq i \leq n$. It follows that, for each $i$, there exists $\tp_i \in \Type$ such that $\q', \neg \q_i \in \tp_i$. Let $\vec{\tp} = (\tp_1, \dots, \tp_n)$. Clearly, $(\A^n, \vec{tp}) \in \mathcal S(Q)$ and it remains to show that $\TO, \A^n \models \q'(\vec{\tp})$. Observe that, for each $\tp_i$, by~\eqref{eq:alc-can-mod2},  there exists a homomorphism $h_i$ that maps $\q'$ into $\I_\A$ with its root mapped to $\tp_i$. By the construction of $\I_\A$, the same holds for $\A$ in place of $\I_A$. Because $\A^n$ is a direct product, there exists a homomorphism that maps $\q'$ into $\A^n$ with its root mapped to $\vec{tp}$. Thus, $\A^n \models \q'(\avec{\vec{\tp}})$.
\end{proof}	

We observe that split partners for conjunctions can be obtained
from split partners for the conjuncts.
\begin{lemma} \label{lem:split-conjuncts}
	Let $\sigma$ be a signature, $\Qmc^\sigma$ be a subset of CQ and
	\Lmc an arbitrary logic.  Let $\q=\q_1\wedge\q_2$ be any CQ and \Omc
	be an \Lmc-ontology.  Then, if $\Smc_1,\Smc_2$ are split partners
	for $\q_1,\q_2$ \wrt \Omc within $\Qmc^\sigma$, then
	$\Smc_1\cup\Smc_2$ is a split partner for $\q$ \wrt \Omc within
	$\Qmc^\sigma$.
\end{lemma}
\begin{proof}
	Let $\q'\in \Qmc^\sigma$ arbitrary.
	
	Suppose first $\Omc,\Amc\models \q'(a)$ for some $(\Amc,a)\in \Smc$.
	Then $\Omc,\Amc\models \q'(a)$ for some $(\Amc,a)\in \Smc_i$, for
	some $i\in \{1,2\}$. Since $\Smc_i$ is a split-partner for $\q_i$
	\wrt \Omc within $\Qmc^\sigma$, we have $\q'\not\models_\Omc \q_i$,
	and thus $\q'\not\models_\Omc \q$.
	
	Suppose now that $\q'\not\models_\Omc \q$. Thus $\q'\not\models_\Omc
	\q_i$, for some $i\in\{1,2\}$.  Since $\Smc_i$ is a split-partner
	for $\q_i$ \wrt \Omc within $\Qmc^\sigma$, we have $\Omc,\Amc\models
	\q'(a)$ for some $(\Amc,a)\in \Smc_i$. Hence, $\Omc,\Amc\models
	\q'(a)$ for some $(\Amc,a)\in \Smc$.
\end{proof}

\elsplit*
\begin{proof}
	We prove the statement for $n=1$, the generalisation is
	straightforward. Let $Q=\{\q\}$. The construction is by induction
	over the depth of $\q$.  Assume $\text{depth}(\q)=0$. Due to
	Lemma~\ref{lem:split-conjuncts}, it suffices to consider
	$\q=A$ with $A$ a concept name. Define a data instance $\Amc$ by
	taking
	\begin{eqnarray*}
		\mathcal{A} & = & \{B(a) \mid \TO \not\models B \sqsubseteq A, B \in \sigma\}\cup\\
		& &  \{ R(a,b) \mid \TO\not\models \exists R\sqsubseteq A,R\in \sigma\} \cup \\
		&  & \{ B(b), R(b,b) \mid B,R\in \sigma\}
	\end{eqnarray*}
	and set $\mathcal{S}(\q) = \{ (\mathcal{A},a)\}$. We show that $\mathcal{S}(\q)$ is as required. Assume
	\[ \q'=\bigwedge_{i=1}^{m_{1}}B_{i} \wedge \bigwedge_{i=1}^{m_{2}}
	\exists R_{i}.\q_{i}.
	\]
	If $\q'\not\models_{\TO}\q$, then
	\begin{itemize}
		
		\item $\TO\not\models B_{j}\sqsubseteq A$ for all $B_{j}$;
		
		\item $\TO\not\models \exists R_{j} \sqsubseteq A$ for all $R_{j}$.
		
	\end{itemize}
	Then $\TO, \mathcal{A}\models \q'(a)$, as required.
	
	Conversely, if $\TO, \mathcal{A} \models \q'(a)$, then
	\begin{itemize}
		
		\item for all $B_{j}$ there exists $B(a)\in \A$ with $\TO
		\models B \sqsubseteq B_{j}$. Then $\TO\not\models B \sqsubseteq
		A$, and so $\TO\not\models B_{j}\sqsubseteq A$;
		
		\item for all $\exists R_{j}.\q_{j}$ there exists $R(a,b)\in
		\A$ with $\TO\models R\sqsubseteq R_{j}$. Then
		$\TO\not\models \exists R\sqsubseteq A$, and so
		$\TO\not\models \exists R_{j}\sqsubseteq A$.
		
	\end{itemize}
	Hence, $\q'\not\models_{\TO}\q$, as required.
	
	\medskip
	
	Assume now that $\text{depth}(\q)=n+1$ and that split partners
	$\mathcal{S}(\q')$ have been defined for queries of depth
	$\leq n$. In view of Lemma~\ref{lem:split-conjuncts} it suffices to
	consider the case
	\[
	\q=\exists S_1.\q_1,\]
	for an ELIQ $\q_1$ of depth $\leq n$ with split partner
	$\mathcal{S}(\q_{1})=
	\{(\mathcal{A}_{1},a_{1}),\ldots,(\mathcal{A}_{k},a_{k})\}$.
	
	Let for all $R$ with
	$\TO\models R \sqsubseteq S_1$, $I_{R}$ be the set of $j\leq k$ with
	$\TO,\A_{j}\models A(a_{j})$ whenever $\TO \models \exists R^-
	\sqsubseteq A$. Define a data instance \Amc by taking
	\begin{eqnarray*}
		\mathcal{A} & = & \{B(a) \mid B \in \sigma\}\cup\\
		&   &  \{ R(a,b), S(b,b), B(b) \mid \TO\not\models R\sqsubseteq
		S_{1},R,B,S\in \sigma\} \cup \\
		& &  \{ R(a,a_{j}) \mid j\in I_{R}, R\in \sigma\} \cup\\
		& &  \mathcal{A}_{1}(a_{1}) \cup \cdots \cup
		\mathcal{A}_{k}(a_{k}).
	\end{eqnarray*}
	and set $\Smc(\q)=\{(\Amc,a)\}$.
	%
	We show that $\mathcal{S}(\q)$ is as required.
	We show the following claim for arbitrary $\q'$ of the form
	\[
	\q'=\bigwedge_{i=1}^{m_{1}}B_{i} \wedge \bigwedge_{i=1}^{m_{2}}
	\exists R_{i}.\q_{i}'.
	\]

	\medskip \noindent \emph{Claim 1.} $\q'\not\models_{\TO} \exists S_{1}.\q_{1}$ iff $\TO,\A_{i}\models \q'(a)$.
	
	\medskip \noindent \emph{Proof of Caim 1.} If
	$\q'\not\models_{\TO}\exists S_{1}.\q_{1}$, then for all $\exists
	R_{j}.\q_{j}'$ with $\TO\models R_{j}\sqsubseteq S_{1}$ we have
	\[ \exists R_{j}^-\sqcap \q_{j}' \not\models_{\TO}\q_{1}.  \]
	Take any $j$. Let $C_{R_{j}}$ be the conjunction of all $A$ with
	$\TO\models\exists R_{j}^-\sqsubseteq A$. Then
	\[ C_{R_{j}}\sqcap \q_{j}' \not\models_{\TO}\q_{1}.  \]
	By the definition of split-partners, there is $\ell$ with
	\[ \TO,\A_{\ell}\models C_{R_{j}}\sqcap \q_{j}'(a_{\ell}).  \]
	It now follows immediately that $\TO,\A\models \exists
	R_{j}.\q_{j}'(a)$. Hence $\TO,\A\models \q'(a)$ follows, as
	required.
	
	Conversely, assume $\q'\models_{\TO}\exists S_{1}.\q_{1}$. Then there exists
	$\exists R_{j}.\q_{j}'$ with $\TO\models R_{j}\sqsubseteq S_{1}$ and
	\[
	\exists R_{j}^-\sqcap \q_{j}' \models_{\TO}\q_{1}.
	\]
	Let again $C_{R_{j}}$ be the conjunction of all $A$ with
	$\TO\models\exists R_{j}^-\sqsubseteq A$. Then
	\[
	C_{R_{j}}\sqcap \q_{j}' \models_{\TO}\q_{i}.
	\]
	By the definition of split-partners,
	$$
	\TO,\A_{\ell}\not\models C_{R_{j}}\sqcap \q_{j}'(a_{\ell}).
	$$
	for all $\ell\leq k$. But then $\TO,\A\not\models \exists R_{j}.\q_{j}'(a)$ and so $\TO,\A\not\models \q'(a)$, as required.
	
	\smallskip
	This finishes the proof of Claim~1 and, in fact, of the
	Theorem.
	%
\end{proof}

\thmcritone*

%
\begin{proof}
	Clearly, $\q$ fits $E$ as $\mathcal{O},\hat{\q}\models \q(a)$ and $\mathcal{O},\A\not\models \q(a)$ for any $(\A,a)\in \mathcal{S}_{\q}$ as otherwise $\q\not\models_{\mathcal{O}} \q$. Let $\q'\not\equiv_{\mathcal{O}}\q$. If $\q'\models_{\mathcal{O}}\q$, then $\q\not\models_{\mathcal{O}}\q'$, and so $\mathcal{O},\hat{\q}\not\models \q'(a)$. Hence $\q'$ does not fit $E$. If $\q'\not\models_{\mathcal{O}}\q$, then there exists $(\A,a) \in \mathcal{S}_{\q}$ with $\mathcal{O},\A\models \q'(a)$, and so again $\q'$ does not fit $E$.
\end{proof}
A \emph{CQ-frontier} for an ELIQ $\q$ \wrt to an ontology $\mathcal{O}$ is a set $\mathcal{F}_{\q}$ of CQs such that
\begin{itemize}
	\item if $\q'\models_\mathcal{O} \q''$, for a CQ $\q'\in \mathcal{F}_{\q}$ and an ELIQ $\q''$, then $\q\models_{\mathcal{O}} \q''$ and $\q''\not\models_{\mathcal{O}} \q$;
	
	\item if $\q\models_{\mathcal{O}}\q''$ and $\q''\not\models_{\mathcal{O}} \q$, for an ELIQ $\q''$,
	then there exists $\q'\in\mathcal{F}_{\q}$ such that $\q'\models_{\mathcal{O}}\q''$.
\end{itemize}
Clearly standard ELIQ frontiers defined above are also CQ-frontiers.

\counterexone*
\begin{proof}
	We show that the query $\q=A\wedge B$ does not have a finite CQ-frontier \wrt the ontology
	$$
	\mathcal{O} = \{ A \sqsubseteq \exists R.A, \ \exists R.A \sqsubseteq A\}
	$$
	%
	within ELIQs. Suppose otherwise. Let $\mathcal{F}_{\q}$ be such a CQ-frontier.
	Consider the ELIQs $\r_{n,m} = \exists R^{n} \exists {R^{-}}^{m}.B$ with $n>m>0$. Clearly, $\q \not \models_\mathcal{O} \r_{n,m}$, and so $\r_{n,m}$ cannot be entailed \wrt $\mathcal{O}$ by any CQ in $\mathcal{F}_{\q}$. Thus, if $\q'\in \mathcal{F}_{\q}$ and $B(x)$ is in $\q'(x)$, then we cannot have an $R$-cycle in $\q'$ reachable from $x$ along an $R$-path as otherwise we would have $\q' \models_\mathcal{O} \r_{n,m}$ for suitable $n,m$.
	
	Consider now the ELIQs $\q_{n}= B \wedge \exists R^{n}.\top$,
	for all $n\geq 1$. Clearly, $\q\models_{\mathcal{O}}\q_{n}$. As $\q_{n}\not\models_{\mathcal{O}}\q$, infinitely many $\q_{n}$ are entailed by some $\q'\in \mathcal{F}_{\q}$ wrt $\mathcal{O}$. Take such a $\q'$. Since $B(x)\in \q'$ because of $\q' \models_{\mathcal{O}}\q_{n}$, no $R$-cycle is reachable from $x$ via an $R$-path in $\q'$.
	Note also that no $y$ with $A(y)\in \q'$ can be reached from $x$ along an $R$-path as otherwise $\q'\models_{\mathcal{O}}B \wedge \exists R^{k}.A$ for some $k\geq 0$ and, since $B \wedge \exists R^{k}.A\models_{\mathcal{O}}\q$ by the second axiom in $\mathcal{O}$, we would have $\q'\models_{\mathcal{O}}\q$.
	
	To derive a contradiction, we show now that there is an $R$-path of any length $n$ starting at $x$ in $\q'$. Suppose this is not the case. Let $n$ be the length of a longest $R$-path starting at $x$ in $\q'$. We construct a model $\mathcal{I}$ of $\mathcal{O}$ and $\q'$ refuting $\q_{n+l}$, for any $l \ge 1$. Define $\mathcal{I}$ by taking
	\begin{itemize}
		\item $\Delta^{\mathcal{I}} = \var(\q') \cup \{d_{1},d_2, \dots\}$,  for fresh $d_{i}$;
		
		\item $a\in B^{\mathcal{I}}$ if $B(a)\in \q'$;
		
		\item $a\in A^{\mathcal{I}}$ if there is an $R$-path in $\q'$ from $a$ to some $y$ with $A(y)\in \q'$ or $a=d_{i}$ for some $i$;
		
		\item $(a,b)\in R^{\mathcal{I}}$ if $R(a,b)\in \q'$ or there is an $R$-path in $\q'$ from $a$ to some $y$ with $A(y)\in \q'$ and $b=d_{1}$, or $a=d_{i}$ and $b=d_{i+1}$.
	\end{itemize}
	By the construction and the fact that no $y$ with $A(y)\in \q'$ can be reached from $x$ along an $r$-path, $\mathcal{I}$ is a model of $\mathcal{O}$ and $\q'$ refuting $\q_{n+l}$.
\end{proof}

\counterexampletwo*

\begin{proof}
	Let $\mathcal{O}=\{ \text{fun}(P), \text{fun}(P^-),B \sqcap \exists P^{-}\sqsubseteq \bot\}$ and $\q=A$. We show that $Q = \{\q\}$ does not have a finite split partner \wrt $\mathcal{O}$ within ELIQ$^{\{A,B,P\}}$. For suppose $\mathcal{S}(Q)$ is such a split-partner. Then there exists $(\mathcal{A},a) \in \mathcal{S}(Q)$ with $\mathcal{O},\A\models B \sqcap \exists P^{n}.\top(a)$ for all sufficiently large $n$ because
	$B \sqcap \exists P^{n}.\top \not\models_{\mathcal{O}} A$. Then $\A$ must contain $n$ nodes if $\mathcal{O}$ and $\A$ are satisfiable, so $\mathcal{S}(Q)$ is infinite.
	
	On the other hand, $\{\top\}$ is a frontier for $A$ \wrt $\mathcal{O}$ within ELIQ.
\end{proof}
The following example shows that even by taking
frontiers and splittings together we do not obtain a universal method for constructing unique characterisations with a single positive example.

\begin{example}\em
	%
	Consider the ELIQ $\q=A\wedge B$ and the $\mathcal{ELIF}$-ontology $\mathcal{O}$ with the following axioms:
	$$
	A \sqsubseteq \exists R.A, \ \exists R.A \sqsubseteq A, \
	\text{fun}(P),\ \text{fun}(P^-),\ E \sqcap \exists P^{-}\sqsubseteq \bot.
	$$
	It can be shown in the same way as above that $\q$ has no frontier \wrt $\mathcal{O}$ within ELIQ and that $\q$ does not have any split-partner \wrt $\mathcal{O}$ within ELIQ$^{\{A,B,R,P,E\}}$. However, a unique characterisation of $\q$ wrt $\mathcal{O}$ within ELIQ is obtained by taking $E^{+}= \{\hat{\q}\}$
	and $E^{-}$ the same as in Example~\ref{count2}. (To show the latter one only has to observe that $E^{-}$ is a split-partner of $\q$ \wrt $\mathcal{O}$ within ELIQ$^{\{A,B,R\}}$ and that $\mathcal{O},\hat{\q}\not\models \r$ for any $\r$ containing any of the symbols $P$ or $E$.)
\end{example}

\section{Results on Meet-Reducibility}
	Recall that a query $\r \in \Q$ is called \emph{meet-reducible}~\cite{McKenzie1972EquationalBA} \wrt $\mathcal{O}$ within $\Q$ if there are queries $\r_{1},\r_{2} \in \Q$ such that
	$\r \equiv_{\mathcal{O}} \r_{1} \land \r_{2}$ and $\r \not\equiv_{\mathcal{O}} \r_{i}$, $i=1,2$.
	
	\begin{lemma}\label{lonecon}
		$(i)$ If an ontology $\mathcal{O}$ admits frontiers within $\Q$, then  $\q \in \Q$ is meet-reducible wrt $\mathcal{O}$ within $\Q$ iff $|\mathcal{F}_{\q}| \ge 2$ provided that $\q' \not \models_\mathcal{O} \q''$, for any distinct $\q',\q'' \in \mathcal{F}_{\q}$.
		
		$(ii)$ If an ontology $\mathcal{O}$ admits containment reduction within $\Q$, then, for any meet-reducible $\q \in \Q$ \wrt $\mathcal{O}$ within $\Q$, we have $|\N_{\q}| \ge 2$ for every characterisation of $\q$ \wrt $\mathcal{O}$ within $\Q$ of the form $(\{\hat{\q} \}, \N_{\q})$.
	\end{lemma}
	\begin{proof}
		$(i,\,\Leftarrow)$ Let $\r_1,\r_2 \in \mathcal{F}_{\q}$ be distinct and  $\r = \r_1 \land \r_2$. If $\q \not\equiv_\mathcal{O} \r$, then there is $\r' \in \mathcal{F}_{\q}$ with $\r' \models_\mathcal{O} \r$, and so $\r' \models_\mathcal{O} \r_i$, which is impossible.
		
		$(i,\,\Rightarrow)$ Suppose $\q \equiv_{\mathcal{O}} \q_{1} \land \q_{2}$ and $\q \not\equiv_{\mathcal{O}} \q_{i}$, for $i=1,2$. Then there are $\r_i \in \mathcal{F}_{\q}$ with $\r_i \not\equiv_{\mathcal{O}} \q_{i}$. Clearly, $\r_1$ and $\r_2$ are distinct because otherwise $\r_i \models_\mathcal{O} \q$.
		
		$(ii)$  Suppose $\q \equiv_{\mathcal{O}} \q_{1} \land \q_{2}$  and $\q \not\equiv_{\mathcal{O}} \q_{i}$, for $i=1,2$. If $\N_{\q} = \{\hat{\r}\}$, then $\TO, \hat{\r} \models \q_i$, for $i=1,2$, because if $\TO, \hat{\r} \not\models \q_i$, then $\q_i$ would fit $(\{\hat{\q} \}, \N_{\q})$, and so would be equivalent to $\q$ \wrt $\TO$, which is not the case.
	\end{proof}
	
	\begin{lemma}\label{lem:meet-reducible}
		$(i)$ Deciding whether an ELIQ $\q$ is meet-reducible \wrt to a $\DL_{\mathcal{F}}^{-}$-ontology is in \PTime.
		
		$(ii)$ Deciding whether an ELIQ $\q$ is meet-reducible \wrt to a $\DL_{\mathcal{H}}$-ontology is \textup{coNP}-complete.
	\end{lemma}
	\begin{proof}
		$(i)$ We first compute a frontier $\mathcal{F}_{\q}$ of $\q$
		in polynomial time~\cite{DBLP:conf/ijcai/FunkJL22}. Then we
		remove from $\mathcal{F}_{\q}$ every $\q''$ for which there is
		a different $\q' \in \mathcal{F}_{\q}$ with $\q' \models_{\TO}
		\q''$. This can also be done in polynomial time because ELIQ
		containment in $\DL_\mathcal{F}$ is tractable~\cite{DBLP:conf/ijcai/BienvenuOSX13}. It remains to use Lemma~\ref{lonecon} $(i)$ to check if the resulting set is a singleton.
		
$(ii)$ Assume $\TO$ and $\q$ are given. For the upper bound, first compute a frontier $\mathcal{F}_{\q}$ of $\q$
		in polynomial time~\cite{DBLP:conf/ijcai/FunkJL22}. To check that $\q$ is meet-reducible guess queries $\q_{1},\q_{2}\in \mathcal{F}_{\q}$ and witness models showing that $\q_{1}\not\models_{\TO}\q_{2}$ and $\q_{2}\not\models_{\TO}\q_{1}$. For the lower bound, consider the ontologies $\mathcal{O}$, ABox $\{A_{0}(a)\}$, and ELIQ $\q$ constructed in the proof of~\cite[Theorem 1]{DBLP:conf/dlog/KikotKZ11} for Boolean CNFs. Recall that $\q$ has a single answer variable $x$, an atom $A_0(x)$, and $A_0$ does not occur elsewhere in $\q$. As shown in that proof, the problem $\mathcal{O},A_{0}(a) \models \q(a)$ is NP-hard. Take a copy $\mathcal{O}'$ of $\TO$ obtained by replacing each predicate symbol $S$ in $\mathcal{O}$ except $A_0$ by a fresh $S'$. Similarly, take a copy $\q'$ of $\q$. We show that
\begin{equation}\label{eq:equiv}
  \q \land \q' \equiv_{\TO \cup \TO'} \q \text{ implies }\TO, A_0(a)\models \q(a)
\end{equation}
$(\Rightarrow)$ Suppose $\TO, A_0(a) \not\models \q(a)$, then there exists a model $\I'$ of $\TO'$ with $a^{\I'} \in A_0^{\I'}$ not satisfying $\q'$ at $a^{\I'}$. Let $a^{\I'} = d$. Take the interpretation $\mathcal J$ that looks like $\hat \q$, let its root be $d$. From~\cite[Theorem 1]{DBLP:conf/dlog/KikotKZ11} it follows that there exists $\I \supseteq \mathcal J$ such that $\I \models \TO$. Clearly, $\I \models \TO'$ (because all $S'^\I = \emptyset$). By taking the union of $\I$ and $\I'$, we obtain an interpretation that satisfies $\TO \cup \TO'$, $A_0(a)$, satisfies $\q$ at $d$ and does not satisfy $\q'$ at $d$. It follows $\q \not \models_{\TO \cup \TO'} \q \land \q'$.

Using\eqref{eq:equiv}, we now show that $\q\wedge \q'$ is not meet-reducible w.r.t. $\TO \cup \TO'$ iff $\mathcal{O},A_{0}(a) \models \q(a)$. If $\q\wedge \q'$ is not meet-reducible, then $\q \land \q' \equiv_{\TO \cup \TO'} \q$ and we are done. For the opposite direction, suppose $\mathcal{O},A_{0}(a) \models \q(a)$. We immediately observe that $\q \equiv_\TO A_0(x)$. It follows then that $\q \land \q' \equiv_{\TO \cup \TO'} A_0(x)$. Suppose $\q_1 \land \q_2 \equiv_{\TO \cup \TO'} \q \land \q'$, for some $\q_1, \q_2$. It follows that $\q_1 \land \q_2 \equiv_{\TO \cup \TO'} A_0(x)$. Clearly, some $\q_i$ contains $A_0(x)$ for otherwise we easily get $\q_1 \land \q_2 \not \models_{\TO \cup \TO'} A_0(x)$ contrary to our assumption. But then it follows that $\q_i \equiv_{\TO \cup \TO'} \q \land \q'$ and $\q\wedge \q'$ is not meet-reducible.
 	\end{proof}

\section{Comments for Section~\ref{sec:tempintro}}
We discuss the relationship between the epistemic semantics used in this article for temporal queries and Tarski semantics based on \emph{temporal structures} $\mathcal{I}$, which are a sequences $\I_{0},\I_{1},\dots$ of domain structures $\I_{i}$ as introduced above such that $a^{\I_{n}}=a^{\I_{m}}$, for any individual $a$ and $n,m\in \mathbb N$. $\I$ is a \emph{model} of $\D=\A_{0},\dots,\A_{n}$
if each $\I_{i}$ is a model of $\A_{i}$, for $i\leq n$; and $\I$ is a model of $\TO$ if each $\I_{i}$ is a model of $\TO$, for $i \in \mathbb N$. The truth relation $\I,\ell,a\models \q$ is then defined in the obvious way. We write  $\mathcal{O},\D,\ell,a \models_{T} \q$ if $\I,\ell,a\models \q$, for every model $\I$ of $\mathcal{O}$ and $\D$. It is easy to see that $\models$ coincides with $\models_{T}$ for any Horn ontology $\TO$, in particular, all \DL{} logics considered here and $\mathcal{ELHIF}$.
Thus, the results presented in this paper also hold under $\models_{T}$ if one considers such ontologies. In general, however, the two entailment relations do not coincide: consider the ontology $\TO=\{\top \sqsubseteq A \sqcup B\}$ and the data instance $\D=\emptyset,\A_{1},\emptyset,\A_{3}$ with $\A_{1}=\{A(a)\}$ and $\A_{3}=\{B(a)\}$. Then $\TO,\D,0,a\models_{T} \Diamond(A \wedge \nxt B)$ but
$\TO,\D,0,a\not\models \Diamond(A \wedge \nxt B)$. We leave an investigation of $\models_{T}$ in the non-Horn case for future work.

\section{Proof for Section~\ref{sec:six}}

\lemquerynormalform*

\begin{proof}
	The transformation is straightforward: to ensure {\bf (n1)}, drop any
	$\r_{0}^{i}$ and $\r_{k_{i}}^{i}$ for which {\bf (n1)} fails and add one
	$<$ to the relevant $\mathcal{R}_{i}$. To ensure
	{\bf (n2)}, replace any $\mathcal{R}_{i}$ containing at least one occurrence of $<$
	with the sequence obtained from $\mathcal{R}_{i}$ by dropping all occurrences of $\leq$ and replace any $\mathcal{R}_{i}$ not containing any occurrence of $<$
	by a single $\leq$. To ensure {\bf (n3)}, drop any
	$\r_{0}^{i+1}$ with $\r_{k_{i}}^{i}\models_{\mathcal{O}}
	\r_{0}^{i+1}$ if $\q_{i+1}$ is primitive and $\Rmc_{i+1}$ is $\le$.
	To ensure {\bf (n4)} drop any $\r_{k_{i}}^{i}$ with
	$\r_{0}^{i+1}\models_{\mathcal{O}}\r_{k_{i}}^{i}$ if $i>0$,
	$\q_{i}$ is primitive and $\Rmc_{i+1}$ is $\le$. Finally, to ensure {\bf (n5)}, replace $\Rmc_{i+1}$ by $<$ if
	$\r_{k_{i}}^{i}\wedge \r_{0}^{i+1}$ is not satisfiable \wrt $\mathcal{O}$ and $\Rmc_{i+1}$ is $\le$.
	
	The second part follows from the fact that query containment in $\DL_{\mathcal{F}}$ is in P and NP-hard in $\DL_{\mathcal{H}}$~\cite{DBLP:conf/dlog/KikotKZ11}.
\end{proof}

\thmfirst*

To prove Theorem~\ref{thm:first} we first provide some notation for talking about the entailment relation $\mathcal{O},\D \models \q$.
Let $\mathcal{D}=\mathcal{A}_{0},\ldots,\mathcal{A}_{n}$, $a\in \ind(\D)$, and let
$\q$ take the form \eqref{eq:cqform}.
A map $h \colon \tvar(\q) \to [0,\len(\D)]$ is called a \emph{root $\mathcal{O}$-homomorphism} from $\q$ to $(\mathcal{D},a)$ if
$h(t_{0})=0$, $\mathcal{O},\mathcal{A}_{h(t)}\models \r(a)$ if $\r(t)\in \q$, $h(t') = h(t) +1$ if $\suc(t,t') \in \q$, and $h(t)\, R\, h(t')$ if $R(t,t') \in \q$ for $R \in \{<,\leq\}$. It is readily seen that $\TO, \mathcal{D},a,0\models \q$ iff there exists a root $\mathcal{O}$-homomorphism from $\q$ to $(\D,a)$.

Let $b\geq 1$. The instance $\D$ is said to be \emph{$b$-normal \wrt $\mathcal{O}$} if it takes the form
\begin{equation}\label{eq:bnormaldata}
	\D=\D_{0}\emptyset^{b}\D_{1} \ldots \emptyset^{b}\D_{n},
	\text{ where } \D_{i}=\mathcal{A}_{0}^{i}\ldots\mathcal{A}_{k_{i}}^{i},
\end{equation}
with $b>k_{i}\geq 0$ and $\mathcal{A}_{0}^{i}\not\equiv_{\mathcal{O}}\emptyset$ if $i>0$, and $\mathcal{A}_{k_{i}}^{i}\not\equiv_{\mathcal{O}}\emptyset$ if $i>0$ or $k_{i}>0$ (thus, of all the first/last $\mathcal{A}$ in a $\D_{i}$ only $\mathcal{A}_0^0$ can be trivial). Following the terminology for queries, we call each $\mathcal{D}_{i}$ \emph{a block of} $\D$.
%
For a block $\D_{i}$ in $\D$, we denote by $I(\D_{i})$ the subset of $[0,\len(\D)]$ occupied by $\D_{i}$. Then we call a root $\mathcal{O}$-homomorphism $h \colon \q \to \D$ \emph{block surjective} if every $j\in I(\D_{i})$ with a block $\D_{i}$ is in the range $\textit{ran}(h)$ of $h$.
We next aim to state that after `weakening' the non-temporal data instances in blocks, no root $\mathcal{O}$-homomorphism from $\q$ to $\D$ exists. This is only required if the non-temporal data instances are obtained from queries in $\Q$, and so we express `weakening' for characterisations $(\{\hat{\s}\}, \N_{\s})$ of queries $\s$ in $\Q$ \wrt $\mathcal{O}$.

In detail, suppose $\A^{i}_{j} = \hat{\s}^{i}_{j}$, for queries $\s^{i}_{j}\in \Q$. Given $\s\in \Q$, take the characterisation $(\{\hat{\s}\}, \N_{\s})$ of $\s$ \wrt $\TO$ within $\Q$, where $\{\hat{\s}\}$ is the only positive data instance, and $\N_{\s}$ is the set of negative data instances. To make our notation more uniform, we think of the pointed data instances in $\N_{\s}$ as having the form $\hat{\s}'$, for a suitable CQ $\s'$ (which is  not necessarily in $\Q$).

For $\ell\in \textit{ran}(h)$, let
$$
\r_{\ell}=\bigwedge_{\r(t)\in \q, h(t)=\ell}\r.
$$
Then $h$ is called \emph{data surjective} if $\mathcal{O},\hat{\s}\not\models \r_{\ell}(a)$, for any $\s\in \N_{\s^{i}_{j}}$ and any $\ell\in \textit{ran}(h)$ such that ${\s}^{i}_{j}\not\equiv \top$, where
$\hat{\s}^{i}_{j}$ is the data instance placed at $\ell$ in $\D$.

We call the root $\mathcal{O}$-homomorphism $h:\q\rightarrow\mathcal{D}$ a \emph{root $\mathcal{O}$-isomorphism} if it is data surjective and,
for the blocks $\q_{0},\dots,\q_{m}$ of $\q$, we have $n=m$ and $h$ restricted to $\tvar(\q_{i})$ is a bijection onto $I(\D_{i})$ for all $i\leq n$ (in particular,  $h$ is block surjective). Intuitively, if we have a root $\mathcal{O}$-isomorphism $h \colon \q \rightarrow \mathcal{D}$, then $\q$ is almost the same as $\D$ except for differences between the sequences $\mathcal{R}_{i}$ in $\q$ and the gaps between blocks in $\D$.


Let $\Q$, $\mathcal{O}$, $\q$, and $\D$ be as before with $\D$ of the form \eqref{eq:bnormaldata}, where $\mathcal{A}^{i}_{j}=\hat{\s}^{i}_{j}$, for  $\s^{i}_{j}\in \Q$. The following rules will be used to define the negative examples in the unique characterisation of $\q$ and as steps in the learning algorithm. They are applied to $\D$:
\begin{enumerate}[label=(\alph*)]
	\item replace some $\hat{\s}^i_j$ with
	$\s^i_j\not\equiv_{\mathcal{O}}\top$ by an $\hat{\s} \in \N_{\s^i_j}$, for $i,j$ such that $(i,j)\not=(0,0)$---that is, the rule is not applied to $\s_{0}^{0}$;
	
	\item replace some pair $\hat{\s}^i_j\hat{\s}^i_{j+1}$
	within block $i$
	by $\hat{\s}^i_j\emptyset^{b}\hat{\s}^i_{j+1}$;
	
	\item replace some $\hat{\s}_j^i$ with $\s_{j}^{i} \not\equiv_{\mathcal{O}}\top$ by $\hat{\s}_j^i \emptyset^b \hat{\s}_j^i$,
	where $k_{i}>j>0$---that is, the rule is not applied to $\s_{j}^{i}$ if it is on the border of its block;
	
	
	\item replace $\hat{\s}^i_{k_i}$ ($k_i > 0$)
	by $\hat{\s}\emptyset^b \hat{\s}^i_{k_i}$, for some $\hat{\s} \in \N_{\s^i_{k_i}}$, or replace $\hat{\s}^i_{0}$ ($k_i > 0$) by $\hat{\s}^i_{0}  \emptyset^b \hat{\s}$, for some $\hat{\s}\in \N_{\s^i_{0}}$;
	
	\item[\rm (e)] replace $\hat{\s}_0^0$ with $\s^0_{0} \not\equiv_{\mathcal{O}}\top$ by $\hat{\s}\emptyset^{b}\hat{\s}^{0}_{0}$, for $\hat{\s}\in \N_{\s^0_0}$, if $k_0 = 0$, and by $\hat{\s}_0^0 \emptyset^b \hat{\s}_0^0$ if $k_0 > 0$.
\end{enumerate}
If $k_{i}=0$, $i>0$, and $\s^{i}_{0}$ is meet-reducible \wrt $\mathcal{O}$ within $\Q$, then we say that $\s^{i}_{0}$ is a \emph{lone conjunct \wrt $\mathcal{O}$ within $\Q$} in $\D$.

\begin{lemma}\label{lem:stepae}
	Assume $\Q,\mathcal{O}$, and $\q$ are as above. Let $b$ exceed the
	number of $\Diamond$ and $\nxt$ in $\q$, let $\q$ be in normal form,
	and let $\D$ be $b$-normal without lone conjuncts \wrt $\mathcal{O}$ within $\Q$. If $\mathcal{O},\D \models \q$ but $\mathcal{O},\D'\not\models \q$, for any $\D'$ obtained from $\D$ by applying any of the rules $(a)$--$(e)$, then any root $\mathcal{O}$-homomorphim $h \colon \q \rightarrow \D$ is a root $\mathcal{O}$-isomorphism.
\end{lemma}
\begin{proof}
	We assume that $\q$ is constructed using $\r_{j}^{i}\in \Q$ and $\D$ is constructed using $\s_{j}^{i}\in\Q$.
	Let $\mathcal{O},\D \models \q$. Take a root $\mathcal{O}$-homomorphism $h \colon \q\rightarrow \D$.
	
	Suppose first that $h$ is not block surjective. Since $h(0)=0$, we find $i,j$ with $(i,j)\not=(0,0)$ such that the time point of $\hat{\s}_{j}^{i}$ is not in the range of $h$. If $\hat{\s}_{j}^{i}$ is not on the border of its block, that is $0 < j <k_{i}$,  then we obtain from $h$ a root $\mathcal{O}$-homomorphism into the data instance $\D'$ obtained from $\D$ by rule (b) and have derived a contradiction. If $0=j$ or $k_{i}=j$, then we obtain from $h$ a root $\mathcal{O}$-homomorphism into the data instance $\D'$ obtained from $\D$ by rule (a) applied to $\s^{i}_{j}$ and have derived a contradiction.
	
	Assume next that $h$ is block surjective but not data surjective.
	Then we find $\ell \in \textit{ran}(h)$ with $\hat{\s}^{i}_{j}$ such that $\s^{i}_{j}\not\equiv_{\mathcal{O}}\top$ placed at $\ell$ such that $\mathcal{O},\hat{\s}\models \r_{\ell}(a)$ for some $\s\in \N_{\s^{i}_{j}}$.
	But then $h$ is a root $\mathcal{O}$-homomorphism
	into the data instance $\mathcal D'$ obtained from $\mathcal{D}$ by applying
	rule~(a) to $\s^i_j$, that is, replacing $\hat{\s}^{i}_{j}$ by $\hat{\s}$ with $\s\in \N_{\s^{i}_{j}}$.
	
	Suppose now that $h \colon \q \to \D$ is a block and data surjective,
	$(t\le t') \in \q$ and $h(t) = h(t') = \ell$ lies in the $i$th block of $\D$. 	
	Then $h^{-1}(\ell) = \{t_{1},\dots,t_{k}\}$ with $k \ge 2$ and $(t_{j}\leq t_{j+1}) \in \q$, $1 \le j < k$. Let $\r_{1},\ldots,\r_{k}$ be the queries with $\r_{j}(t_{j})$ in $\q$. As $\q$ satisfies {\bf (n1)} and {\bf (n2)}, there is $j$ with $\r_{j}\not\equiv_{\mathcal{O}}\top$.
	Hence, $\s^{i}_{j_{0}}\not\equiv_{\mathcal{O}}\top$ for the query $\s^{i}_{j_{0}}$ with $\hat{\s}^{i}_{j_{0}}$ at $\ell$ in $\D$. Moreover, by data surjectivity,
	$\r_{1} \wedge_{\mathcal{O}} \cdots \wedge_{\mathcal{O}} \r_{k}\equiv_{\mathcal{O}} \s^{i}_{j_{0}}$.
	Consider possible locations of $j_{0}$ in its block.
	
	\emph{Case} 1: $j_{0}$ has both a left and a right neighbour in its block. Then there is $\D'$ obtained by (c)---i.e., by replacing $\hat{\s}_{j_{0}}^i$ with $\hat{\s}_{j_{0}}^i \emptyset^b \hat{\s}_{j_{0}}^i$---and a root $\mathcal{O}$-homomorphism $h' \colon \q \to \D'$, which `coincides' with $h$ except that $h'(t_1)$ is the point with the first $\hat{\s}_{j_{0}}^i$ and  $h'(t_j)$, for $j = 2,\dots,k$, is the point with the second $\hat{\s}_{j_{0}}^i$.
	
	\emph{Case} 2: $j_{0}$ has no neighbours in its block and $i\not=0$,
	so this block is primitive and $\s^{i}_{j_{0}}$ is not equivalent to a
	conjunction of queries as $\mathcal{D}$ has no lone conjuncts by our
	assumption. Observe that the blocks of $t_{1},\ldots,t_{k}$ are all different and primitive. As $\s^{i}_{j_{0}}$ is not equivalent to a conjunction of queries, we have
	$$
	\r_{1} \equiv_{\mathcal{O}} \cdots \equiv_{\mathcal{O}} \r_{k}\equiv_{\mathcal{O}} \s^{i}_{j_{0}}.
	$$
	However, $\r_{j}\not\equiv_{\mathcal{O}}\r_{j+1}$ by {\bf (n3)} and {\bf (n4)}. Thus, Case 2 cannot happen.
	
	\emph{Case} 3: $j_{0}$ has a left neighbour in its block but no right neighbour. Then $\r_{1}\not\models_{\mathcal{O}}\r_{2}$ in view of {\bf (n3)}, and so $\r_{1}\not\models_{\mathcal{O}}\s^{i}_{j_{0}}$. As $\s^{i}_{j_{0}} \models_{\mathcal{O}} \r_{1}$, there is $\hat{\s}\in \N_{\s^{i}_{j_{0}}}$ with $\TO, \hat{\s} \not \models_{\mathcal{O}}\r_{1}$. Let $\D'$
	be obtained by the first part of (d) by replacing $\hat{\s}^i_{j_{0}}$ with $\hat{\s} \emptyset^b \hat{\s}^i_{j_{0}}$. Then there is a root $\mathcal{O}$-homomorphism $h' \colon \q \to \D'$ that sends $t_1$ to the point of $\s$ and the remaining $t_j$ to the point of $\s^i_{j_{0}}$.
	
	\emph{Case} 4: $j_{0}$ has a right neighbour in its block, $i \ne 0$, and it has no left neighbour. This case is dual to Case 3 and we use the second part of (d).
	
	\emph{Case} 5: $i=0$ and $j_{0}=0$. If block $0$ is primitive, then all of the $\r_{i}(t_i)$ are primitive blocks in $\q$. By {\bf (n3)}, $\r_{1}\not\models_{\mathcal{O}}\r_{2}$, and so $\r_{1}\not\models_{\mathcal{O}}\s^{i}_{j_{0}}$. Take $\hat{\s} \in \N_{\s^{i}_{j_{0}}}$. By the first part of (e), we have $\D'$ obtained by replacing $\hat{\s}_0^0$ with $\hat{\s}\emptyset^{b}\hat{\s}^{0}_{0}$. Then there is a root $\mathcal{O}$-homomorphism $h' \colon \q \to \D'$ that sends $t_1$ to the point of $\s$ and the remaining $t_j$ to the point of $\s^0_0$.
	
	Finally, if block $0$ is not primitive, the second part of (e) gives $\D'$ by replacing $\hat{\s}_0^0$ in $\D$ with $\hat{\s}_0^0 \emptyset^b \hat{\s}_0^0$. We obtain a root $\mathcal{O}$-homomorphism from $\q$ to $\D'$ by sending $t_1$ to the first $\hat{\s}_0^0$ and the remaining $t_j$ to the second $\hat{\s}_0^0$.
\end{proof}

We can now prove Theorem~\ref{thm:first}. Suppose an ontology $\mathcal{O}$ admits characterisations $(\{\hat{\s}\}, \N_{\s})$ of queries $\s$ \wrt $\TO$ within a class of domain queries $\Q$. Let $\q\in \LTL_p\NDD(\Q)$ in normal form \wrt $\mathcal{O}$ take the form~\eqref{fullq2appendix} with $\q_i$ of the form~\eqref{subqi}. We define an example set $E=(E^{+},E^{-})$ characterising $\q$ under the assumption that $\q$ has no lone conjuncts \wrt $\mathcal{O}$. Let $b$ be the number of $\nxt$ and $\Diamond$ in $\q$ plus 1.
For every block $\q_i$ of the form~\eqref{subqi}, let $\hat{\q}_{i}$ be the temporal data instance
\begin{align*}
	\hat{\q}_{i}= \hat{\r}_{0}^{i} \hat{\r}_{1}^{i} \dots \hat{\r}_{k_{i}}^{i} .
\end{align*}
For any two blocks $\q_i, \q_{i+1}$ such that $\r_{k_{i}}^{i}\wedge\r_{0}^{i+1}$ is satisfiable \wrt $\mathcal{O}$, we take the temporal data instance
\begin{align*}
	\hat{\q}_{i} \Join \hat{\q}_{i+1} = \hat{\r}_{0}^{i} \dots \hat{\r}_{k_{i-1}}^{i}  \widehat{\r_{k_{i}}^{i} \wedge \r_{0} ^{i+1}}  \hat{\r}_{1}^{i+1}\dots \hat{\r}_{k_{i+1}}^{i+1} .
\end{align*}
Now, the set $E^{+}$ contains the data instances given by
\begin{itemize}
	\item[--] $\D_{b} = \qw_{0} \emptyset^{b} \dots \qw_{i} \emptyset^{b} \qw_{i+1} \dots \emptyset^{b} \qw_{n}$,
	
	\item[--] $\D_{i} = \qw_{0} \emptyset^{b} \dots (\qw_{i} \! \Join \! \qw_{i+1}) \dots \emptyset^{b} \qw_{n}$ \  if $\mathcal{R}_{i+1}$ is $\leq$,
	
	\item[--] $\D_{i} = \qw_{0} \emptyset^{b} \dots \qw_{i} \emptyset^{n_{i+1}} \qw_{i+1} \dots \emptyset^{b} \qw_{n}$ \ otherwise.
\end{itemize}
Here, $\emptyset^b$ is a sequence of $b$-many $\emptyset$ and similarly for $\emptyset^{n_{i+1}}$.
The set $E^{-}$ contains all data instances of the form
\begin{itemize}
	\item[--] $\D_i^- = \qw_{0} \emptyset^{b} \dots \qw_{i} \emptyset^{n_{i+1} - 1} \qw_{i+1} \dots \emptyset^{b} \qw_{n}$ \  if $n_{i+1} > 1$,
	
	\item[--] $\D^-_i = \qw_{0} \emptyset^{b} \dots \qw_{i} \! \Join \! \qw_{i+1} \dots \emptyset^{b} \qw_{n}$ \ if $\mathcal{R}_{i+1}$ is a single $<$ and $\r_{k_{i}}^{i}\wedge \r_{0}^{i+1}$ is satisfiable \wrt $\mathcal{O}$,
	
	\item[--] the data instances obtained from $\D_{b}$ by applying to it each of the rules (a)--(e) in all possible ways exactly once.
\end{itemize}
We show that $E$ characterises $\q$. Clearly, $\mathcal{O},\D \models \q$ for all $\D \in E^+$. To establish $\mathcal{O},\D\not\models\q$ for $\D\in E^{-}$, we need the following:

\medskip
\noindent
\textbf{Claim 1.} $(i)$ \emph{There is only one root $\mathcal{O}$-homomorphism $h \colon \q \to \D_{b}$, and it maps isomorphically each $\tvar(\q_i)$ onto $I(\qw_i)$.}

$(ii)$ \emph{$\mathcal{O},\D^-_i \not\models \q$, for any $\mathcal{R}_i$ different from $\le$.}

$(iii)$ \emph{If $\D_{b}'$ is obtained from $\D_b$ by replacing some $\qw_{i}$ with $\qw'_{i}$ such that $\mathcal{O},\qw'_{i},\ell \not\models \q_i$ for any $\ell \le \max (\qw'_i)$, then $\mathcal{O},\D_{b}' \not\models \q$. In particular, $\mathcal{O},\D \not\models \q$, for all $\D \in E^-$.}

\smallskip
\noindent
\emph{Proof of claim.}
$(i)$ Let $h$ be a root $\mathcal{O}$-homomorphism. As $\q$ is in normal form and the gaps between $\qw_i$ and $\qw_{i+1}$ are not shorter than any block in $\q$, every $\tvar(\q_{i})$, where $\q_{i}$ is a block in $\q$, is mapped by $h$ to a single $I(\qw_j)$, where $\qw_{j}$ is a block of $D_{b}$. Hence we can define a function $f \colon [0,n] \rightarrow [0,n]$ by setting $f(i)=j$ if $f(\tvar(\q_{i}))\subseteq I(\qw_{j})$. Observe that $f(0)=0$ and $i<j$ implies $f(i)\leq f(j)$.
It also follows from the definition of the normal form that if $f(i)=i$, then $h$ isomorphically maps $\tvar(\q_i)$ onto $I(\qw_i)$ and $f(i-1)<i$ and $f(i+1)>i$
(observe that here ${\bf (n3)}$ and ${\bf (n4)}$ are required as they prohibit that $\tvar(\q_{i})$ and $\tvar(\q_{i+1})$ are merged if $R_{i+1}=\leq$ and $\tvar(\q_{i})$ or $\tvar(\q_{i+1})$ are a singleton). It remains to show that $f(i)=i$ for all $i$.

We first observe that $f(1) \geq 1$ and $f(j)=j$, for $j = \max \{i\mid f(i) \ge i\}$, from which again $f(j-1)<j$ and $f(j+1)>j$. Then we can proceed in the same way inductively by considering $h$ and $f$ restricted to the smaller intervals $[j,n]$ and $[0,j]$.

$(ii)$ Suppose $\mathcal{R}_i$ is not $\le$ but there is a root $\mathcal{O}$-homomorphism $h \colon \q \to \D^-_i$. Consider the location of $h(s^i_0) = \ell$. One can show similarly to $(i)$ that $\ell\in I(\qw_j)$ for some $j\geq i$. Since $\r^{i+1}_{k_{i+1}} \not\equiv_{\mathcal{O}}\top$ and by the construction of $\D^-_i$,   $h(s^{i+1}_{0})$ lies in some $I(\qw_j)$ with $j > i+1$. But then there is a root $\mathcal{O}$-homomorphism $h'\colon \q \to \D_b$ different from the one in $(i)$, which is impossible.

$(iii)$ is proved analogously. This completes the proof of the claim.

\medskip


Now assume that $\q'\in \Q$ in normal form is given and $\q'\not\equiv_{\mathcal{O}}\q$. We have to show that $\q'$ does not fit $E$.
If $\mathcal{O},\D_{b}\not\models \q'$, we are done as $\D_{b}\in E^{+}$. Otherwise, let $h$ be a root $\mathcal{O}$-homomorphism witnessing $\mathcal{O},\mathcal{D}_{b}\models \q'$. If $h$ is not a root $\mathcal{O}$-isomorphism, then by Lemma~\ref{lem:stepae},
there exists a data instance $\D$ obtained from $\D_{b}$ by applying one of the rules (a)--(e) such that $\mathcal{O},\D\models \q'$. As $\D\in E^{-}$, we are done.

So suppose $h \colon \q' \to \D_b$ is a root $\mathcal{O}$-isomorphism.
Then the difference between $\q'$ and $\q$ can only be in the sequences of $\Diamond$ and  $\Diamondw$ between blocks. To be more precise, $\q$ is of the form~\eqref{fullq2appendix},
\begin{align}
	\q' = \q_{0} \mathcal{R}'_{1} \q_{1} \dots \mathcal{R}'_{n} \q_{n}
\end{align}
and $\mathcal{R}_i \ne \mathcal{R}'_i$ for some $i$.
Four cases are possible:
\begin{itemize}
	\item[--] $\mathcal{R}_i = (r_0 \le r_1)$ and $\mathcal{R}_i' = (s_0 < s_1) \dots (s_{l-1} < s_l)$, for $l \ge 1$. In this case, $\mathcal{O},\D_i \not \models \q'$, for $\D_i \in E^+$.
	
	\item[--] $\mathcal{R}_i = (r_0 < r_1) \dots (r_{k-1} < r_k)$, $\mathcal{R}_i' = (s_0 < s_1) \dots$ $(s_{l-1} < s_l)$, for $l > k$. Then again $\mathcal{O},\D_i \not \models \q'$.
	
	\item[--] $\mathcal{R}_i = (r_0 < r_1) \dots (r_{k-1} < r_k)$, $\mathcal{R}_i' = (s_0 \le s_1)$, for $k \ge 1$. In this case $\mathcal{O},\D_i^- \models \q'$, for $\D^-_i \in E^-$. (Note that the compatibility condition is satisfied as $\q'$ is in normal form.)
	
	
	\item[--] $\mathcal{R}_i = (r_0 < r_1) \dots (r_{k-1} < r_k)$ and $\mathcal{R}_i' = (s_0 < s_1) \dots$ $(s_{l-1} < s_l)$, for $l < k$. Then again $\mathcal{O},\D_i^- \models \q'$.
\end{itemize}

We now show the converse direction in Theorem~\ref{thm:first} $(i)$. Suppose $\q$ in normal form~\eqref{fullq2appendix} does contain a lone conjunct $\q_i = \r$ \wrt $\mathcal{O}$ within $\Q$. Let $\r^-$ be the last query of the block $\q_{i-1}$ and let $\r^+$ be the first query of the block $\q_{i+1}$.

Now let $\r \equiv_\mathcal{O} \r_{1} \wedge \r_{2}$ and $\r_{i}\not\models_{\mathcal{O}}\r$, $i =1,2$. Observe that, for $\s\in \{\r^{-},\r^{+}\}$,
\begin{itemize}
	\item $\s\wedge \r_{i}$ is satifiable \wrt $\mathcal{O}$ if $\r^{-}\wedge \r$ is satisfiable \wrt $\mathcal{O}$;
	
	\item if $\s\not\models_{\mathcal{O}}\r$, then $\s\not\models_{\mathcal{O}}\r_{1}$ or $\s\not\models_{\mathcal{O}}\r_{2}$.
\end{itemize}
Hence one of the queries $\s'_1$ or $\s''_1$ below is in normal form:
\begin{align*}
	& \s'_1 = \q_{0} \mathcal{R}_{1} \dots \mathcal{R}_{i} \s_{1} (\le) \s_{2} \mathcal{R}_{i+1}  \dots \mathcal{R}_{n} \q_{n},\\
	& \s'_2 = \q_{0} \mathcal{R}_{1} \dots \mathcal{R}_{i} \s_{1} (\le) \s_{2} (\le) \s_{1}   \mathcal{R}_{i+1}  \dots \mathcal{R}_{n} \q_{n},
\end{align*}
where $\{\s_{1},\s_{2}\}=\{r_{1},\r_{2}\}$.
Pick one of $\s'_1$ and $\s'_2$, which is in normal form, and denote it by $\s'_1$. For $n \ge 2$, let $\s'_n$ be the query obtained from $\s'_1$ by duplicating $n$ times the part $\s_{1} (\le) \s_{2}$ in $\s'_1$ and inserting $\le$ between the copies. It is readily seen that $\s'_n$ is in normal form.
Clearly, $\q \models_{\mathcal{O}} \s'_n$ and, similarly to the proof of Claim~1, one can show that $\s'_n \not\models_{\mathcal{O}}\q$, for any $n \ge 1$.

Suppose $E = (E^+,E^-)$ characterises $\q$ and $n = \max \{\len(\D) \mid \D \in E^-\} + 1$. Then there exists $\D\in E^{-}$ with $\mathcal{O},\D\models \s'_{n}$, so we have a root $\mathcal{O}$-homomorphism $h \colon \s'_n \to \D$. By the pigeonhole principle, $h$ maps some variables of the queries $\s_{1},\s_{2}$ in $\s'_n$ to the same point in $\D$. But then  $h$ can be readily modified to obtain a root $\mathcal{O}$-homomorphism $h' \colon \q \to \D$, which is a contradiction. This finishes the proof of $(i)$.

\medskip

$(ii)$ follows from the proof of $(i)$ as $(E^{+},E^{-})$ is of polynomial size if the characterisations $(\{\hat{\s}\}, \N_{\s})$ of the domain queries $\s$ \wrt $\TO$ within $\Q$ are of polynomial size.

\medskip

$(iii)$ We aim to characterise $\q$ in normal form~\eqref{fullq2appendix}, which may contain lone conjuncts \wrt $\mathcal{O}$ within $\Q$ in the class of queries from $\LTL_p\NDD(\Q)$ of temporal depth at most $n = \dep(\q)$.
We first observe a variation of Lemma~\ref{lem:stepae}.
Extend the rules (a)--(e) by the following rule: if $\hat{\s}$ is a block in $\mathcal{D}$ with $\s$ a lone conjunct in $\D$, then let $\N_{\q}=\{\s_{1},\dots,\s_{k}\}$ with $\s_{i}\not\equiv_{\mathcal{O}}\s_{j}$, for $i\ne j$, and
\begin{description}
	\item[\rm (f$_{n}$)] replace $\s$ with $(\s_{1}\emptyset^{b}\cdots \emptyset^{b} \s_{k})^{n}$.
\end{description}
By Lemma~\ref{lonecon} $(ii)$, $|\N_{\q}|\geq 2$.
Now Lemma~\ref{lem:stepae} still holds if we admit lone conjuncts in $\D$ but only
consider $\q$ with at most $n$ blocks and add rule (f$_{n}$) to  (a)--(e). To see this, one only has to modify the argument for Case~2 in a straightforward way.
With the above modification of Lemma~\ref{lem:stepae}, we continue as follows.
The set $E^{+}$ of positive examples is defined as before. The set $E^{-}$ of negative examples is defined by adding to the set $E^{-}$ defined under $(i)$
the results of applying (f$_{n}$) to $\D_{b}$ in all possible ways exactly once.

For the proof that $(E^{+},E^{-})$ characterises $\q$ within the class of queries of temporal depth at most $n$, observe that $\mathcal{O},\D'\not\models \q$ for
the data instance $\D'$ obtained from $\D_{b}$ by applying (f$_{n}$).

\medskip

$(iv)$ Assume $\q \in \LTL_p\ND(\Q)$ is given. The proof of $(i)$  shows that $(E^{+},E^{-})$, defined in the same way as in $(i)$,  characterises $\q$ \wrt $\mathcal{O}$ within $\LTL_p\ND(\Q)$ even if $\q$ contains lone conjuncts: the proof of Lemma~\ref{lem:stepae} becomes much simpler as any block and type surjective root $\mathcal{O}$-homomorphism $h$ is now a root $\mathcal{O}$-isomorphism. Note that therefore rules (c), (d), and (e) are not needed. This completes the proof of Theorem~\ref{thm:first}.

\section{Proofs for Section~\ref{sec:seven}}

\thmmainuntil*

The proof is by reduction to the ontology-free propositional \LTL{} case. Namely, we require the following result proved in~\cite{DBLP:conf/kr/FortinKRSWZ22}, where $\P$ is the class of propositional queries (conjunctions of unary atoms:

\begin{theorem}[Fortin et al.~\citeyear{DBLP:conf/kr/FortinKRSWZ22}] \label{uno}
$\LTL_{pp}\UN(\mathcal{P}^{\sigma})$ is polysize characterisable within $\LTL_{p}\UN(\mathcal{P}^{\sigma})$ \wrt the empty ontology, with the characterisation defined below\textup{:}

Let $\s \in \mathcal{P}^\sigma \cup \{\bot\}$. We treat each such $\s \neq \bot$ as a set of its conjuncts and define $\bar{\s} = \{ A(a) \mid A(x) \in \s\}$. For $\s = \bot$, we set $\bar{\s} = \varepsilon$, where $\varepsilon$ is the empty word in the sense that $\varepsilon \D = \D$, for any data instance $\D$, and $\varepsilon \varepsilon = \varepsilon$.
	Consider $\q \in \LTL_{pp}\UN(\mathcal{P}^{\sigma})$ of the form~\eqref{upath}. Then $\q$ is uniquely characterised within $\LTL_{p}\UN(\mathcal{P}^{\sigma})$ by the example set $E = (E^+,E^-)$, where $E^+$ contains all data instances of the following forms\textup{:}
	\begin{description}
		\item[$(\mathfrak p_0)$] $\bar \r_0\dots \bar \r_n$,
		
		\item[$(\mathfrak p_1)$] $\bar \r_0\dots\bar \r_{i-1}\bar \el_i \bar \r_i \dots \bar \r_n$,
		
		\item[$(\mathfrak p_2)$] 
		$\bar \r_0 \dots \bar \r_{i-1} \bar \el_{i}^k \bar \r_{i} \dots \bar \r_{j-1} \bar \el_{j} \bar \r_j \dots \bar \r_n$, for  $i < j$, and $k = 1,2$ (where $\bar{\el}_{i}^k$ is a sequence of $k$-many $\bar{\el}_{i}$)\textup{;}
		
	\end{description}
	and $E^-$ contains all instances $\D$ with $\D\not\models \q$ of the forms\textup{:}
	\begin{description}
		\item[$(\mathfrak n_0)$] 
		$\bar \sigma^{n}$ and $\bar \sigma^{n-i} \overline{\sigma \setminus \{A\}} \bar \sigma^{i}$, for $A(x) \in\r_i$ (here, the whole $\sigma$ is regarded as a query),
		
		\item[$(\mathfrak n_1)$] $\bar \r_0\dots\bar \r_{i-1}\overline{\sigma \setminus X} \bar \r_i\dots \bar \r_n$, for $X=\{A,B\}$ with $A(x)\in \el_i$, $B(x)\in\r_i$,
		$X=\emptyset$, and $X=\{A\}$ with $A(x) \in \el_{i}$,
		
		\item[$(\mathfrak n_2)$]
		for all $i$ and $A(x) \in\el_i\cup\{\bot(x)\}$, \emph{some} data instance
		\begin{equation}\label{raznost}
			\D^i_{\!A} = \bar \r_0 \dots \bar \r_{i-1} (\overline{\sigma\setminus\{A\}}) \bar \r_i \bar \el_{i+1}^{k_{i+1}} \dots \bar \el_{n}^{k_n}\bar \r_n,
			%
		\end{equation}
		if any, such that $\max(\D^i_{\!A}) \le (n+1)^2$
		%
		%
		%
		and $\D^i_{\!A}\not\models \q^\dag$ for $\q^\dag$ obtained from $\q$ by replacing $\el_j$, for all $j \le i$, with $\bot$.
		(Note that $\D^i_{\!A} \not\models \q$ for peerless $\q$.)
	\end{description}
\end{theorem}

\smallskip

Returning to the proof of Theorem~\ref{thm:mainuntil}, assume a signature $\sigma$, an ontology $\mathcal{O}$ in $\sigma$ admitting containment reduction and general split-partners within $\Q^\sigma$, and a $\q \in \LTL_{pp}\UN(\Q^{\sigma})$ of the form \eqref{upath} are given.
We may assume that $\r_{n}\not\equiv_{\mathcal{O}} \top$.
We obtain the set $E^{+}$ of positive examples by taking the following data instances:
\begin{description}
	\item[$(\mathfrak p_0')$] $\hat{\r}_{0}\dots\hat{\r}_n$,
	
	\item[$(\mathfrak p_1')$] $\hat{\r}_0\dots\hat{\r}_{i-1}\hat{\el}_i \hat\r_i \dots \hat{\r}_n = \D^i_\q$,
	
	\item[$(\mathfrak p_2)'$] 
	$\hat{\r}_0 \dots \hat{\r}_{i-1} \hat{\el}_{i}^k \hat{\r}_{i} \dots \hat{\r}_{j-1} \hat{\el}_{j}\hat{\r}_j\dots\hat{\r}_n=\D^j_{i,k}$, for  $i < j$ and $k = 1,2$.
	
\end{description}
We obtain the set $E^{-}$ of negative examples by taking the following data instances $\mathcal{D}$ whenever $\mathcal{D}\not\models \q$:
\begin{description}
	\item[$(\mathfrak n_0')$] 
	$\mathcal{A}_{1},\ldots,\A_{n}$ and $\mathcal{A}_{1},\ldots,\A_{n-i},\A,\A_{n-i+1},\ldots,\A_{n}$, for $(\mathcal{A}, a) \in \mathcal{S}(\{\r_i\})$ and $(\A_{1}, a),\ldots,(\A_{n}, a) \in \mathcal{S}_\bot$;
	
	\item[$(\mathfrak n_1')$] $\hat{\r}_0\dots\hat{\r}_{i-1}\mathcal{A}\hat{\r}_i\dots\hat{\r}_n$, where $(\mathcal{A}, a) \in \mathcal{S}(\{\el_{i},\r_{i}\})\cup \mathcal{S}(\{\el_{i}\}) \cup \mathcal{S}_\bot$;
	
	\item[$(\mathfrak n_2')$]
	for all $i$ and $(\mathcal{A}, a) \in \mathcal{S}(\{\el_i,\r_{i}\})\cup\mathcal{S}(\{\el_i\})\cup\mathcal{S}_\bot$, \emph{some} data instance
	$$
	\D^i_{\mathcal{A}} = \hat{\r}_{0}\dots \hat{\r}_{i-1} \mathcal{A} \hat{\r}_{i}\hat{\el}_{i+1}^{k_{i+1}}\hat{\r}_{i+1} \dots \hat{\el}_{n}^{k_n}\r_n,
	$$
	if any, such that $\max(\D^i_{\mathcal{A}}) \le (n+1)^2$ and $\D^i_{\mathcal{A}}\not\models \q^\dag$ for $\q^\dag$ obtained from $\q$ by replacing all $\el_j$, for  $j \le i$, with $\bot$.
\end{description}
We show now that $\q$ is uniquely characterised by the constructed example set $E=(E^{+},E^{-})$ \wrt $\mathcal{O}$ within $\LTL_{p}\UN(\Q^{\sigma})$.
Consider any query
\begin{equation*}
	\q' = \r'_0 \land (\el'_1 \U (\r'_1 \land ( \el'_2 \U ( \dots (\el'_m \U \r'_m) \dots ))))
\end{equation*}
in $\LTL_{p}\UN(\Q^{\sigma})$ such that $\q'\not\equiv_{\mathcal{O}}\q$.  We can again assume that $\r_{m}'\not\equiv_{\mathcal{O}}\top$. Thus, in what follows we can safely ignore what the ontology $\mathcal{O}$ entails after the timepoint $\max\D$, for any database $\D$, as these points do not contribute to entailment of $\q$ or $\q'$.  In order to show that $\q$ fits $E$ and $\q'$ does not fit $E$, we need a few definitions.

We define a map $f$ that reduces the 2D case to the 1D case. Consider the alphabet
\begin{equation*}
	\Gamma =\{\r_0,\dots,\r_n,\el_1,\dots,\el_n,\r'_0,\dots,\r_m,\el'_1,\dots,\el'_m, \}\setminus\{\bot\},
\end{equation*}
in which we regard the CQs $\r_i,\el_i,\r'_j,\el'_j$ as symbols.  Let $\hat{\Gamma}=\{(\hat{\avec a},a) \mid \avec a \in\Gamma\}$, that is, $\hat{\Gamma}$ consists of the pointed databases corresponding to the CQs $\avec{a} \in \Gamma$. For any CQ $\avec{a}$, we set
$$
f(\avec{a}) = \{\avec{b}(x) \mid \avec{b} \in \Gamma \text{ and } \mathcal{O},\hat{\avec{a}} \models \avec{b}(a)\}.
$$
Similarly, for any pointed data instance $(\A,a)$, we set
$$
f(\A,a) = \{ \avec{b}(x) \mid \avec{b} \in \Gamma \text{ and } \mathcal{O},\A \models \avec{b}(a)\}
$$
and, for any temporal data instance $\D = \A_{0},\ldots,\A_{k}$ with a point
$a$, set
\begin{equation*}
	f(\D,a) = (f(\A_0,a), \dots ,f(\A_{k},a)),
\end{equation*}
which is a temporal data instance over the signature $\Gamma$. Finally, we define an $\LTL_{p}\UN(\mathcal{P}^{\Gamma})$-query
\begin{equation*}
	f(\q) = \rho_0 \land (\lambda_1 \U (\rho_1 \land (\lambda_2 \U( \dots (\lambda_n \U \rho_n) \dots ))))
\end{equation*}
by taking $\rho_i = f(\r_i)$ and  $\lambda_i = f(\el_i)$, and similarly for $\q'$.
It follows immediately from the definition that, for any data instance $\D$, we have $\mathcal{O},\D \models \q$ iff $f(\D,a) \models f(\q)$ and $\mathcal{O},\D\models f(\q')$ iff $f(\D,a)\models f(\q')$.

We first observe that $f(\q)$ is a peerless $\LTL_{p}\UN(\mathcal{P}^{\Gamma})$-query:
indeed, since $\mathcal{O},\hat\r_i\not\models\el_i(a)$, we have $\el_i\in f(\el_i)\setminus f(\r_i)$, and since $\mathcal{O},\hat\el_i\not\models\r_i(a)$, we have $\r_i\in f(\r_i)\setminus f(\el_i)$. It follows that $f(\q) \not\equiv f(\q')$

Let $E_{\text{prop}}=(E_{\text{prop}}^{+},E_{\text{prop}}^{-})$ be the example set defined for $f(\q)$ using $(\mathfrak p_0)$--$(\mathfrak p_2)$ and $(\mathfrak n_0)$--$(\mathfrak n_2)$. By Theorem~\ref{uno}, $f(\q)$ fits $E_{\text{prop}}$ and $f(\q')$ does not fit $E_{\text{prop}}$.

A \emph{satisfying root $\mathcal{O}$-homomorphism}
for any query
$$
\r_0 \land (\el_1 \U (\r_1 \land (\el_2 \U( \dots (\el_n \U \r_n) \dots ))))
$$
in $\D,a=(\A_{0},a)\ldots,(\A_{k},a)$ is a map $h$ from $\{0,\ldots,n\}$ to $\mathbb{N}$  such that
$h(0)=0$ and $h(i)<h(i+1)$ for $i<n$ and
\begin{itemize}
	\item $\mathcal{O},\A_{f(i)}\models \r_{i}(a)$;
	\item $\mathcal{O},\A_{i'} \models \el_{i}(a)$ for all $i'\in (f(i),f(i+1))$.
\end{itemize}
Clearly, such a root $\mathcal{O}$-homomorphism exists iff the query is satisfied in $\D,a$. If the query is in $\LTL_{p}\UN(\mathcal{P}^{\sigma})$ and $\TO$ is empty, then we call the homomorphism above a \emph{satisfying homomorphism}.

We are now in a position to show that $\q$ fits $E$ but $\q'$ does not fit $E$.
It is immediate from the definitions that $\q$ fits $E$. So we show that $\q'$ does not fit $E$.

Assume first that $f(\q')$ is not entailed by some example in
$E^{+}_{p}$. Then $\q'$ is not entailed by some example in $E^{+}$ as the  examples from $(\mathfrak p_0)$--$(\mathfrak p_2)$ are exactly the $f$-images of the examples $(\mathfrak p'_0)$--$(\mathfrak p'_2)$.

Assume now that $f(\q')$ is entailed by all data instances in $E_{\text{prop}}^{+}$ and
is also entailed by some $\D$ from $E_{\text{prop}}^{-}$. We show that then there is a data instance in $E^{-}$ that entails $\q'$ under $\mathcal{O}$.

If $\D=\Gamma^n$, then it follows that the temporal depth of $f(\q')$ is less than the temporal depth of $f(\q)$. Then $m<n$ and the query $\q'$ is entailed by some $\A_{1},\ldots,\A_{n}\in E^{-}$ with $(\A_{i}, a) \in \mathcal{S}_\bot$: we obtain $\A_{i}$ by taking $(\A_{i}, a) \in \mathcal{S}_\bot$ such that $\mathcal{O},\A_{i}\models \r_{i}'(a)$.

Suppose $\D=\Gamma^{n-i}(\Gamma\setminus \{\avec{a}\})\Gamma^{i}\models f(\q')$. Observe that the only satisfying homomorphim that witnesses this is the identity mapping. So we have $f(\r_{n-i})\not\subseteq\Gamma\setminus\{\avec{a}\}$ and therefore $\mathcal{O},\hat\r_{n-i}\models \avec a(a)$ but $f(\r'_{n-i})\subseteq\Gamma\setminus\{\avec{a}\}$. Then $\mathcal{O},\hat{\r}'_{n-i} \not\models \avec a(a)$, and so $\r'_{n-i}\not\models_{\mathcal{O}}\r_i$.  Therefore, there is $(\mathcal A,a)\in \mathcal S(\{\r_{n-i}\})$ such that $\mathcal{O},\mathcal A\models\r'_{n-i}(a)$. Observe that also $\mathcal{O},\mathcal A\not\models\r_{n-i}(a)$.

Now take, for any $j\not=n-i$, some $(\A_{j}, a) \in \mathcal{S}_\bot$ with $\mathcal{O},\A_{j}\models \r_{j}'$.
Then
\begin{align*}
&	\mathcal{O},\mathcal A_{0}\cdots\A_{n-i-1}\mathcal A\mathcal A_{n-i+1}\cdots\A_{n}\not\models \q, \\
&	\mathcal{O},\mathcal A_{0}\cdots\A_{n-i-1}\mathcal A\mathcal A_{n-i+1}\cdots\A_{n}\not\models \q'.
\end{align*}

Assume next that $\D$ is from $(\mathfrak n_1)$. We have $\D\models f(\q')$ and
$\D\not\models f(\q)$.
As $\D\models f(\q')$, we have a satisfying homomorphism $h$ for $f(\q')$ in $\D$.

If there is $j$ such that $h(j)=i$, then let $\r=\r_{j}$. Otherwise, there is $j$ such that $h(j)<i<h(j+1)$. Then let $\r=\el_{j}$. In both cases $f(\r)\subseteq Y$, where $Y$ depends on $\D$ and is either:
\begin{enumerate}
	\item $\Gamma\setminus \{\avec{a},\avec{b}\}$ with $\mathcal{O},\hat \el_{i}\models \avec{a}(a)$ and $\mathcal{O},\hat \r_i\models \avec{b}(a)$ or
	\item $\Gamma\setminus \{\avec{a}\}$ with $\mathcal{O},\hat \el_{i}\models \avec{a}(a)$ or
	\item $\Gamma$ (only if $\el_{i}=\bot$) or
	\item $\Gamma\setminus \{\avec{b}\}$.
\end{enumerate}

\emph{Case} 1. We have $\mathcal{O},\hat \r\not\models \avec{a}(a)$ and $\mathcal{O},\hat \r\not\models \avec{b}(a)$. Hence $\mathcal{O},\hat \r\not\models \el_{i}(a)$ and $\mathcal{O},\hat \r\not\models \r_{i}(a)$.
By the definition of split-partners,
there exists $(\A,a) \in \mathcal{S}(\{\el_{i},\r_{i}\})$ such that $\mathcal{O},\A\models \r(a)$. But then $h$ is also a satisfying root $\mathcal{O}$-homomorphism in $\D', a$ witnessing that $\q'$ is entailed by $\D'=\hat{\r}_{0}\dots \hat{\r}_{i-1} \mathcal{A} \hat{\r}_{i}\hat{\r}_{i+1} \dots \r_n$ \wrt $\mathcal{O}$.

It remains to show that $\mathcal{O},\D'\not\models \q$. Assume otherwise. Take a satisfying root $\mathcal{O}$-homomorphism $h^{\ast}$ witnessing $\mathcal{O},\D'\models \q$. By peerlessness
of $\q$, $h^{\ast}(j)=j$ for all $j<i$. But then $\mathcal{O},\A \models \el_{i}$ or $\mathcal{O},\A \models \r_{i}$ which both contradict to $(\A,a) \in \mathcal{S}(\{\el_{i},\r_{i}\})$.

\emph{Case} 2. We have $\mathcal{O},\hat \r\not\models \avec{a}(a)$. Hence $\mathcal{O},\hat \r\not\models \el_{i}(a)$. We now distinguish two cases.
If also $\mathcal{O},\hat \r\not\models \r_{i}(a)$, then we proceed as in the previous case and choose a split-partner $(\A,a) \in \mathcal{S}(\{\el_{i},\r_{i}\})$ such that $\mathcal{O},\A\models \r(a)$. We proceed as in Case 1.

If $\mathcal{O},\hat \r\models \r_{i}(a)$, then we proceed as follows.
Choose a split-partner $(\A,a) \in \mathcal{S}(\{\el_{i}\})$ such that $\mathcal{O},\A\models \r(a)$.
Then $h$ is also a satisfying root $\mathcal{O}$-homomorphism in $\D', a$ witnessing that $\q'$ is entailed by $\D'=\hat{\r}_{0}\dots \hat{\r}_{i-1} \mathcal{A} \hat{\r}_{i}\hat{\r}_{i+1} \dots \r_n$ \wrt $\mathcal{O}$.

It remains to show that $\mathcal{O},\D'\not\models \q$. Assume otherwise. Take a satisfying root $\mathcal{O}$-homomorphism $h^{\ast}$ in $\D', a$ witnessing $\D'\models \q$. By peerlessness
of $\q$, $h^{\ast}(j)=j$ for all $j<i$. Then $h^{\ast}(i)=i$ as
$\mathcal{O},\A \models \el_{i}$ would contradict $(\A,a) \in \mathcal{S}(\{\el_{i}\})$.
But then $h^{\ast}$ is a satisfying homomorphism in $\D, a$ witnessing $\D\models f(\q)$ and we have derived a contradiction.

\emph{Case} 3. We set $\D'=\hat{\r}_{0}\dots \hat{\r}_{i-1} \mathcal{A} \hat{\r}_{i}\hat{\el}_{i+1}^{k_{i+1}}\hat{\r}_{i+1} \dots \hat{\el}_{n}^{k_n}\r_n$ for some $(\A, a) \in \mathcal{S}_\bot$ with $\mathcal{O},\A\models \r_{i}'$. It directly follows from $\D\models f(\q')$ that $\mathcal{O},\D'\models\q'$ and also from
$\D\not\models f(\q)$ that $\mathcal{O},\D'\not\models \q$.

\emph{Case} 4. We have $\mathcal{O},\hat{\r}\not\models \avec{b}(a)$. Hence $\mathcal{O},\hat{\r}\not\models \r_{i}(a)$. We distinguish two cases.
If also $\mathcal{O},\hat \r\not\models \el_{i}(a)$, then we proceed as in Case 1
and choose split-partner $(\A,a) \in \mathcal{S}(\{\el_{i},\r_{i}\})$ such that $\mathcal{O},\A\models \r(a)$.

If $\mathcal{O},\hat \r\models \el_{i}(a)$, then we proceed as follows.
Choose a split-partner $(\A,a) \in \mathcal{S}(\{\r_{i}\})$ such that $\mathcal{O},\A\models \r(a)$.
Then $h$ is also a satisfying root $\mathcal{O}$-homomorphism in $\D', a$ witnessing that $\q'$ is entailed by $\D'=\hat{\r}_{0}\dots \hat{\r}_{i-1} \mathcal{A} \hat{\r}_{i}\hat{\r}_{i+1} \dots \r_n$ \wrt $\mathcal{O}$.

It remains to show that $\mathcal{O},\D'\not\models \q$. Assume otherwise. Take a satisfying root $\mathcal{O}$-homomorphism $h^{\ast}$ in $\D',a$ witnessing $\mathcal{O},\D'\models \q$. By peerlessness
of $\q$, $h^{\ast}(j)=j$ for all $j<i$. Then $h^{\ast}(i)>i$ as
$\mathcal{O},\A \models \r_{i}(a)$ would contradict the definition of $\A$.
But then, as $\Gamma\setminus \{\avec{b}\}\models f(\el_{i})$,
$h^{\ast}$ is also a satisfying homomorphism in $\D, a$ witnessing $\D\models f(\q)$ and we have derived a contradiction.

The case when $\D$ is from $(\mathfrak n_2)$ is considered similarly to the case of $(\mathfrak n_1)$.

\frontiersnotworking*

\begin{proof}
Consider the ontology
$$
\mathcal{O}=\{ \text{fun}(P), \text{fun}(P^-),B \sqcap \exists P^{-}\sqsubseteq \bot\}
$$
from the proof of Theorem~\ref{counterexample}. We know from~\cite{DBLP:conf/ijcai/FunkJL22} and that proof that $\TO$ admits frontiers within ELIQ$^{\{A,B,P\}}$ but not split-partners. We show that the query $\nxt A$ is not uniquely characterisable \wrt $\TO$ within $\LTL_{p}\UN(\text{ELIQ}^{\{A,B,P\}})$. Indeed, suppose $E=(E^{+},E^{-})$ is such a unique characterisation.

Consider the following set of pointed data instances:
\begin{multline*}
\mathcal{S}(\{A\}) = \{(\A_{i},a) \mid i >0,\ \exists \D = \A_{0},\dots,\A_{n} \in E^-,\\  \TO,\A_{j}\not\models A(a) \text{ for $0<j\leq i$}\}.
\end{multline*}
We claim that the defined $\mathcal{S}(\{A\})$ is a split-partner for $\{A\}$ within ELIQ$^{\{A,B,P\}}$, which is a contradiction.


Take any $\q'\in \text{ELIQ}^{\{A,B,P\}}$.
If $\TO,\A\models \q'(a)$, for some $(\A,a)\in \mathcal{S}(\{A\})$, then $\q'\not\models_{\TO} \nxt A$ because otherwise $\TO,\A\models A(a)$ which is not the case by definition of $\mathcal{S}(\{A\})$.

Now suppose $\TO,\A\not\models \q'(a)$ for all $(\A,a)\in
\mathcal{S}(\{A\})$. Then $\D\not\models \q'\U A$ for all $\D$ of the
form $\A_0,\ldots,\A_{n}$ in $E^{-}$ with $\TO,\A_{1}\not\models
A(a)$. Hence $\D\not\models \q'\U A$ for all $\D\in E^{-}$. On the other hand, from $\nxt A \models \q'\U A$ we obtain $\D\models \q'\U A$ for all $\D\in E^{+}$, and so $\q'\U A$ is equivalent to $\nxt A$ \wrt $\TO$. By the shape of $\TO$, this implies that $\q'$ is equivalent to $\bot$, and so $\q'\models_{\TO} A$, as required by the definition of split-partners.
\end{proof}


\section{Proofs for Section~\ref{sec:learning}}

This section is mainly devoted to give a full proof of
Theorem~\ref{thm:learning}, but we need some preparation.

\subsection{Normal Form}

In order to lift some results obtained in the atemporal case~\cite{DBLP:conf/dlog/FunkJL22}
to the temporal setting, we have to rely on the same normal form for
ontologies. 
An \ELHIF ontology is in \emph{normal form} if every concept inclusion
takes one of the following forms:
\[
A\sqsubseteq \exists R.A',\quad\quad \exists R.A\sqsubseteq A',\quad\quad
A\sqcap A'\sqsubseteq B,
\]
where $A,A'$ are concept names or $\top$, $B$ is a concept name or
$\bot$, and $R$ is a role. 

We describe next how to convert an \ELHIF ontology \Omc into an
\ELHIF ontology $\Omc'$ in normal form. Let us use $\Cmf(\Omc)$ to denote
the set of all concepts that occur in a concept inclusion in \Omc.
Note that $\Cmf(\Omc)$ is closed under taking sub-concepts.  We
introduce a fresh concept name $X_C$ for every complex concept $C \in
\Cmf(\Omc)$, and set $X_\bot = \bot$ and $X_A=A$ for concept names
$A\in\Cmf(\Omc)$.  The ontology $\Omc'$ consists of all functionality
assertions and all role inclusions in \Omc and the following concept inclusions: 
\begin{itemize}

  \item $X_C \sqsubseteq X_D$ for every 
    $C \sqsubseteq D \in \Omc$;

  \item $X_{D_1\sqcap D_2} \sqsubseteq X_{D_i}$ and $X_{D_1}\sqcap
    X_{D_2}\sqsubseteq X_{D_1\sqcap D_2}$, for every $D_1\sqcap
    D_2\in \Cmf(\Omc)$ and $i\in\{1,2\}$;

  \item $X_{\exists R.C}\sqsubseteq \exists R.X_C$ and  $\exists
    R.X_C\sqsubseteq X_{\exists R.C}$, for every $\exists R.C\in
    \Cmf(\Omc)$.

\end{itemize}
Clearly, $\Omc'$ can be computed in polynomial time. 
Regarding
the relationship between \Omc and $\Omc'$, we observe the following
consequences of the definition of $\Omc'$.
\begin{lemma}
  \label{lem:normalform}
  ~\\[-5mm]
  \begin{enumerate}

    \item $\Omc'$ is a conservative extension of $\Omc$;


    \item $\text{sig}(\Omc') = \text{sig}(\Omc)\cup \{X_C\mid
      C\in\Cmf(\Omc)\}$;

    \item $\Omc'\models X_C\equiv C$, for all $C\in \Cmf(\Omc)$.

  \end{enumerate}

\end{lemma}

Lemma~\ref{lem:normalform} essentially says that $\Omc'$ is a
conservative extension of $\Omc$, but is slightly stronger in 
also making precise how exactly a model of $\Omc$ can be extended
to a model of $\Omc'$.

\medskip We next show that it suffices to provide learning algorithms
\wrt ontologies in normal form.

\begin{lemma} \label{lem:obda-normalform}
  Let $\Lmc$ be an ontology language contained in $\ELHI$ or $\ELIF$. 
  If a class $\Q\subseteq \LTL_p\NDD(\text{ELIQ})$ of queries is
  polynomial query learnable \wrt \ELHIF ontologies in normal form using membership queries, then the same is
  true for \Lmc ontologies. If, additonally, $\Lmc$ admits
  polynomial time instance checking, then even polynomial time
  learnability is preserved.
\end{lemma}

\begin{proof}
  Let $L'$ be a polynomial time learning algorithm for \Qmc \wrt
  ontologies in normal form. We transform it into a polynomial
  time learning algorithm $L$ for \Qmc \wrt unrestricted \ELIF ontologies, relying on the normal form provided by
  Lemma~\ref{lem:normalform}. The construction for \ELHI is similar,
  and we strongly conjecture that it is possible to lift it to full
  \ELHIF but it is beyond the scope of the paper. 

  Given an \ELIF ontology $\Omc$ and a signature $\Sigma =
  \text{sig}(\Omc)$ with $\text{sig}(q_T)\subseteq \Sigma$, algorithm $L$ first computes the
  ontology $\Omc'$ in normal form as per Lemma~\ref{lem:normalform},
  choosing the fresh concept names so that they are not from $\Sigma$.
  It then runs $L'$ on $\Omc'$ and $\Sigma' = \Sigma \cup \text{sig}(\Omc')$.
  In contrast to $L'$, the oracle still works with the original ontology
  $\Omc$. To ensure that the answers to the queries posed to the oracle are
  correct, $L$ modifies $L'$ as follows.

  Whenever $L'$ asks a membership query $\Dmc',a$ with
  $\Dmc'=\Amc_0',\ldots,\Amc_n'$,
  we may assume that each $\Amc'_i$ satisfies the functionality assertions
  from \Omc, since otherwise the answer is trivially ``yes''. Then, 
  $L$ asks the membership query $\Dmc,a$,
  where $\Dmc$ is obtained from $\Dmc'$. Note that the \Dmc we are
  going to construct contains concept assertions $C(d)$ for complex
  concepts $C$, but these can be removed at the cost of introducing more
  fresh individuals and using standard concept assertions. 

  We start with setting 
  $\Amc_i=\Amc'_i\cup \{C(d)\mid X_C(d)\in \Amc_i'\}$, for all $i$
  and then extending the $\Amc_i$, for every role $R$ and every
  individual $b\in\text{ind}(\Amc_i')$ as follows: 
  \begin{itemize}

    \item[$(\dagger)$] Let $C_{R,b}$ be the set of all concepts $\exists R.D\in
      \Cmf(\Omc)$ such that $\Omc',\Amc_i'\models \exists R.D(b)$ but
      $\Omc',\Amc_i'\not\models D(b')$ for any $R(b,b')\in \Amc_i'$. Then 
      \begin{itemize}

	\item if $\text{fun}(R)\notin \Omc$, then add for each $\exists
	  R.D\in C_{R,b}$ one fresh individual $c$ together with
	  assertions $R(b,c),D(c)$;

	\item otherwise, add one fresh individual $c$ and add
	  assertions $R(b,c)$ and $D(c)$, for all $\exists R.D\in
	  C_{R,b}$.

      \end{itemize}

  \end{itemize}
  By the following claim, the answer to the modified membership
  query coincides with that to the original query.

  \medskip
  \noindent\textit{Claim~1.} $\Omc',\Dmc',0,a \models \q$ iff
  $\Omc,\Dmc,0,a \models \q$ for all $\q\in \LTL_p\NDD(\text{ELIQ})$
  that only use symbols from $\Sigma$,
  and all $a\in \text{ind}(\Dmc')$.

  \medskip

  \noindent\textit{Proof of Claim~1.} For ``if'', suppose that
  $\Omc,\Dmc,0,a \models \q$ and let $\Imc'$ be a model of $\Dmc'$ and
  $\Omc'$. We can assume that $\Delta^{\Imc'}$ does not mention any of
  the individuals that were introduced in the construction of \Dmc. We
  will construct a model $\Imc$ of $\Dmc$ and \Omc such
  that $(\Imc_i,a) \to (\Imc'_i,a)$, for every $0\leq i\leq n$. This
  clearly suffices since $\Imc,0,a\models \q$. 

  Fix some $i$ with $0\leq i\leq n$ and start with setting $\Imc_i$ to
  the restriction of $\Imc'_i$ to $\text{ind}(\Amc_i')$. Then process
  every individual $b\in \text{ind}(\Amc_i')$ and every role $R$.

  Let $C_{R,b}$ be the set of concepts in $(\dagger)$. We
  distinguish cases:

  \begin{itemize}

    \item If $\text{fun}(R)\notin \Omc$, process each $\exists R.D\in
      C_{R,b}$ as follows. By definition of $C_{R,b}$, we have
      $\Omc',\Amc_i'\models \exists R.D(b)$. As $\Imc_i'$ is a model of
      $\Omc'$ and $\Amc_i'$, there is an element $c$ with $(b,c)\in
      R^{\Imc'_i}$ and $c\in D^{\Imc'_i}$. Take the unraveling
      $\Jmc_c$ of $\Imc_i'$ at $c$, omit the $R^-$-successor of $c$ if
      $\text{fun}(R^-)\in \Omc$, and add the root of $\Jmc_c$ as an
      $R$-successor of $b$.

    \item If $\text{fun}(R)\in \Omc$, we proceed as follows. By
      definition of $C_{R,b}$, we have $\Omc',\Amc_i'\models
      \exists R.D(b)$, for all $\exists R.D\in C_{R,b}$. As
      $\Imc'$ is a model of $\Omc'$ and $\Amc_i'$, there is an element
      $c$ with $(b,c)\in R^{\Imc_i'}$ and $c\in D^{\Imc'_i}$, for all
      $\exists r.D\in C_{R,b}$. By definition of $C_{R,b}$ and since
      $\text{fun}(R)\in\Omc$, we know that there is no $b'\in
      \text{ind}(\Amc_i')$ with $(b,b')\in \Amc_i'$. Take the unraveling
      $\Jmc_c$ of $\Imc_i'$ at $c$, omit the $R^-$-successor of $c$ if
      $\text{fun}(R^-)\in \Omc$, and add the root of $\Jmc_c$ as an
      $R$-successor of $b$.

  \end{itemize}

  For the sake of completeness, we provide a formal definition of $\Jmc_c$. Its
  domain $\Delta^{\Jmc_c}$ consists of all sequences $a_0R_1a_1\ldots R_na_n$
  such that 
  \begin{itemize}

    \item $a_0=c$; 

    \item $a_i\in \Delta^{\Imc'}$, for all $i$ with $0\leq i\leq
      n$;

    \item $(a_i,a_{i+1})\in R_{i+1}^{\Imc'}$, for all $i$ with
      $0\leq i<n$;

    \item if $\text{fun}(R_i^-)\in \Omc$, then $R_{i+1}\neq R_i^-$,
      for all  $i$ with
      $0\leq i<n$;

    \item if $R_1=R^-$ then $\text{fun}(R^-)\notin \Omc$.

  \end{itemize}
  The interpretation of concept names $A\in \mn{N_C}$ and role names
  $r\in \mn{N_R}$ is then as expected:
  \begin{align*}
    A^{\Jmc_c} & = \{a_0R_1a_1\ldots R_na_n\in \Delta^{\Jmc_c}\mid
    a_n\in A^{\Imc'}\}\\
    r^{\Jmc_c} & = \{(\pi,\pi r a)\mid \pi r a\in
      \Delta^{\Jmc_c}\}\cup{}\\
      & \phantom{ {} = {}} \{(\pi r^- a,\pi)\mid \pi r^- a\in
	\Delta^{\Jmc_c}\}.
  \end{align*}
  Note that each $\Jmc_c$ has a homomorphism into $\Imc'$: just map
  every sequence $a_0R_1\ldots a_n$ to $a_n$.

  Let $\Imc$ be the result of the above process.
  Due to the initialization, we have $\Imc\models \Amc$. It is routine to
  verify that $\Imc$ is also a model of \Omc and that there is a
  homomorphism $(\Imc,a)\to (\Imc',a)$. 

  \smallskip

  For ``only if'', suppose that $\Omc',\Dmc',0,a\models \q$ and let
  $\Imc$ be a model of $\Dmc$ and \Omc. Since $\Omc'$ is a
  conservative extension of $\Omc$, there is a model $\Imc'$ of
  $\Omc'$ that coincides (in every time point) with $\Imc$ on $\Sigma$. Moreover, by Point~3
  of Lemma~\ref{lem:normalform}, it is also a model of $\Dmc'$. 
  It follows that $\Imc,0,a\models \q$ as required. This finishes the
  proof of Claim~1. 

  \medskip For \ELHI, we use the following variant of $(\dagger)$:
  \begin{itemize}
    \item[$(\ddagger)$] Let $C_{R,b}$ be the set of all concepts
      $\exists R.D\in \Cmf(\Omc)$ such that $\Omc',\Amc_i'\models
      \exists R.D(b)$. Then add for each $\exists R.D\in C_{R,b}$ one
      fresh individual $c$ together with assertions $R(b,c),D(c)$.

  \end{itemize}
  Now, polynomial query learnability is preserved simply due to the
  fact that the construction of $\Dmc$ from $\Dmc'$ is computable
  because instance checking \wrt \ELHIF ontologies is decidable.
  If the ontology language admits polynomial time instance checking,
  then the construction can actually be computed in polynomial time;
  thus polynomial time learnability is preserved. 
\end{proof}

\subsection{Generalisation Sequences} 

Generalisation sequences have been introduced as a generic tool to
show that exact learning algorithms in the atemporal case need only
polynomially many steps~~\cite{DBLP:conf/dlog/FunkJL22}. We recall the
definition. 

 A \emph{generalisation sequence for a CQ $\q$
\wrt \Omc} is a sequence
$\q_1,\q_2,\ldots$ of CQs that satisfies the following conditions, for all
$i\geq 1$: 
\begin{itemize}

  \item $\q_i\models_\Omc \q_{i+1}$ and $\q_{i+1}\not\models_\Omc \q_{i}$, and

  \item $\q_i\models_\Omc \q$.

\end{itemize}
Intuitively, a generalisation sequence is a sequence of weaker and
weaker CQs which, however, still entail $\q$ \wrt \Omc. We recall next that suitable generalisation sequences have
bounded length. 

Let us fix CQs $\q,\q_T$. We say that $\q$ is
\textit{$(\q_T,\Omc)$-minimal} if $\q'\not\models_\Omc \q_T$, for
every restriction $\q'$ of $\q$ to a strict subset of the variables in
$\q$.  For a variable $y\in\text{var}(\q)$, we denote with $\q(y)$,
the variant of $\q$ where the unique free variable is $y$. We then say
that $\q$ is \textit{\Omc-saturated} if $\q(y)\models_\Omc A(y)$
implies that $A(y)$ is a conjunct in $\q$, for every variable $y$ in
$\q$ and every concept name $A$ that occurs in \Omc. As usual, a CQ is
\emph{rooted} if the graph $(\text{var}(\q),\{ \{x,y\}\mid r(x,y)\in
\q\})$ is connected. Clearly, all ELIQs are rooted.

We recall Theorem~13 from~\cite{DBLP:conf/dlog/FunkJL22}, adapted to our
notation.
\begin{theorem} \label{thm:bounded-atemporal} Let \Omc be an \ELIF
  ontology in normal form, $\q_T$ be a rooted CQ, and
  $\q_1,\q_2,\ldots$ be a generalization sequence towards $\q$ \wrt
  \Omc such that $\q_1$ is satisfiable \wrt \Omc. If all $\q_i$ are
  $(\q_T,\Omc)$-minimal and \Omc-saturated, then the length of the
  sequence is bounded by a polynomial in the sizes of $\Omc$ and
  $\q_T$.  \end{theorem}
Using the same techniques it can be proved that
Theorem~\ref{thm:bounded-atemporal} is remains true for
\ELHIF ontologies.


We lift the notion of generalisation sequences to temporal data
instances as discussed in the main part of the paper, and show an
analogue of Theorem~\ref{thm:bounded-atemporal}. We repeat the
definition here for the sake of convenience. 

Let $\q_T\in
\LTL_p\NDD(\text{ELIQ})$ be a temporal query, and let us fix
throughout the rest of the subsection an individual name $a$. A sequence
$\D_{1},\ldots$ of temporal data instances is a \emph{generalisation
  sequence towards $\q_T$ \wrt $\mathcal{O}$} if for all $i\geq 1$:
  \begin{itemize}

  \item $\D_{i+1}$ is obtained from $\D_{i}$ by modifying one
    non-temporal CQ $\r_{j}$ in $\D_{i}$ to $\r_{j}'$ such that
    $\r_{j}\models_\Omc \r_{j}'$ and $\r_{j}'\not\models_\Omc \r_{j}$;

  \item $\mathcal{O},\D_{i},0,a\models \q_T$ for all $i\geq 1$.

\end{itemize}
The notion of $\Omc$-saturatedness lifts from CQs to temporal data
instances $\Dmc=\q_0\ldots\q_n$ as expected: \Dmc is $\Omc$-saturated
if every $\q_i$ is. We further say that \Dmc is
\emph{$(\q_T,\Omc)$-minimal} if the result $\Dmc'$ of dropping any
atom from any $\q_i$ satisfies $\Omc,\Dmc',0,a\not\models q_T$. The
\emph{support} $\text{supp}(\Dmc)$ of a temporal data instance
$\Dmc=\Amc_0\ldots\Amc_n$ the set of all $i$ such that $\Amc_i\neq
\emptyset$
\begin{lemma}\label{lem:temporal-genseq}
  Let $\q_T\in \LTL_p\NDD(\text{ELIQ})$. The length of a generalisation sequence
  $\Dmc_1,\ldots,\Dmc_n$ towards $\q_T$ \wrt \Omc such that all
  $\Dmc_i$ are satisfiable \wrt \Omc, $\Omc$-saturated, and
  $(\q_T,\Omc)$-minimal is bounded by a polynomial in the sizes of
  $\q_T$, \Omc, and $|\text{supp}(\Dmc_1)|$.
\end{lemma}
\begin{proof}
  Consider a time point $i$ and let $\r_{1}, \r_{2}, \ldots$ be the sequence of different queries at time point
  $i$ that occur in the generalisation sequence, that is, $\r_{j}\models_\Omc \r_{j+1}$ and
  $\r_{j+1}\not\models_\Omc \r_{j}$, for each $j$. Let $h$ be a root
  homomorphism from $\q_T$ to $\Dmc_n$ and let $I$ be the set of all
  $t$ with $h(t)=i$. (By construction, $h$ is a root homomorphism from
  $\q_T$ to all $\Dmc_j$.) Consider $\q'=\bigwedge_{i\in I}\q_i$.
  Clearly, $\r_1,\r_2,\ldots,$ is a generalisation sequence towards
  $\q'$ \wrt \Omc. Since all $\Dmc_j$ are satisfiable \wrt \Omc, \Omc-saturated
  and $(\q_T,a,\Omc)$-minimal, it follows that in particular, all
  $\r_1,\r_2,\ldots$ are satisfiable \wrt \Omc, \Omc-saturated,
  and $(\q_T,\Omc)$-minimal. By
  Theorem~\ref{thm:bounded-atemporal}, the length of
  $\r_1,\r_2,\ldots$ is bounded by a polynomial in the sizes of
  $\q_T$ and $\Omc$. Since there are
  only $|\text{supp}(\Dmc_1)|$ time points to consider, the overall sequence
  $\Dmc_1,\Dmc_2,\ldots,\Dmc_n$ is bounded by a polynomial in $\q_T$,
  $\Omc$, and $|\text{supp}(\Dmc_1)|$.
\end{proof}

\subsection{Proof of Theorem~\ref{thm:learning}}

We restate Theorem~\ref{thm:learning} for convenience.

\thmlearning*

\begin{proof} 
  Let \Lmc be as in the theorem. Let $\q_T$ be the target query, \Omc
  an \Lmc ontology, and $\Dmc,a$ be a positive example with $\Dmc =
  \Amc_0\ldots\Amc_n$ and such that $\Dmc$ is satisfiable \wrt \Omc.
  By Lemma~\ref{lem:obda-normalform}, we can assume that \Omc is
  actually in normal form. Moreover, by
  Lemma~\ref{lem:query-normalform}, we can assume $\q_T$ to be in
  normal form as well. We further assume $\q_T$ to be of
  shape~\eqref{fullq2appendix}: 
  \[ \q_{T} = \q_{0} \mathcal{R}_{1} \q_{1} \dots \mathcal{R}_{n}
  \q_{n}.\]
  As $\q_T$ is safe, it does not have lone conjuncts.

  \medskip We start with showing~$(i)$ and then describe the necessary
  modifications for~$(ii)$ and~$(iii)$.  The idea of the proof is to
  modify \Dmc in a number of steps such that in the end $\Dmc$ viewed
  as a temporal query is equivalent to $\q_T$. As a general proviso we
  assume that at all times: each $\Amc_i$ (viewed as CQ) is
  \Omc-saturated; this is without loss of generality since instance
  checking \wrt \ELHIF ontologies is
  decidable~\cite{DBLP:phd/dnb/Tobies01},
 
  We call a temporal data instance \Dmc \emph{temporally
  minimal} if there is no time point $i$ such that $\Dmc',a$ is a positive
  example where $\Dmc'$ is obtained from $\Dmc$ by dropping $\Amc_i$
  from \Dmc. Clearly, temporal minimality can be established using at
  most $\max(\Dmc)$ membership queries, and a temporally minimal data
  instance \Dmc satisfies that $\max(\Dmc)$ is at least the number of
  occurrences of $suc$ and $<$ in $\q_T$ and at most as large as the
  size of $\q_T$. 

  Thus, we can assume without loss of generality that the initial data
  example \Dmc is temporally minimal. Thus, every root
  $\Omc$-homomorphism $h:\q_T\to \Dmc$ is block
  surjective\footnote{Recall the 
    notion of \emph{(block surjective) root \Omc-homomorphisms} from
    the proof of Theorem~\ref{thm:first}.} for the
  block size \[b:=\max(\Dmc)+1,\] as $\Dmc$ has only one block for
  this $b$. In fact, during all
    modifications, we maintain the invariant that every root
    $\Omc$-homomorphism is block surjective for this number $b$. We
    use this initial constant $b$ in steps~2,~3, and~4 below.
  
\paragraph{Step~1.} We first aim to find a temporal data instance
which is \emph{tree-shaped}, meaning that in $\Dmc = \Amc_0\ldots\Amc_n$ each $\Amc_i$ is tree-shaped. To
achieve this, we exhaustively apply the following rules
\mn{Unwind} and \mn{Minimise} with a preference given to
\mn{Minimise}. A \textit{cycle} in a data instance is a sequence
$R_1(a_1,a_2),\ldots,R_n(a_n,a_1)$ of distinct atoms such that
$a_1,\ldots,a_n$ are distinct.

\smallskip \mn{Minimise.} If there is some $i$ and some individual
$b\in \ind(\Amc_i)$ such that $\Dmc',a$ is a positive example where
$\Dmc'$ is obtained from \Dmc by dropping from $\Amc_i$ all atoms that
mention $b$, then replace \Dmc with $\Dmc'$.

\smallskip \mn{Unwind.} Choose an atom $R(a_1,a_2)\in \Amc_i$ that is
part of a cycle. Obtain $\Amc_i'$ by first adding a disjoint copy
$\Amc'_i$
of $\Amc_i$ to $\Amc_i$ and let $a_1',a_2'$ be the copies of $a_1,a_2$ in
$\Amc_i'$. Then replace all atoms $S(a_1,a_2)$ (respectively,
$S(a_1',a_2')$) by $S(a_1,a_2')$ (respectively, $S(a_1',a_2)$), for all roles $S$.

\smallskip
It is clear that the resulting temporal data instance is tree-shaped
as required. It is still temporally minimal and the invariant that
every root $\Omc$-homomorphism is block surjective is preserved. 


\paragraph{Step 2.} In this step, we `close' $\Dmc$ under
applications of the Rules~(a)--(e) used in Lemma~\ref{lem:stepae}.
Formally, consider the following Rule~2(x), for
$\text{x}\in\{\text{a,b,c,d,e}\}$.
\begin{enumerate}[label=2(x),align=left,leftmargin=*]

  \item Let $\Dmc'$ be a data instance obtained from $\D$ by applying
    Rule~(x) from the proof of Lemma~\ref{lem:stepae}. If $\D',a$ is a positive example, replace $\D$
    with the result of the exhaustive application of \mn{Minimise}
    to $\D'$.

\end{enumerate}
We first apply~2(b) and~2(c) until \Dmc stabilises. Then, we exhaustively
apply~2(a), 2(d), and 2(e) giving preference to~2(a).

After \textit{Step~2}, \Dmc satisfies that, if $\Dmc'$ is the result of an
application of Rules~(a)--(e), then $\Dmc',a$ is not a positive
example.

\paragraph{Step 3.} In this step, we take care of lone conjuncts in
\Dmc by applying~$(\ast)$ below as long as \Dmc contains one. Recall
that $\q_T$ does not, so we can simplify \Dmc.
\begin{itemize}[align=left,leftmargin=*]

  \item[$(\ast)$] Choose a primitive block
    $\emptyset^b\Amc\emptyset^b$ in $\Dmc$ such that $\Amc$, viewed as
    CQ $\q$ is meet-reducible \wrt \Omc within ELIQ.
    Let $\Fmc_\q=\{\q_1,\ldots,\q_\ell\}$
    and $w = \hat\q_1\emptyset^b \hat \q_2\emptyset^b\dots \hat\q_\ell
    \emptyset^b$. Denote with $\D_k$ the result of replacing
    $\emptyset^b\Amc\emptyset^b$ in \Dmc with $\emptyset^b(w)^k$. Then
    identify some $i\geq 1$ such that $\Dmc_i,a$ is a positive
    example, by using membership queries for $i=1,2,\dots$. Notice
    that this requires only polynomially many membership queries as $\D_k,a$ is a
    positive example for $k=|\q_T|$, and that all queries are of
    polynomial size since $\Fmc_\q$ is of polynomial size. Replace \Dmc with the result of
    exhaustively applying Rule~2(a) to $\D_i$ and subsequently
    shortening blocks $\emptyset^d$ for $d>b$ to $\emptyset^b$.

\end{itemize}
Let \Dmc be the result of \textit{Step~3}. It is routine to verify that
2(a)--2(e) are not applicable, that \Dmc is $b$-normal, and that \Dmc
is without lone conjuncts \wrt \Omc within
$\LTL_p\NDD(\text{ELIQ})$.
By Lemma~\ref{lem:stepae}, any root $\Omc$-homomorphism is a root
\Omc-isomorphism. Thus, the algorithm has identified all blocks in the
following sense. Suppose that $\q_T=\q_0\Rmc_1\q_1\ldots\Rmc_m\q_m$ is
a sequence of blocks $\q_i=\r_0^i\ldots \r_{\ell_i}^i$ and
\[\Dmc = \D_{0}\emptyset^{b}\D_{1} \ldots \emptyset^{b}\D_{n} \text{
where } \D_{i}=\mathcal{A}_{0}^{i}\ldots\mathcal{A}_{k_{i}}^{i}.\]
Then $m=n$ and each block $\Dmc_i$ in \Dmc is isomorphic to $\q_i$,
that is, $\ell_i=k_i$ and $\hat\r_{j}^i=\Amc_j^i$, for
all $i,j$ with $0\leq i\leq n$ and $0\leq j\leq k_i$. It is unclear, however, whether the
$\Rmc_i$ are (a single) $\leq$ or a sequence of $<$. This is resolved in the final step.

\paragraph{Step 4.} We determine $\Rmc_{i+1}$, for each $i$ with $0\leq i<n$, as follows:
\begin{itemize}

  \item If $\r^i_{k_i}\wedge \r^{i+1}_0$ is satisfiable \wrt \Omc and
    $\Dmc_i,a$ with $\Dmc_i=\Dmc_0 \emptyset^{b} \ldots\emptyset^b
    \Dmc_i\Join\Dmc_{i+1} \emptyset^b\ldots\emptyset^{b}\D_{n}$
    ($\Join$ defined as in the proof of Theorem~\ref{thm:first}) is a
    positive example, then $\Rmc_{i+1}$ is $\leq$. Otherwise, let $s$
    be minimal such that $\Dmc'_i,a$ is a positive example for
    $\Dmc'_i=\Dmc_0 \emptyset^{b} \ldots\emptyset^b
    \Dmc_i\emptyset^{s}\Dmc_{i+1}
    \emptyset^b\ldots\emptyset^{b}\D_{n}$.  Then, $\Rmc_{i+1}$ is a
    sequence of $s$ times $<$.

\end{itemize}
We have thus shown that indeed the returned query is equivalent to
$\q_T$. It remains to argue that the algorithm issues only
polynomially many membership queries.  We analyse \textit{Steps~1--4}
separately.

For \textit{Step~1}, let $\Dmc_1,\Dmc_2,\ldots$ be the sequence of temporal data
instances that \mn{Unwind} is applied to during \textit{Step~1}. Clearly,
all these queries are $(\q_T,\Omc)$-minimal (recall that we give
preference to \mn{Minimise}) and \Omc-saturated. Since an
application of \mn{Minimise} decreases the overall number of
individuals in the instance, there are only polynomially many
applications of \mn{Minimise} between $\Dmc_i$ and
$\Dmc_{i+1}$. In the
proof of Lemma~14 in~\cite{DBLP:conf/dlog/FunkJL22}, it is shown that
the operation
\mn{Unwind}\footnote{\mn{Unwind} is called~\textit{Double Cycle}
in~\cite{DBLP:conf/dlog/FunkJL22}.} applied to a
$(\q_T,\Omc)$-minimal CQ $\q$ leads to a strictly weaker CQ $\q'$,
that is, $\q\models_\Omc \q'$, but not vice versa. This applies
here as well, and implies that
$\Dmc_1,\Dmc_2,\ldots$ is a
generalisation sequence towards $\q_T$ \wrt \Omc. Applying
Lemma~\ref{lem:temporal-genseq} yields that \textit{Step~1} terminates in
time polynomial in the size of $\q_T$, \Omc, and $|\text{supp}(\Dmc)|$ which in
turn is bounded by the size of $\q_T$ (recall that \Dmc is
temporally minimal).

We next analyse \textit{Step~2}, starting with Rules~2(b) and~2(c). First note that the
number of applications of Rules~2(b) and~2(c) is bounded by the number of
$<$ and $\leq$ in $\q_T$. To see this, we inductively show
that Rules~2(b) and~2(c) preserve the fact that every root
$\Omc$-homomorphism $h:\q_T\to \Dmc$ is block surjective. As \Dmc is
temporally minimal, this certainly holds before \textit{Step~2}. Suppose now that
$\Dmc'$ is obtained by a single application of~2(b) or~2(c) to $\Dmc$, and
that there is a root \Omc-homomorphism that is not block surjective.
Then we can easily construct a non-block surjective homomorphism from
$\q_T$ to $\Dmc$, a contradiction. Applications of~\mn{Minimise}
also preserve the claim. Also note that the block-surjectivity implies
that the support of the resulting \Dmc is bounded by the size of~$\q_T$.

We next analyse Rules~2(a), 2(d), and~2(e). Let $\Dmc_1,\Dmc_2,\ldots$ be
a sequence of temporal data instances obtained by a sequence of applications
of~2(a). Clearly, $\Dmc_1,\Dmc_2,\ldots$ is a generalisation sequence
towards $\q_T$ \wrt \Omc such that all $\Dmc_i$ are satisfiable \wrt
\Omc, \Omc-saturated, and $(\q_T,\Omc)$-minimal. By
Lemma~\ref{lem:temporal-genseq}, the length of the sequence is
bounded by a polynomial in the sizes of $\q_T$ and $\Omc$, and in
$|\text{supp}(\Dmc)|\leq |\q_T|$. Further note that applications of~2(a)
preserve that every root \Omc-homomorphism is block surjective and
that~2(b) and~2(c) remain not applicable.

Next consider an application of~2(d) to a temporal data instance
\Dmc where~2(a) is not applicable and such that every root
\Omc-homomorphism is block surjective. Let $\Dmc'$ be the result.
Since~2(a) is not applicable, every root \Omc-homomorphism to $\Dmc'$
must also be block surjective. Thus, the number of applications
of~2(d) is bounded by the number of $<$ and $\leq$ in $\q_T$. The same
argument works for~2(e). It is readily seen that~2(b) and~2(c) are still
not applicable, thus none of the rules is applicable to \Dmc.
Overall, we obtain a polynomial number of rule
applications. This finishes the analysis of \textit{Step~2}.
%

Consider now \textit{Step~3}.  Recall that, by Lemma~\ref{lonecon}, a
CQ $\q$ is meet-reducible \wrt \Omc in ELIQ iff $|\Fmc_\q|\geq 2$
provided that $\q'\not\models_\Omc \q''$, for all distinct
$\q',\q''\in \Fmc_\q$. Thus, to find a lone conjunct
$\emptyset^b\Amc\emptyset^b$ in \Dmc, we can compute such a minimal
frontier $\Fmc_\q$ of $\q$ \wrt \Omc by first computing any frontier
$F$ (which is possible by assumption) and then exhaustively removing
from $F$ queries $\q''$ such that $\q'\models_\Omc \q''$ for some
$\q'\in F$ with $\q'\neq\q''$.  Note that the test $\q'\models_\Omc
\q''$ is decidable for \ELHIF
ontologies~\cite{DBLP:conf/ijcai/Bienvenu0LW16}.

As noted in~$(\ast)$, identifying the right $\Dmc_i$ needs only polynomially many membership queries
(despite the fact that deciding meet-reducibility might require more time).
Since exhaustive application of~2(a) requires only polynomial time, a
single application of~$(\ast)$ requires only polynomially many
membership queries. Moreover, using the fact that~2(a) is not
applicable before application of~$(\ast)$ one can show that the number
of `gaps' is increased and that the rule preserves that every root
\Omc-homomorphism is block surjective.  Hence,~$(\ast)$ is applied at
most once for each $\leq$ in $\q_T$.

It remains to analyse the running time of \textit{Step~4}. Clearly, only
linearly many (in the size of $\q_T$) membership queries are asked. To
finish the argument, it remains to note that \ELHIF admits tractable
containment reduction and that satisfiability \wrt \ELHIF ontologies is decidable.

\medskip

We argue next that the above algorithm runs in polynomial time if \Lmc
additionally admits polynomial time instance checking, polynomial time
computable frontiers, and meet-reducibility of ELIQs \wrt
\Lmc ontologies can be decided in polynomial time. First note
that \Omc-saturation of \Dmc (which is assumed throughout the algorithm) can
be established in polynomial time via instance checking. Then observe
that, in~\textit{Step~3}, a (not necessarily minimal) frontier
$\Fmc_\q$ can be computed in polynomial time and meet-reducibility can
also be decided in polynomial time, by assumption. Together with the analysis
of~\textit{Step~3} above, this yields that~\textit{Step~3} needs only
polynomial time. Finally observe that also~\textit{Step~4} runs in
polynomial time since the tests for satisfiability can be reduced to
(polynomial time) instance checking.

It remains to prove Points~$(ii)$ and~$(iii)$
from Theorem~\ref{thm:learning}. The learning algorithm for
Point~$(ii)$ is similar to the algorithm provided above, but
with a modified \textit{Step~3} since in this case~$\q_T$ might have lone
conjuncts and possibly more than one variable from $\var(\q_T)$ is mapped to the same time
point in \Dmc. Let $T$ be the temporal depth of the target query.

\paragraph{Step~3$'$.} In this step, we apply~$(\ast')$ until \Dmc
stabilises.
\begin{itemize}[align=left, leftmargin=*]

  \item[$(\ast')$] Choose a primitive block
    $\emptyset^b\Amc\emptyset^b$ and an ELIQ $\q$ with $\Amc=\hat \q$.
    Let $\Fmc_\q=\{\q_1,\ldots,\q_\ell\}$
    and $w = \hat\q_1\emptyset^b \hat \q_2\emptyset^b\dots \hat\q_\ell
    \emptyset^b$. Let $\Dmc'$ be the result of replacing
    $\emptyset^b\Amc\emptyset^b$ in \Dmc with $\emptyset^b(w)^T$. If
    $\Dmc',a$ is a positive
    example, then replace \Dmc with the result of
    exhaustively applying Rule~2(a) to $\Dmc'$ and subsequently
    shortening blocks $\emptyset^d$ for $d>b$ to $\emptyset^b$.

\end{itemize}
Let \Dmc be the result of \textit{Step~3}. It is routine to verify
that 2(a)--2(e) are not applicable. The following can be proven
similar to Lemma~\ref{lem:stepae}.
\begin{lemma} \label{}
  Let $\mathcal{O}$ and $\q_T$ be as above. Let $b$ exceed the number
  of $\Diamond$ and $\nxt$ in $\q_T$, and let $\D$ be $b$-normal. If
  Rules~2(a)--(e) and $(\ast')$ are not applicable, then any root
  $\mathcal{O}$-homomorphim $h \colon \q \rightarrow \D$ is a root
  $\mathcal{O}$-isomorphism.
\end{lemma}

As an immediate consequence, after \textit{Step 3$'$}, the modified
algorithm has identified all blocks in $\q_T$ as described above and
it remains to apply \textit{Step~4}.

\medskip The learning algorithm for Point~$(iii)$ is a similar
modification. Note that for the query class $\LTL_p\ND(\text{ELIQ})$,
we know that the temporal depth of $\q_T$ is exactly $T_0=\max(\Dmc)$
for the temporally minimal input example $\Dmc$. This $T_0$ can then
be used in place of $T$ in~$(\ast')$ above. This finishes the proof of
Theorem~\ref{thm:learning}. 
\end{proof}

\subsection{Proof of Theorem~\ref{thm:learning-dllite}}

\thmlearningdllite*

\begin{proof}
  The theorem is a direct consequence of Theorem~\ref{thm:learning}
  and the fact that the considered ontology languages satisfy all
  conditions mentioned in that theorem. 
  Most importantly:
  \begin{itemize}

    \item $\DL_{\cal{F}}$ admits polynomial time instance
      checking~\cite{romans} and $\DL_{\cal{F}}^{-}$ admits polynomial time computable
      frontiers~\cite{DBLP:conf/ijcai/FunkJL22}. Meet-reducibility in
      $\DL_{\cal{F}}^{-}$ is decidable in polynomial time, by
      Lemma~\ref{lem:meet-reducible}. 

    \item $\DL_{\cal{H}}$ admits polynomial time instance
      checking~\cite{romans} and admits polynomial time computable
      frontiers~\cite{DBLP:conf/ijcai/FunkJL22}.  
  \end{itemize}
  This completes the proof.
\end{proof}


\end{appendix}
\end{document}